%% file: iclr2021_conference.tex
\documentclass{article} 
\input{assets/definitions}
\iclrfinalcopy
\title{Discovering Non-monotonic Autoregressive Orderings with Variational Inference}
\input{assets/authors}

\begin{document}
\maketitle
\vspace{-1em}
\input{sections/00_abstract}
\input{sections/01_intro}
\input{sections/02_background}
\input{sections/03_prelim}
\input{sections/04_voi}

\input{sections/05_experiments}
\input{sections/06_order_analysis}
\input{sections/07_ablations}

\input{sections/08_conclusion}
\newpage
\bibliography{iclr2021_conference}
\bibliographystyle{iclr2021_conference}
\newpage
\input{sections/09_appendix}
\end{document}

%% file: assets/definitions.tex
\usepackage{iclr2021_conference,times}
\input{math_commands.tex}

\usepackage{url}
\usepackage{xcolor}
\usepackage{graphicx}
\usepackage{amsmath}
\usepackage{amsfonts}
\usepackage{booktabs}
\usepackage{siunitx}
\usepackage{multirow}
\usepackage{appendix}
\usepackage{algorithm}
\usepackage{amsthm}
\usepackage{amssymb}
\usepackage{subfigure}
\usepackage{float}
\usepackage{wrapfig}
\usepackage{booktabs} 
\usepackage{nicefrac} 
\usepackage{algpseudocode}
\usepackage{textcomp}

\usepackage{todonotes}
\usepackage{footnote}
\makesavenoteenv{table}

\definecolor{citecolor}{RGB}{34, 149, 34}
\usepackage[colorlinks, citecolor=citecolor, final]{hyperref}
\newtheorem{theorem}{Theorem}[section]
\newtheorem*{*theorem}{Theorem}
\newtheorem{definition}[theorem]{Definition}

\ExplSyntaxOn
\NewDocumentCommand{\longdash}{ O{2} }
 {
  --\prg_replicate:nn { #1 - 1 } { \negthinspace -- }
 }
\ExplSyntaxOff

\newcommand{\Variational}{\textit{Variational Order Inference }}

%% file: math_commands.tex
\usepackage{amsmath,amsfonts,bm}

\def\eqref#1{equation~\ref{#1}}

\def\1{\bm{1}}

\DeclareMathAlphabet{\mathsfit}{\encodingdefault}{\sfdefault}{m}{sl}
\SetMathAlphabet{\mathsfit}{bold}{\encodingdefault}{\sfdefault}{bx}{n}



%% file: assets/authors.tex
\author{Xuanlin Li$^1$\thanks{Authors contributed equally.}\;, \;Brandon Trabucco$^1$\footnotemark[1]\;, \;Dong Huk Park$^1$, \;Michael Luo$^1$, \;Sheng Shen$^1$,\\
\textbf{Trevor Darrell$^1$, \;Yang Gao$^2$}\\
$^1$University of California, Berkeley, $^2$Tsinghua University\\
\small\texttt{\{xuanlinli17, btrabucco\}@berkeley.edu}
}

%% file: sections/00_abstract.tex
\begin{abstract}
The predominant approach for language modeling is to process sequences from left to right, but this eliminates a source of information: the order by which the sequence was generated. 
One strategy to recover this information is to decode both the \textit{content} and \textit{ordering} of tokens. 
Existing approaches supervise content and ordering by designing problem-specific loss functions and pre-training with an ordering pre-selected. 
Other recent works use iterative search to discover problem-specific orderings for training, but suffer from high time complexity and cannot be efficiently parallelized.
We address these limitations with an unsupervised parallelizable learner that discovers high-quality generation orders purely from training data---no domain knowledge required. 
The learner contains an encoder network and decoder language model that perform variational inference with autoregressive orders (represented as permutation matrices) as latent variables. The corresponding ELBO is not differentiable, so we develop a practical algorithm for end-to-end optimization using policy gradients. We implement the encoder as a Transformer with non-causal attention that outputs permutations in one forward pass. Permutations then serve as target generation orders for training an insertion-based Transformer language model.
Empirical results in language modeling tasks demonstrate that our method is context-aware and discovers orderings that are competitive with or even better than fixed orders.
\end{abstract}

%% file: sections/01_intro.tex

\section{Introduction}

\begin{figure}[h]
    \vspace{-.4cm}
    \centering
    \includegraphics[width=0.48\textwidth,page=1,trim={4cm 3cm 4cm 3cm},clip]{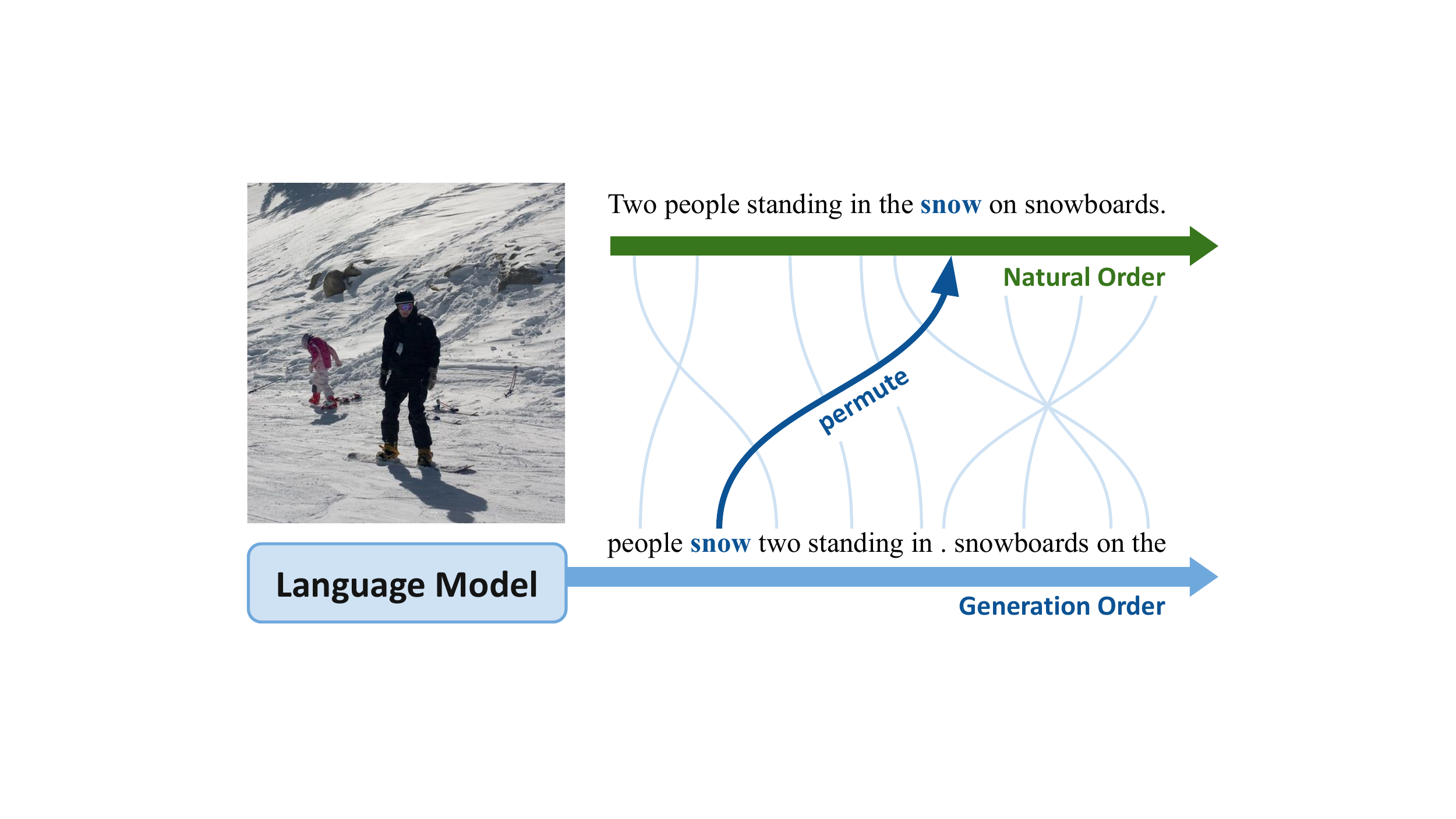}
    \includegraphics[width=0.51\textwidth,page=2,trim={1.5cm 1.5cm 1.5cm 2cm},clip]{figures/high_level_figures.pdf}    
    \vspace{-.4cm}
    \caption{Left: our language model, shown in light blue, learns to decode in non-monotonic generation orders, rather than pre-determined orders, such as left-to-right. Right: during training, we leverage an encoder in a variational inference pipeline to parameterize a latent distribution over the generation orders for the autoregressive language model. In this way, training can be done in just one forward / backward pass per batch, unlike previous approaches in non-monotonic sequence modeling that require multiple forward passes per batch to determine a generation order.}
    \vspace{-0.15cm}
    \label{fig:high_lvl_summary}
\end{figure}

Autoregressive models have a rich history. Early papers that studied autoregressive models, such as \citep{DBLP:journals/jmlr/UriaCGML16} and \citep{pmlr-v37-germain15}, showed an interest in designing algorithms that did not require a gold-standard autoregressive order to be known upfront by researchers. However, these papers were overshadowed by developments in natural language processing that demonstrated the power of the left-to-right autoregressive order \citep{DBLP:conf/emnlp/ChoMGBBSB14, DBLP:conf/nips/SutskeverVL14}. Since then, the left-to-right autoregressive order has been essential for application domains such as image captioning \citep{DBLP:conf/cvpr/VinyalsTBE15, DBLP:conf/icml/XuBKCCSZB15}, machine translation \citep{DBLP:conf/emnlp/LuongPM15, DBLP:journals/corr/BahdanauCB14} and distant fields like image synthesis \citep{DBLP:conf/nips/OordKEKVG16}. However, interest in non left-to-right autoregressive orders is resurfacing \citep{DBLP:conf/icml/WelleckBDC19, DBLP:conf/icml/SternCKU19}, and evidence \citep{DBLP:journals/corr/VinyalsBK15, gu-etal-2018-top, DBLP:conf/iclr/Alvarez-MelisJ17} suggests adaptive orders may produce more accurate autoregressive models. These positive results make designing algorithms that can leverage adaptive orders an important research domain.

\begin{figure}
    \centering
    \vspace{-0.1cm}
    \includegraphics[width=0.92\textwidth,page=3,trim={.7cm .3cm .7cm .5cm},clip]{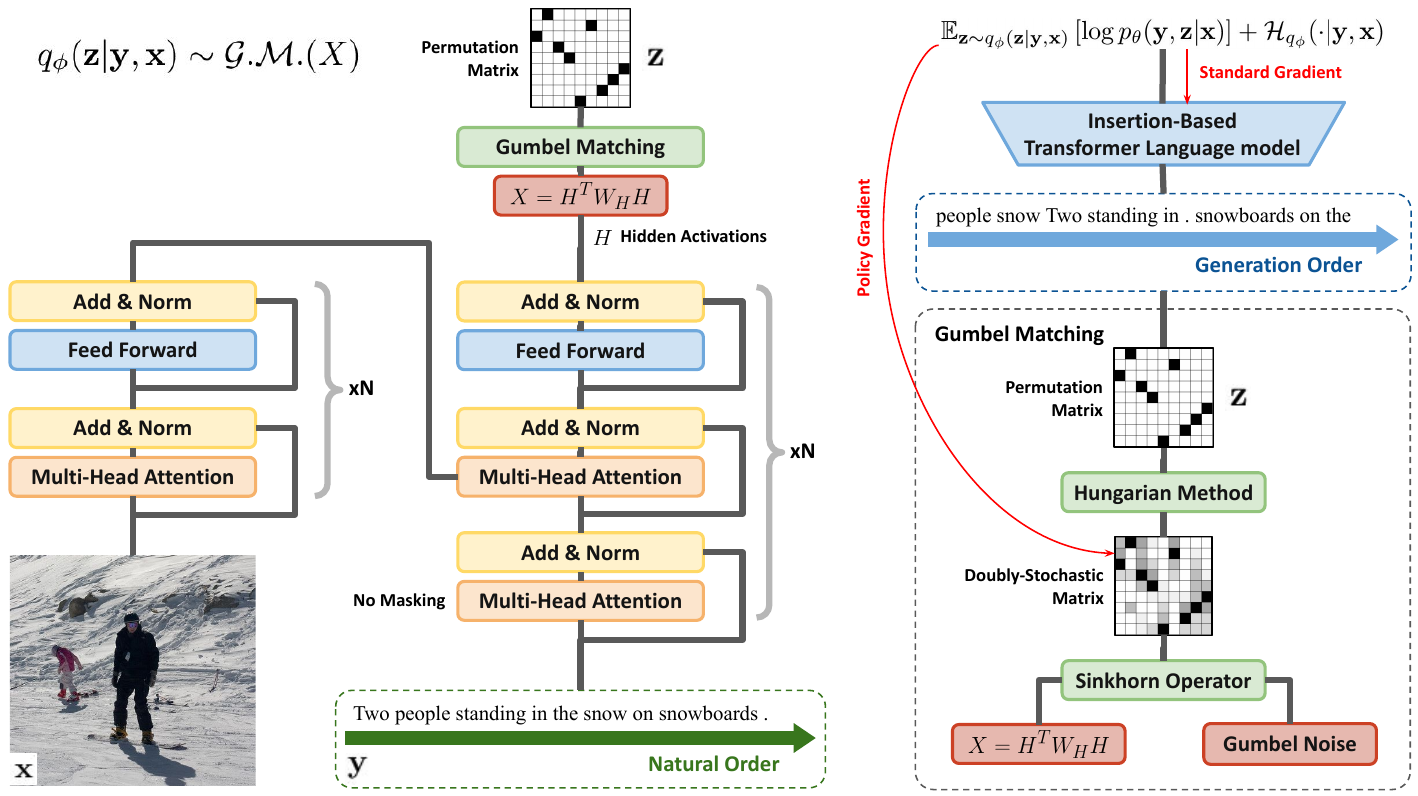}
    \vspace{-0.35cm}
    \caption{Architecture for sequence-modeling tasks. The goal is to predict the target sequence $\mathbf{y}$ given the source sequence $\mathbf{x}$, with latent generation orders $\mathbf{z}$ represented as permutation matrices. We use a Transformer without causal masking to serve as the encoder in \Variational (VOI), which samples orderings in a single forward pass. These orderings are used to train an insertion-based Transformer language model, which serves as the VOI decoder. As the objective is non-differentiable over permutation matrices, policy gradient algorithms (e.g. Reinforce \citep{NIPS1999_1713}, PPO \citep{schulman2017proximal}) are applied to update the permutation-generating encoder.}
    \vspace{-1.00cm}
    \label{fig:voi_transformer}
\end{figure}

Inferring autoregressive orderings in a data-driven manner is challenging. Modern benchmarks for machine translation \citep{Stahlberg2019NeuralMT} and other tasks \citep{oda2015ase:pseudogen1} are not labelled with gold-standard orders, and left-to-right seems to be the default. This could be explained if domain-independent methodology for identifying \textit{high-quality} orders is an open question. Certain approaches \citep{DBLP:conf/icml/SternCKU19, DBLP:conf/icml/WelleckBDC19, DBLP:journals/corr/abs-2001-05540} use hand-designed loss functions to promote a \textit{genre} of orders---such as balanced binary trees. These loss functions incorporate certain domain-assumptions: for example, they assume the balanced binary tree order will not disrupt learning. Learning disruption is an important consideration, because prior work shows that poor orders may prohibitively slow learning \citep{DBLP:conf/icml/ChenMRA18}. Future approaches to inferring autoregressive orders should withhold domain knowledge, to promote their generalization.

To our best knowledge, we propose the first domain-independent unsupervised learner that discovers high-quality autoregressive orders through fully-parallelizable end-to-end training without domain-specific tuning. We provide three main contributions that stabilize this learner. First, we propose an encoder architecture that conditions on training examples to output autoregressive orders represented as permutation matrices using techniques in combinatorial optimization. Second, we propose \Variational (VOI) that learns an approximate posterior over autoregressive orders. Finally, we develop a practical algorithm for solving the resulting non-differentiable ELBO end-to-end with policy gradients. A high-level summary of our approach is presented in Figure \ref{fig:high_lvl_summary}, and a detailed architecture diagram for sequence modeling tasks is presented in Figure \ref{fig:voi_transformer}.

Empirical results with our solution on various sequence modeling tasks suggest that with similar hyperparameters, our algorithm is capable of recovering autoregressive orders that are even better than fixed orders. Case studies suggest that our learned orders depend adaptively on content, and resemble a type of \textit{best-first} generation order, which prioritizes salient objects / phrases and deprioritizes auxillary tokens (see Fig. \ref{fig:example_orders_intro}). Our experimental framework is available at \href{https://github.com/xuanlinli17/autoregressive_inference}{this link}.

\begin{figure}
    \centering
    \includegraphics[width=0.9\linewidth, trim=3.0cm 1.75cm 1.5cm 1.5cm, clip]{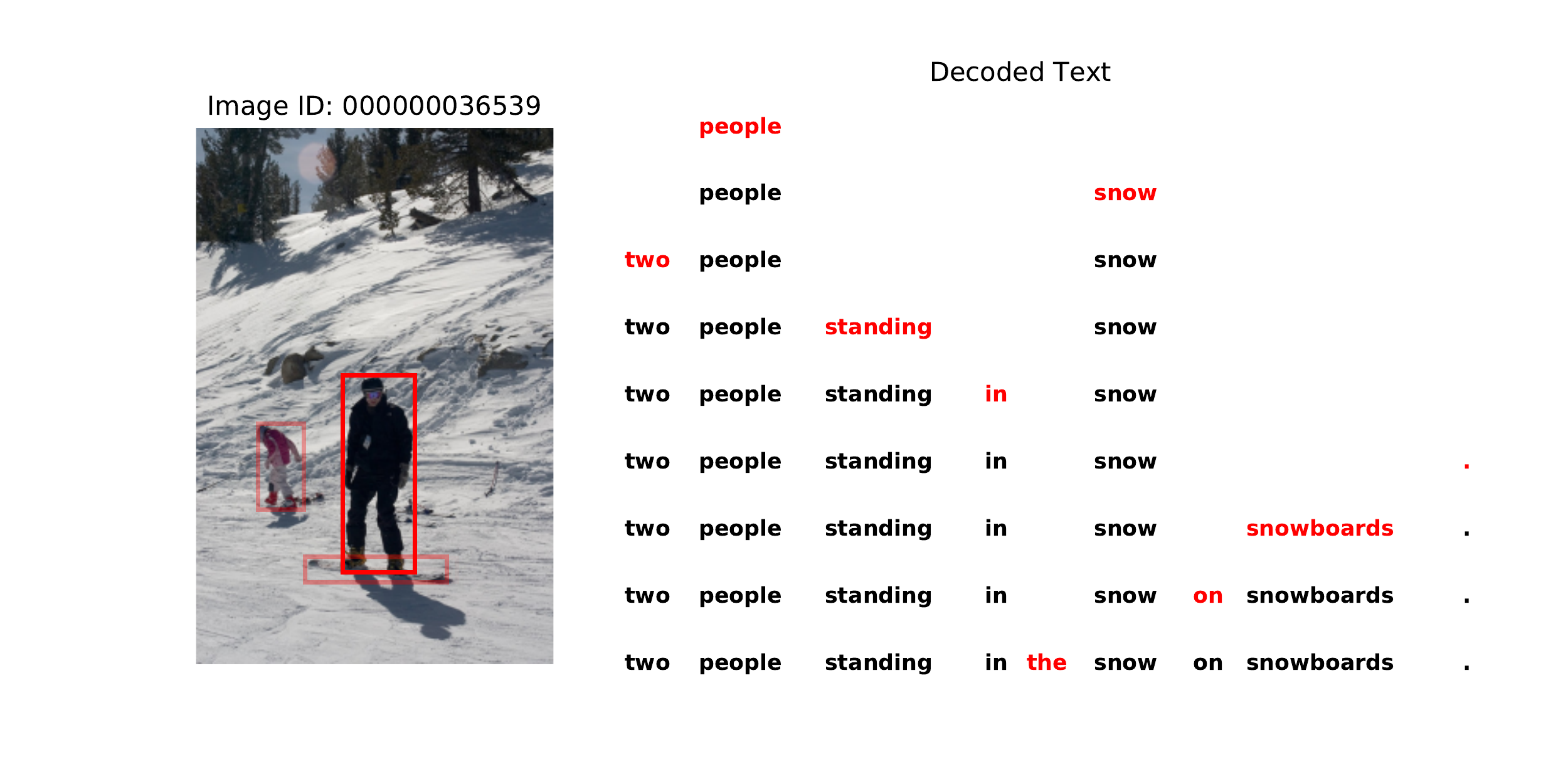}
    \includegraphics[width=\linewidth, trim=0 1.8cm 1.5cm 0, clip]{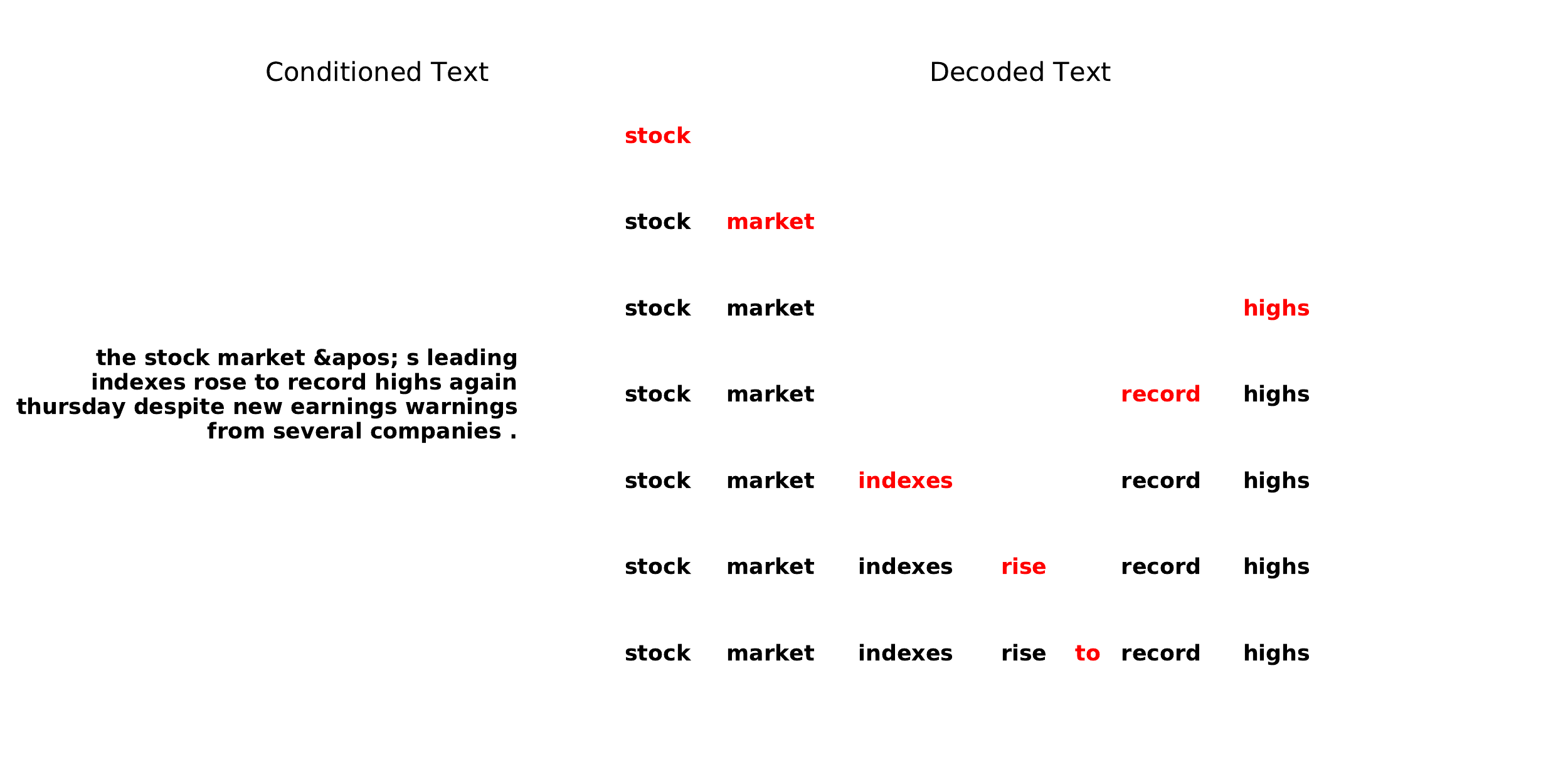}
    \includegraphics[width=\linewidth, trim=0 1.8cm 1.5cm 0, clip]{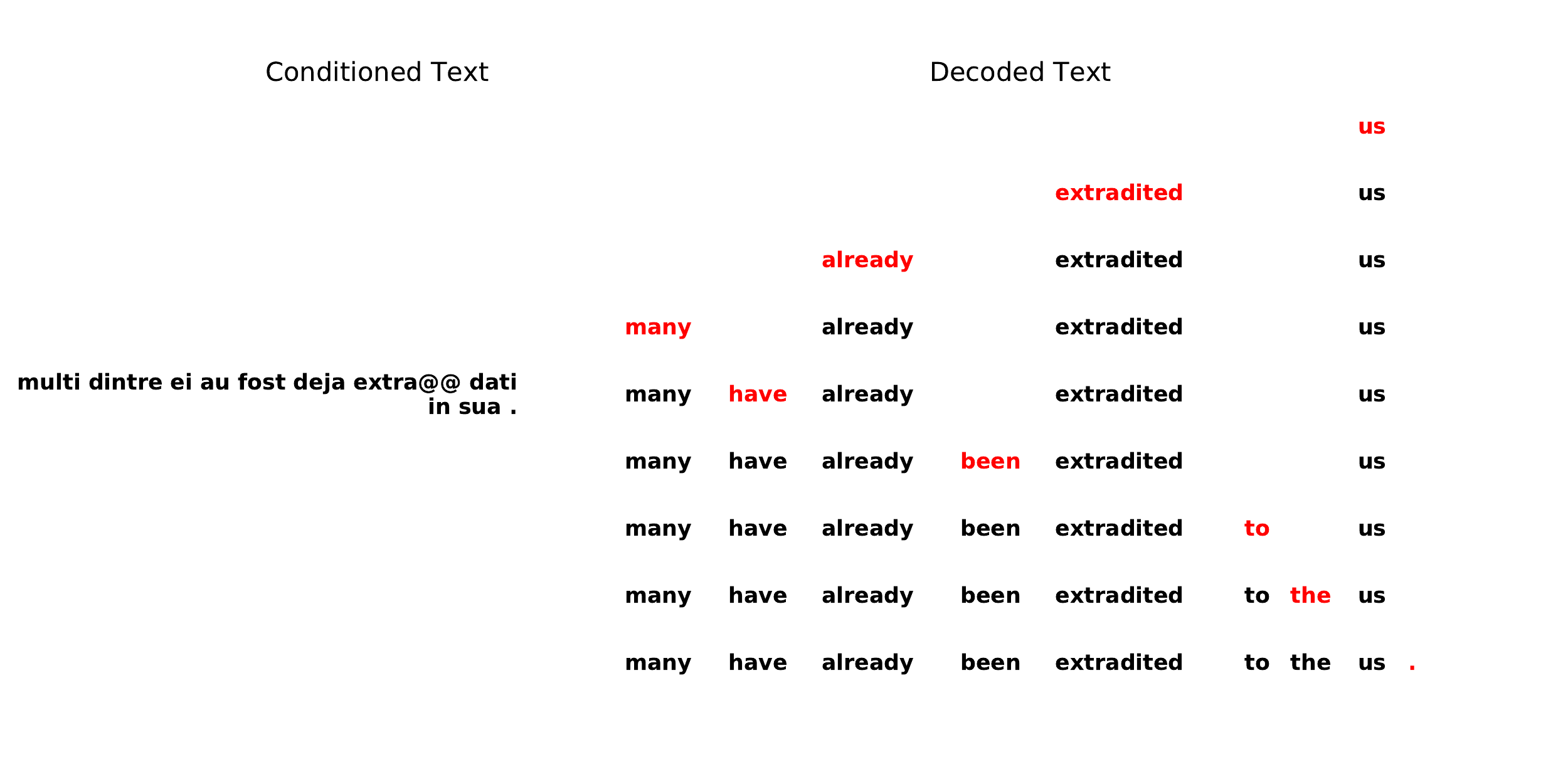}    
    \vspace{-0.2cm}
    \caption{We present example sequence generations on various conditional sequence generation tasks (image captioning, text summarization, and machine translation) from the decoder insertion-based language model (top right of Fig. \ref{fig:voi_transformer}) in \Variational. The orderings tend to prioritize descriptive phrases, such as focal objects in conditioned images (e.g. ``people'', ``snow'') and salient phrases in conditioned sentences (e.g. ``stock market'', ``U.S.'', ``extradicted''), while putting modifier tokens (e.g. ``to'', ``on'', ``the'') last, resembling a \textit{best-first} generation order found with unsupervised learning. }
    \label{fig:example_orders_intro}
    \vspace{-0.2cm}
\end{figure}


%% file: sections/02_background.tex
\vspace{-0.1cm}
\section{Related Works}
\textbf{Autoregressive Models} Autoregressive models decompose the generation of a high dimensional probability distribution by generating one dimension at a time, with a predefined order. Combined with high capacity neural networks, this approach to modeling complex distributions has been very successful~\citep{sutskever2011generating,mikolov2012statistical}. Recent works have achieved great improvements with autoregressive models in many applications, including language modeling~\citep{radford2018improving, radford2019language,brown2020language}, machine translation~\citep{sutskever2014sequence} and image captioning~\citep{karpathy2015deep}. Most previous works on autoregressive models regress to an ordering selected by designers, with left-to-right emerging as the primary choice. In contrast, our method is capable of learning arbitrary orderings conditioned on data and is more flexible. 

\textbf{Non-Monotonic Autoregressive Orderings} 
\citet{ford2018importance} shows that a sub-optimal ordering can severely limit the viability of a language model and propose to first generate a partially filled sentence template and then fill in missing tokens. 
Previous works have also studied bidirectional decoding \citep{sun2017bidirectional,zhou2019synchronous,mehri2018middle} and syntax trees based decoding ~\citep{yamada2001syntax,charniak2003syntax,dyer2016recurrent,aharoni2017towards,wang2018tree} in the natural language setting. 
However, all of the works mentioned above do not learn the orderings and instead opt to use heuristics to define them. \cite{chan2019kermit} performs language modeling according to a known prior, such as balanced binary tree, and does not allow arbitrary sequence generation orders.
\citet{welleck2019non} proposes to use a tree-based recursive generation method to learn arbitrary generation orders. However, their performance lags behind that of left-to-right. 
\citet{gu2019insertion} proposes Transformer-InDIGO to allow non-monotonic sequence generation by first pretraining with pre-defined orderings, such as left-to-right, then fine-tuning use Searched Adaptive Order (SAO) to find alternative orderings. They report that without pretraining, the learned orders degenerate. In addition, they perform beam search to acquire plausible orderings, which cannot be efficiently parallelized across different time-steps. \citet{DBLP:conf/nips/EmelianenkoVS19} proposes an alternative to SAO, but suffers from similar poor time complexity. In contrast, our method learns high-quality orderings directly from data under fully-parallelizable end-to-end training.

\textbf{Variational Methods} Our method optimizes the evidence lower bound, or ELBO in short. ELBO is a quantity that is widely used as an optimization proxy in the machine learning literature, where the exact quantity is hard to compute or optimize. Variational methods have achieved great success in machine learning, such as VAE~\citep{kingma2013auto} and $\beta$-VAE~\citep{DBLP:conf/iclr/HigginsMPBGBML17}.

\textbf{Combinatorial Optimization} 
Recent works have studied gradient-based optimization in the combinatorial space of permutations \citep{Mena2018sinkhorn, plackett_luce1903, stick_breaking}. These works have been applied in tasks such as number sorting, jigsaw puzzle solving, and neural signal identification in worms. To our best knowledge, we are the first to build on these techniques to automatically discover autoregressive orderings in vision and language datasets. 

%% file: sections/03_prelim.tex
\section{Preliminaries}

The goal of autoregressive sequence modelling is to model an ordered sequence of target values $\mathbf y = \left( y_{1}, y_{2} \hdots, y_{n} \right) : y_{i} \in \mathbb{R}$, possibly conditioned on an ordered sequence of source values $\mathbf x = \left( x_{1}, x_{2} \hdots, x_{m} \right) : x_{i} \in \mathbb{R}$, where $(\mathbf x, \mathbf y)$ is sampled from the dataset $\mathcal{D}$. In the context of language modeling, $x_i,y_i \in \mathbb{N}$ as the token distribution is categorical.

Inspired by \cite{DBLP:conf/nips/VinyalsFJ15} and \cite{gu2019insertion}, we formulate the generation process of $\mathbf y$ as a $2n$ step process, where at time step $2t-1$ we generate a value, and at timestep $2t$ we select a not-yet-chosen position in $\{1,2,\cdots,n\}$ to insert the value. Thus, we introduce the latent sequence variable $\mathbf z = \left( z_{1}, z_{2} \hdots, z_{n} \right) : \mathbf z \in S_{n}$, where $S_{n}$ is the set of one-dimensional permutations of $\{1,2,\cdots,n\}$, and $z_{t}$ is defined as the absolute position of the value generated at time step $2t-1$ in the naturally ordered $\mathbf y$. Then $p( \mathbf y, \mathbf z | \mathbf x)$ denotes the probability of generating $\mathbf y$ in the ordering of $\mathbf z$ given the source sequence $\mathbf x$. We can thus factorize $p( \mathbf y, \mathbf z | \mathbf x)$ using the chain rule:
\begin{gather}\label{eq:joint}
    p( \mathbf y, \mathbf z | \mathbf x) = p ( y_{z_{1}} | \mathbf x) p ( z_{1} | y_{z_{1}}, \mathbf x) \prod_{i = 2}^{n} p ( y_{z_{i}} | z_{<i}, y_{z_{<i}} , \mathbf x) p ( z_{i} | z_{<i}, y_{z_{<=i}}, \mathbf x)
\end{gather}
For example, $p(y_1, y_2, z_1=2, z_2=1| \mathbf x)=p(y_2| \mathbf x)p(z_1|y_2, \mathbf x)p(y_1|z_1, y_2, \mathbf x)p(z_2|y_1,z_1,y_2, \mathbf x)$ is defined as the probability of generating $y_2$ in the first step, then inserting $y_2$ into absolute position $2$, then generating $y_1$, and finally inserting $y_1$ into absolute position $1$.

Note that in practice, the length of $\mathbf y$ is usually varied. Therefore, we do not first create a fixed-length sequence of blanks and then replace the blanks with actual values. Instead, we dynamically insert a new value at a position relative to the previous values. One common approach to predict such relative position is Pointer Network \citep{DBLP:conf/nips/VinyalsFJ15}. In other words, at timestep $t$, we insert the value at position $r_t$ relative to the previous generated values. Here, for any $\mathbf z \in S_n$, $\mathbf r = (r_1, r_2, \dots, r_n)$ is constructed such that there is a bijection between $S_n$ and the set of all constructed $\mathbf r$. Due to such bijection, we can use $\mathbf z$ and $\mathbf r$ interchangeably. We will use $\mathbf z$ throughout the paper.

%% file: sections/04_voi.tex
\section{Variational Order Inference (VOI)}
\label{sec:vans}
\input{figures/voi_architecture}

Starting from just the original data $\mathbf y$ in natural order, we can use variational inference to create an objective (\ref{eq:elbo}) that allows us to recover latent order $\mathbf z$, parametrized by two neural networks $\theta$ and $\phi$. The encoder network $\phi$ samples autoregressive orders given the ground truth data, which the decoder network $\theta$ uses to recover $\mathbf y$. More specifically, $\phi$ is a non-autoregressive network (permutation generator in Fig. \ref{fig:voi_diagram}) that takes in the source sequence $\mathbf x$ and the entire ground truth target sequence $\mathbf y$ and outputs latent order $\mathbf z$ in a single forward pass. $\theta$ is an autoregressive network (autoregressive decoder in Fig. \ref{fig:voi_diagram}) that takes in $\mathbf x$ and predicts both the target sequence $\mathbf y$ and the ordering $\mathbf z$ through the factorization in Equation (\ref{eq:joint}). We name this process \Variational (VOI).
\begin{gather}\begin{split}\label{eq:elbo}
    \mathbb{E}_{(\mathbf x,\mathbf y) \sim \mathcal{D}} \left[ \log p_\theta(\mathbf y| \mathbf x) \right] & =  \mathbb{E}_{(\mathbf x,\mathbf y) \sim \mathcal{D}} \left[ \log \mathbb{E}_{\mathbf{z} \sim q_\phi(\mathbf{z}|\mathbf{y}, \mathbf x)} \left[  \frac{p_\theta(\mathbf y, \mathbf z| \mathbf x)}{q_\phi(\mathbf{z}|\mathbf{y}, \mathbf x)} \right]\right]\\
    & \geq \mathbb{E}_{(\mathbf x,\mathbf y) \sim \mathcal{D}} \left[ \mathbb{E}_{\mathbf{z} \sim q_\phi(\mathbf{z}|\mathbf{y}, \mathbf x)} \left[ \log p_\theta(\mathbf y, \mathbf z | \mathbf x) \right] + \mathcal{H}_{q_\phi}(\cdot | \mathbf y, \mathbf x ) \right]
\end{split}\end{gather}
Here, $\mathcal{H}_{q_\phi}$ is the entropy term. In practice, a closed form for $\mathcal{H}_{q_\phi}$ usually cannot be obtained, so an approximation is needed. During training, we train $\phi$ and $\theta$ jointly to maximize the ELBO in (\ref{eq:elbo}). During testing, we only keep the decoder $\theta$.

To optimize the decoder network $\theta$ in (\ref{eq:elbo}), for each $\mathbf y$, we first sample $K$ latents $\mathbf z_1,\mathbf z_2,\dots,\mathbf z_K$ from $q_\phi(\cdot | \mathbf y, \mathbf x)$. We then update $\theta$ using the Monte-Carlo gradient estimate $\mathbb{E}_{\mathbf y \sim \mathcal{D}} \left[ \frac{1}{K} \sum_{i=1}^{K} \nabla_{\theta} \log p_\theta(\mathbf y, \mathbf z_i| \mathbf x) \right]$.

Optimizing the encoder network $\phi$ is tricky. Since $\mathbf z$ is a discrete latent variable, the gradient from $\log p_\theta(\mathbf y, \mathbf z)$ does not flow through $\mathbf z$. Thus, we formulate (\ref{eq:elbo}) in a reinforcement learning setting with a one-step Markov Decision Process $(\mathcal{S}, \mathcal{A}, \mathcal{R})$. Under our setting, the state space $\mathcal{S} = \mathcal{D}$; for each state $(\mathbf x, \mathbf y) \in \mathcal{D}$, the action space $\mathcal{A}_{(\mathbf x, \mathbf y)} = S_{\textrm{length}(\mathbf{y})}$ with entropy term $\mathcal{H}_{q_\phi}(\cdot|\mathbf y, \mathbf x)$; the reward function $\mathcal{R}((\mathbf x, \mathbf y), \mathbf z \in S_{\textrm{length}(\mathbf{y})}) = \log p_\theta(\mathbf y, \mathbf z | \mathbf x)$. We can then set the optimization objective $L(\phi)$ to be (\ref{eq:elbo}). In practice, we find that adding an entropy coefficient $\beta$ and gradually annealing it can speed up the convergence of decoder while still obtaining good autoregressive orders.

To compute $\nabla_\phi L(\phi)$, we derive the policy gradient with baseline formulation \citep{NIPS1999_1713}:
\begin{gather}\label{eq:L_phi}
\begin{split}
    \nabla_\phi L(\phi) = \mathbb{E}_{(\mathbf x,\mathbf y) \sim \mathcal{D}} \left[\mathbb{E}_{\mathbf{z} \sim q_\phi} \left[ \nabla_\phi \log q_\phi(\mathbf z | \mathbf y, \mathbf x) (\log p_\theta(\mathbf y, \mathbf z| \mathbf x) - b(\mathbf y, \mathbf x)) \right] + \beta \nabla_\phi \mathcal{H}_{q_\phi} \right] 
\end{split}
\end{gather}
where $b(\mathbf y, \mathbf x)$ is the baseline function independent of action $\mathbf z$. The reason we use a state-dependent baseline $b(\mathbf y, \mathbf x)$ instead of a global baseline $b$ is that the the length of $\mathbf y$ can have a wide range, causing significant reward scale difference. In particular, we set $b(\mathbf y, \mathbf x) = \mathbb{E}_{\mathbf{z} \sim q_\phi} \left[ \log p_\theta(\mathbf y, \mathbf z_i | \mathbf x) \right]$. If we sample $K\ge 2$ latents for each $\mathbf y$, then we can use its Monte-Carlo estimate $\frac{1}{K} \sum_{i=1}^K \log p_\theta(\mathbf y, \mathbf z_i | \mathbf x)$.

Since we use policy gradient to optimize $\phi$, we still need a closed form for the distribution $q_\phi(\mathbf{z}|\mathbf{y}, \mathbf{x})$. Before we proceed, we define $\mathcal{P}_{n\times n}$ as the set of $n \times n$ permutation matrices, where exactly one entry in each row and column is $1$ and all other entries are $0$; $\mathcal{B}_{n\times n}$ as the set of $n \times n$ doubly stochastic matrices, i.e. non-negative matrices whose sum of entries in each row and in each column equals $1$; $\mathbb R_{n\times n}^+$ as the set of non-negative $n \times n$ matrices. Note $\mathcal{P}_{n\times n} \subset \mathcal{B}_{n\times n} \subset \mathbb{R}_{n\times n}^+$.

To obtain $q_\phi(\mathbf{z}|\mathbf{y}, \mathbf{x})$, we first write $\mathbf z$ in two-dimensional form. For each $\mathbf z \in S_n$, let $f_{n}(\mathbf z) \in \mathcal{P}_{n\times n}$ be constructed such that $f_{n}(\mathbf z)_i = \mathrm{one\_hot}(z_i)$, where $f_{n}(\mathbf z)_i$ is the $i$-th row of $f_{n}(\mathbf z)$. Thus $f_{n}$ is a natural bijection from $S_n$ to $\mathcal{P}_{n\times n}$, and we can rewrite $q_\phi$ as a distribution over $\mathcal{P}_{n\times n}$ such that $q_\phi(f_{n}(\mathbf z) | \mathbf y, \mathbf x) = q_\phi(\mathbf z | \mathbf y, \mathbf x)$.

Next, we need to model the distribution of $q_\phi(\cdot | \mathbf y, \mathbf x)$. Inspired by \citep{Mena2018sinkhorn}, we model $q_\phi(\cdot | \mathbf y, \mathbf x)$ as a Gumbel-Matching distribution $\mathcal{G.M.}(X)$ over $\mathcal{P}_{n\times n}$, where $X=\phi(\mathbf y, \mathbf x) \in \mathbb{R}^{n\times n}$ is the output of $\phi$. Then for $P \in \mathcal{P}_{n\times n}$,
\begin{gather}
q_\phi(\mathbf z|\mathbf y, \mathbf x)=q_\phi(f_{n}^{-1}(P)|\mathbf y, \mathbf x)=q_\phi(P|\mathbf y, \mathbf x) \propto \exp{\langle X,P\rangle_F}    
\end{gather}
where $\langle X,P\rangle_F=\mathrm{trace}(X^TP)$ is the Frobenius inner product of $X$ and $P$. To obtain samples in $\mathcal{P}_{n\times n}$ from the Gumbel-Matching distribution, \citep{Mena2018sinkhorn} relaxes $\mathcal{P}_{n\times n}$ to $\mathcal{B}_{n\times n}$ by defining the Gumbel-Sinkhorn distribution $\mathcal{G.S.}(X, \tau): \tau > 0$ over $\mathcal{B}_{n\times n}$, and proves that $\mathcal{G.S.}(X, \tau)$ converges almost surely to $\mathcal{G.M.}(X)$ as $\tau \to 0^+$. Therefore, to approximately sample from $\mathcal{G.M.}(X)$, we first sample from $\mathcal{G.S.}(X, \tau)$, then apply Hungarian algorithm \citep{hungarian_alg} to obtain $P \in \mathcal{G.M.}(X)$. The entropy term $H_{q_{\phi}}$ can be approximated as $-\mathcal{D}_{KL}(\mathcal{G.S.}(X, \tau)\;||\;\mathcal{G.S.}(\mathbf{0}, \tau)) + \log n!\;$, and can be further approximated using the technique in Appendix B.3 of \cite{Mena2018sinkhorn}. Further details are presented in Appendix \ref{app:gumbel_sinkhorn_details}.

The Gumbel-Matching distribution allows us to obtain the numerator for the closed form of $q_\phi(\mathbf{z}|\mathbf{y}, \mathbf{x})=q_\phi(f_{n}^{-1}(P)|\mathbf{y}, \mathbf{x})$, which equals $\exp{\langle X,P\rangle_F}$. However, the denominator is intractable to compute and equals $\sum_{P\in \mathcal{P}_{n\times n}} \exp{\langle X,P\rangle_F}$. Upon further examination, we can express it as $\textrm{perm}(\exp(X))$, the matrix permanent of $\exp(X)$, and approximate it using $\textrm{perm}_B(\exp(X))$, its Bethe permanent. We present details about matrix permanent and Bethe permanent along with the proof that the denominator of $q_\phi(\cdot|\mathbf{y}, \mathbf{x})$ equals $\textrm{perm}(\exp(X))$ in Appendix \ref{app:perm_proof}.

\input{algorithms/voi_algorithm}

After we approximate $q_\phi$, we can now optimize $\phi$ using the policy gradient in (\ref{eq:L_phi}). We present a computational diagram of VOI in Figure \ref{fig:voi_diagram}, and a pseudocode of VOI in Algorithm \ref{alg: VOI}. Note that even though latent space $S_n$ is very large and contains $n!$ permutations, in practice, if $p_\theta(\mathbf y, \mathbf z^*| \mathbf x) \ge p_\theta(\mathbf y, \mathbf z| \mathbf x) \; \forall \mathbf z \in S_n$, then $p_\theta(\mathbf y, \mathbf z| \mathbf x)$ tends to increase as the edit distance between $\mathbf z$ and $\mathbf z^*$ decreases. Therefore,  $\phi$ does not need to search over the entire latent to obtain good permutations, making variational inference over $S_n$ feasible.

%% file: figures/voi_architecture.tex
\begin{figure}[t]
    \centering
    \vspace{-0.1cm}
    \includegraphics[width=0.90\textwidth]{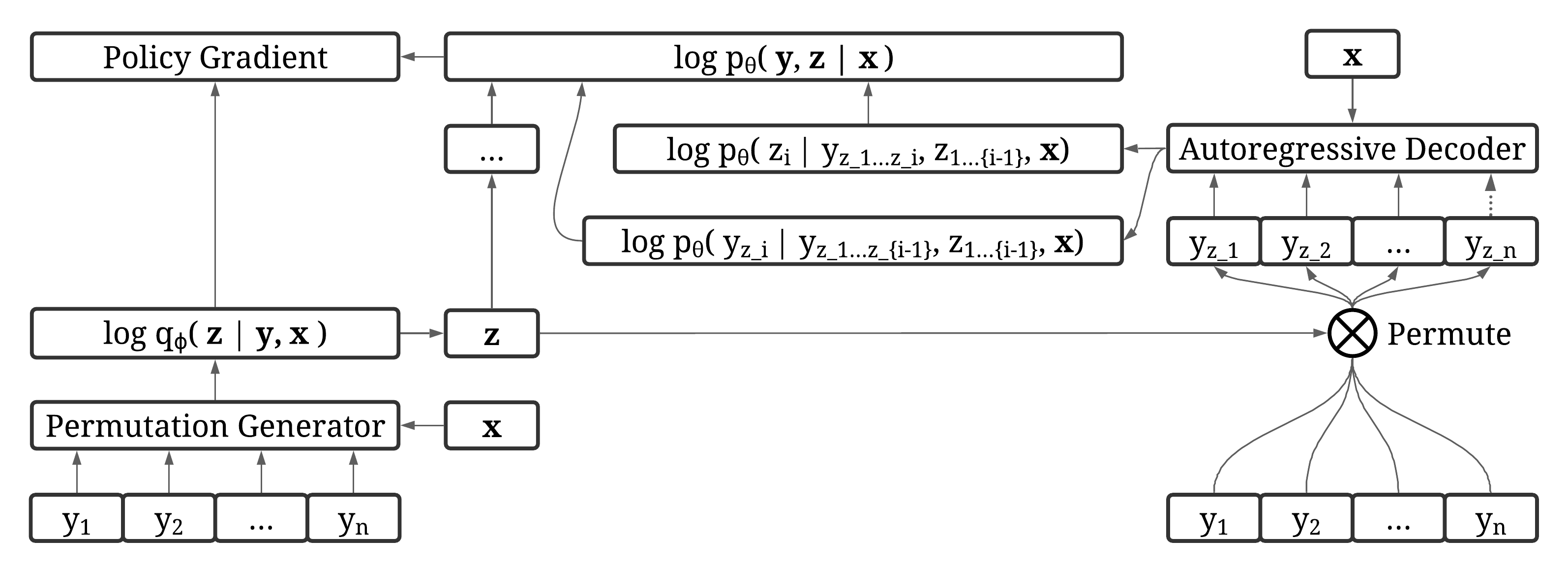}
    \vspace{-0.35cm}
    \caption{Computational diagram for the encoder (left) and decoder (right) that compose Variational Order Inference. We optimize a lower bound on the standard maximum likelihood objective.}
    \vspace{-0.3cm}
    \label{fig:voi_diagram}
\end{figure}

%% file: algorithms/voi_algorithm.tex
\begin{algorithm}[t]
\caption{Variational Order Inference}
\begin{algorithmic}[1]
\State \textbf{Given:} encoder network $\phi$ with learning rate $\alpha_\phi$, decoder network $\theta$ with learning rate $\alpha_\theta$, entropy coefficient $\beta$, batch of training data $(\mathbf X, \mathbf Y) = \{ (\mathbf x_b, \mathbf y_b) \}_{b=1}^{N}$ sampled from dataset $\mathcal{D}$
\State Set gradient accumulators $g_\phi = 0$, $g_\theta = 0$
\For{$(\mathbf x, \mathbf y) \in (\mathbf X, \mathbf Y)$} 
\Comment{In practice, this is done through parallel tensor operations}
\State $X = \phi(\mathbf y, \mathbf x)$
\State Sample $K$ doubly stochastic matrices $B_1,B_2,\dots,B_K \in \mathcal{B}_{n\times n}$ from $\mathcal{G.S.}(X, \tau)$
\State Obtain $ P_1,P_2,\dots,P_K \in \mathcal{P}_{n\times n}$ from $ B_1,B_2,\dots, B_K$ using Hungarian Algorithm
\State 
Obtain latents $\mathbf z_1,\mathbf z_2,\dots,\mathbf z_K = f_{\textrm{len}(\mathbf y)}^{-1}(P_1), f_{\textrm{len}(\mathbf y)}^{-1}(P_2),\dots,f_{\textrm{len}(\mathbf y)}^{-1}(P_K)$
\State $g_\theta = g_\theta + \frac{1}{N\cdot K} \sum_{i=1}^{K} \nabla_{\theta} \log p_\theta(\mathbf y, \mathbf z_i | \mathbf x)$
\State Calculate $\log q_\phi(\mathbf{z}_i|\mathbf{y}, \mathbf{x}) = \langle X,P_i\rangle_F - \log(\textrm{perm}(\exp{(X)}))$ \par
        \hskip\algorithmicindent\quad\quad\quad\quad\quad\quad\quad\quad $\approx \langle X,P_i\rangle_F - \log(\textrm{perm}_B(\exp(X)))$
\State Calculate $b(\mathbf y, \mathbf x) = \frac{1}{K} \sum_{i=1}^K \log p_\theta(\mathbf y, \mathbf z_i | \mathbf x)$
\State $g_\phi = g_\phi + \frac{1}{N\cdot K} \sum_{i=1}^{K} \nabla_\phi \log q_\phi(\mathbf z_i | \mathbf y, \mathbf x) (\log p_\theta(\mathbf y, \mathbf z_i | \mathbf x) - b(\mathbf y, \mathbf x)) + \beta \cdot \nabla_\phi \mathcal{H}_{q_\phi}(\cdot | \mathbf y, \mathbf x)$
\EndFor
\State $\phi = \phi + \alpha_\phi \cdot g_\phi$ 
\State $\theta = \theta + \alpha_\theta \cdot g_\theta$ 
\end{algorithmic}
\label{alg: VOI}
\end{algorithm}

%% file: sections/05_experiments.tex
\section{Experiments}
\label{sec: experiments}

\textbf{Encoder and Decoder Architectures.} We implement \Variational on conditional sequence generation tasks, specifically language modeling tasks. We implement the encoder of VOI as a Transformer with non-causal attention that outputs permutations in one forward pass. The generated permutations then serve as target generation orders for training an insertion-based Transformer language model. A summary of our architectures for conditional sequence generation tasks is illustrated in Figure \ref{fig:voi_transformer}. We would like to note that VOI is also applicable to unconditional sequence generation domains, such as image generation, through different encoder and decoder architectures, which we leave for future work. We would also like to note that ``encoder" and ``decoder" refer to the two networks $\phi$ and $\theta$ in Algorithm \ref{alg: VOI}, respectively, instead of Transformer's encoder and decoder.

For decoder $\theta$, we use the Transformer-InDIGO \citep{gu2019insertion} architecture, which maximizes $p_\theta(\mathbf y, \mathbf z| \mathbf x)$ by alternating token generation and token insertion processes. Note that the ordering $\mathbf{z}$ used to train $\theta$ is obtained through the output of encoder $\phi$ in our approach, instead of through Searched Adaptive Order (SAO) proposed in the Transformer-InDIGO paper, which requires multiple forward passes per batch to obtain a generation order. Once $\mathbf{z}$ is already given, $p_\theta(\mathbf y, \mathbf z| \mathbf x)$ can be optimized in one single pass through teacher forcing. 

For encoder $\phi$, we adopt the Transformer \citep{attallyouneed} architecture. Note that our encoder generates latents based on the entire ground truth target sequence $\mathbf y$. Therefore, it does not need to mask out subsequent positions during attention. We also experiment with different position embedding schemes (see Section \ref{sec:ablation}) and find that Transformer-XL's \citep{transformer_xl} relative positional encoding performs the best, so we replace the sinusoid encoding in the original Transformer.

\textbf{Tasks.} We evaluate our approach on challenging sequence generation tasks: natural language to code generation (NL2Code) \citep{Ling2016LatentPN}, image captioning, text summarization, and machine translation. For NL2Code, we use Django \citep{oda2015ase:pseudogen1}. For image captioning, we use COCO 2017 \citep{lin2015coco}. For text summarization, we use English Gigaword \citep{gigaword2003, gigaword2015}. For machine translation, we use WMT16 Romanian-English (Ro-En).

\textbf{Baselines.} We compare our approach with several pre-defined fixed orders: Left-to-Right (L2R) \citep{wu-etal-2018-beyond}, Common-First (Common) \citep{Ford2018TheIO}, Rare-First (Rare) \citep{Ford2018TheIO}, and Random-Ordering (Random). Here, Common-First order is defined as generating words with ordering determined by their relative frequency from high to low; Rare-First order is defined as the reverse of Common-First order; and Random-Ordering is defined as training with a randomly sampled order for each sample at each time step. 

\textbf{Preprocessing.} For Django, we adopt the same preprocessing steps as described in \citep{gu2019insertion}, and we use all unique words as the vocabulary. For MS-COCO, we find that the baseline in \cite{gu2019insertion} is much lower than commonly used in the vision and language community. Therefore, instead of using Resnet-18, we use the pretrained Faster-RCNN checkpoint using a ResNet-50 FPN backbone provided by TorchVision to extract 512-dimensional feature vectors for each object detection. To make our model spatially-aware, we also concatenate the bounding box coordinates for every detection before feeding into our Transformers' encoder. For Gigaword and WMT, we learn 32k byte-pair encoding (BPE, \cite{bytepairencoding}) on tokenized data.
\begin{figure}[htbp]
    \centering
    \includegraphics[width=0.25\linewidth]{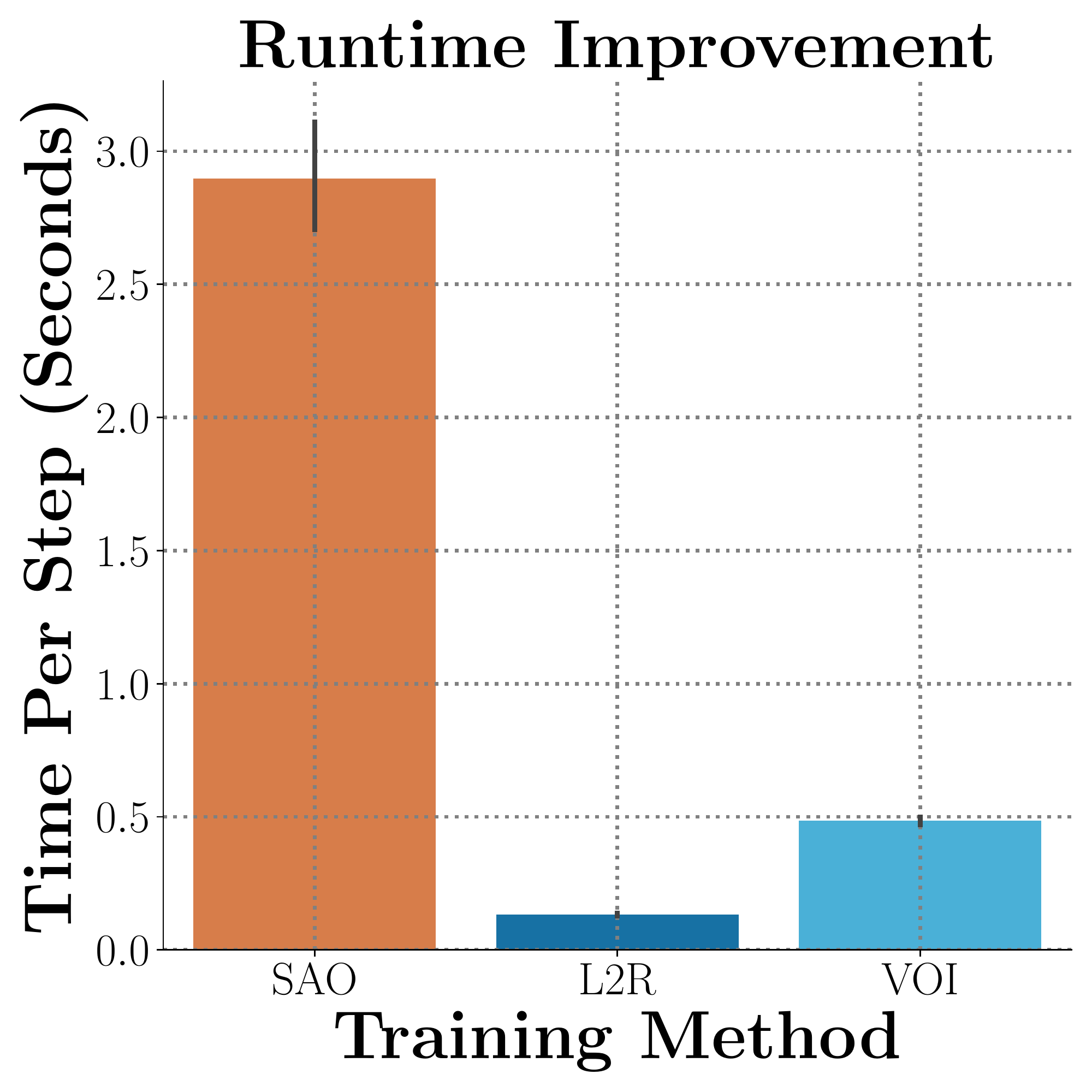}
    \hspace{1cm}
    \includegraphics[width=0.25\linewidth]{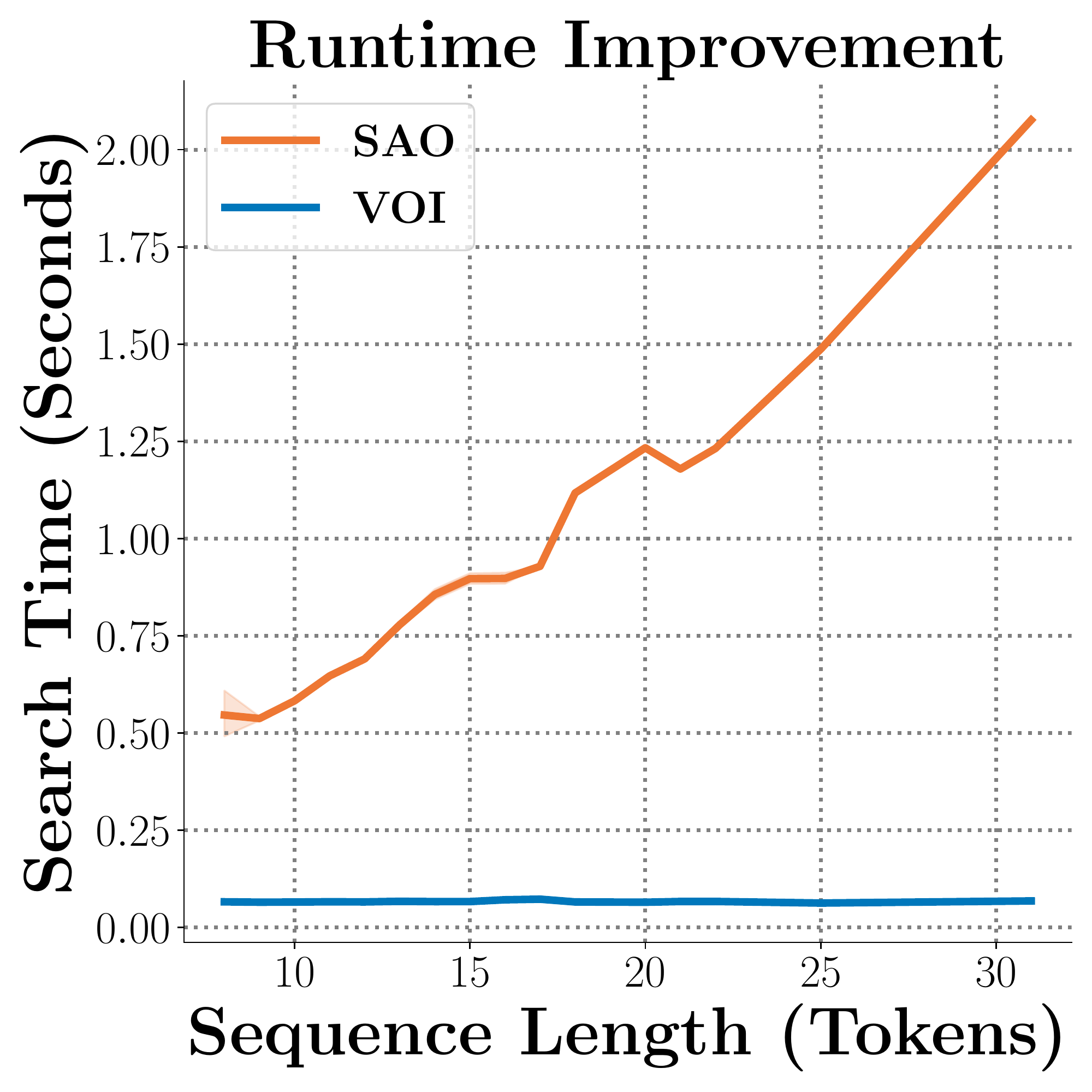}
    \hspace{1cm}
    \includegraphics[width=0.25\linewidth]{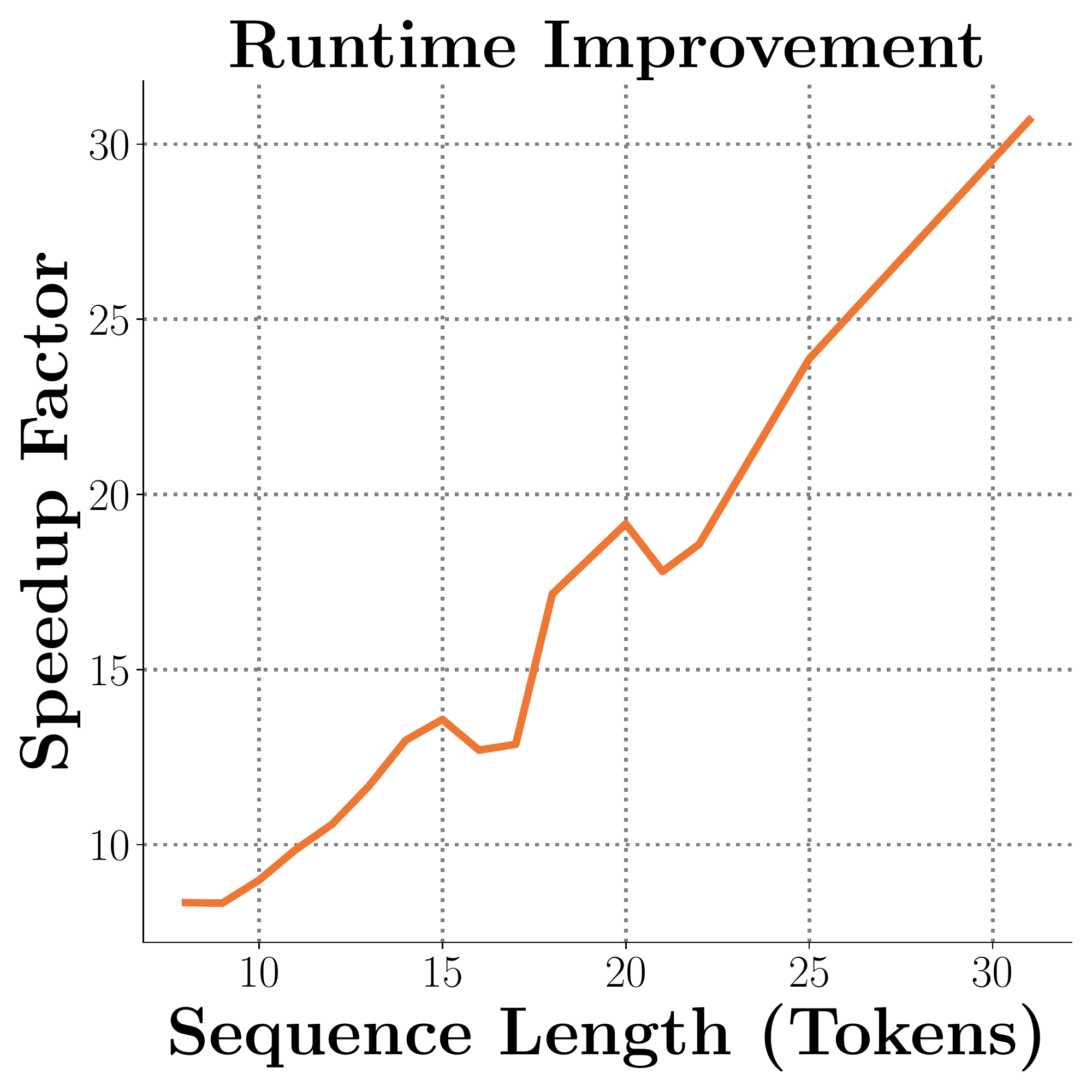}
    \vspace{-0.2cm}
    \caption{\textbf{Runtime performance improvement.} We compare the runtime performance of VOI ($K=4$) with SAO on a single Tesla P100 GPU, in terms of time per training iteration and ordering search time. VOI outputs latent orderings in a single forward pass, and we observe a significant runtime improvement over SAO that searches orderings sequentially. The speedup factor linearly increases with respect to the sequence length.}
    \label{fig:time_stats}
    \vspace{-0.2cm}
\end{figure}

\input{tables/results}

\textbf{Training.} For our decoder, we set $d_{\textrm{model}}=512$, $d_{\textrm{hidden}}=2048$, 6 layers for both Transformer's encoder and decoder, and 8 attention heads.
This is the same model configuration as Transformer-Base \citep{attallyouneed} and as described in \cite{gu2019insertion}. Our encoder also uses the same configuration. For our model trained with \Variational, we sample $K=4$ latents for each training sample for Django, COCO, and Gigaword and $K=3$ latents for WMT (due to computational resource constraints, we were unable to set a higher $K$ for WMT). An ablation on the choices of $K$ on a small dataset is presented in Section \ref{sec:ablation}. For WMT, many previous works on nonsequential orderings \citep{DBLP:conf/icml/SternCKU19} and nonautoregressive sequence generation \citep{levenhstein_transformer} have found sequence-level knowledge distillation \citep{kim-rush-2016-sequence} helpful. Therefore, we first train the L2R model on the original WMT corpus, then create a new training corpus using beam search. We find that this improves the BLEU of VOI model by about 2.0. Even though the training set changed, the orderings learned by VOI are very similar to the ones trained on the original corpus. More detailed training processes are described in Appendix \ref{app:hyperparams}. 

During training, our encoder and decoder are optimized in one single pass per batch. If we let $N$ denote the batch size, $l$ denote the length of each target sequence, and $d$ denote the size of hidden vector, then one single forward pass of our model has computation complexity $O(NKdl^2)$, while Transformer-InDIGO trained with SAO has total complexity $O(Ndl^3)$. Since $K \ll l$ in general, our algorithm has better theoretical computational complexity during training. During evaluation, we only keep the decoder to iteratively generate the next position and token, which is as efficient as any standard fixed-order autoregressive models.

We also empirically compare VOI's runtime with that of SAO and fixed-order baselines (e.g. L2R). We implement SAO as described in \cite{gu2019insertion}. We test the runtime on a single GPU in order to accurately measure the number of ops required. For training speed per iteration, we use a batch size of 8. For ordering search time, we use a batch size of 1 to avoid padding tokens in the input for accurate measure. We observe that VOI is significantly faster than SAO, which searches orderings sequentially. In practice, as we distribute VOI across more GPUs, the $K$ factor in the runtime is effectively divided by the number of GPUs used (if we ignore the parallelization overhead), so we can achieve further speedups.

\textbf{Results.} 
We compare VOI against predefined orderings along with Transformer-InDIGO trained with SAO in Table \ref{tab:results}. The metrics we used include BLEU-4 \citep{papineni-etal-2002-bleu}, Meteor \citep{denkowski:lavie:meteor-wmt:2014}, Rouge \citep{lin-2004-rouge}, CIDEr \citep{DBLP:conf/cvpr/VedantamZP15}, and TER \citep{Snover06astudy}. The "accuracy" reported for Django is defined as the percentage of perfect matches in code generation. Our results illustrate consistently better performance across fixed orderings. Most notably, CIDEr for MS-COCO, BLEU for Django, and Rouge-1 for Gigaword reveal the largest improvements in performance. 

%% file: tables/results.tex
\begin{table}[htbp]
\centering
\footnotesize
\setlength{\tabcolsep}{2.1pt}

\begin{tabular}{lcccccccccccc} \toprule
    
    \multirow{2}{*}{Order} & \multicolumn{4}{c}{MS-COCO} & \multicolumn{2}{c}{Django} & \multicolumn{3}{c}{Gigaword} & \multicolumn{3}{c}{WMT16 Ro-En}\\
    & {$\text{BLEU}$} & {$\text{Meteor}$} & {$\text{R-L}$} & {$\text{CIDEr}$} & {$\text{BLEU}$} & {$\text{Accuracy}$} & {$\text{R-1}$} & {$\text{R-2}$} & {$\text{R-L}$} &{BLEU$\uparrow$} &{Meteor$\uparrow$} & {TER$\downarrow$} \\ \midrule
    $\text{InDIGO - SAO}$ \footnote[1]{For InDIGO-SAO, we report the results on COCO and Django trained using our own implementation. We did not attempt SAO on Gigaword or WMT due to the large dataset sizes, which can take a very long time to train. For WMT, we report the SAO result as in the original paper, and we follow their evaluation scheme.}  & 29.3 & 24.9  & 54.5  & 92.9 & 42.6 & 32.9 & \longdash & \longdash & \longdash & 32.5 & 53.0 & \textbf{49.0} \\    \midrule
    $\text{Ours - Random}$  & 28.9 & 24.2 & 55.2 & 92.8 & 21.6 & 26.9 & 30.1 & 11.6 & 27.6 & 20.3 & 43.5 & 62.0 \\
    $\text{Ours - L2R}$  & 30.5  & 25.3 & 54.5  & 95.6 & 40.5 & 33.7 & 35.6 & 17.2 & 33.2 & 32.7 & 54.4 & 50.2 \\
    $\text{Ours - Common}$  & 28.0  & 24.8 & 55.5  & 90.3 & 37.1 & 29.8 & 33.9 & 15.0 & 31.1 & 28.2 & 50.8 & 53.1 \\
    $\text{Ours - Rare}$  & 28.1  & 24.5 & 52.9  & 91.4 & 31.1 & 27.9 & 34.1 & 15.2 & 31.3 & 26.4 & 48.5 & 55.1 \\ \midrule
    $\text{Ours - VOI}$  & \textbf{31.0}  & \textbf{25.7} & \textbf{56.0}  & \textbf{100.6} & \textbf{45.9} & \textbf{34.5} & \textbf{36.6} & \textbf{17.6} & \textbf{34.0} & \textbf{32.9} & \textbf{54.6} & 49.3 \\ \bottomrule
    \bottomrule
    \vspace{-0.3cm}
\end{tabular}


\caption{Results of MS-COCO, Django, Gigaword, and WMT with fixed orders (L2R, Random, Common, Rare) as baseline. Here, R-1, R-2, and R-L indicate ROUGE-1, ROUGE-2, and ROUGE-L, respectively. For TER, lower is better; for all other metrics, higher is better. ``\longdash" = not reported.} 


\vspace{-0.1cm}
\label{tab:results}
\end{table}

%% file: sections/06_order_analysis.tex
\section{Order Analysis}

In this section, we analyze the generation orders learned by \Variational on a macro level by comparing the similarity of our learned orders with predefined orders defined in Section \ref{sec: experiments}, and on a micro level, by inspecting when the model generates certain \textit{types} of tokens.

\input{figures/global_stats}
\subsection{Understanding The Model Globally}

We find that prior work \citep{gu2019insertion, welleck2019non, gu2018nonautoregressive} tends to study autoregressive orders by evaluating performance on validation sets, and by visualizing the model's generation steps. We provide similar visualizations in Appendix \ref{app:visualization_app}. However, this does not merit a quantitative understanding of the \textit{strategy} that was learned. We address this limitation by introducing methodology to quantitatively study decoding strategies learned by non-monotonic autoregressive models. We introduce \textit{Normalized Levenshtein Distance} and \textit{Order Rank Correlation}, to measure similarity between decoding strategies. Given two generation orders $\mathbf w, \mathbf z \in S_n$ of the same sequence $\mathbf y$, where $n$ is the length of $\mathbf y$, we define the \textit{Normalized Levenshtein Distance}.
\begin{gather}
    \mathcal{D}_{NLD} \left( \mathbf w, \mathbf z \right) = \text{lev} \left(\mathbf w, \mathbf z \right) / n \\
    \text{lev} \left(\mathbf w, \mathbf z \right) =
  1 + \min \left\{ \text{lev} \left(\mathbf w_{1:}, \mathbf z \right),
  \text{lev} \left(\mathbf w, \mathbf z_{1:} \right), \text{lev} \left(\mathbf w_{1:}, \mathbf z_{1:} \right) \right\}
\end{gather}
The function $\text{lev} \left(\mathbf w, \mathbf z \right)$ is the Levenshtein distance, and $z_{1:}$ removes the first element of $z$. This metric has the property that a distance of $0$ implies that two orders $\mathbf w$ and $\mathbf z$ are the same, while a distance of $1$ implies that the same tokens appear in distant locations in $\mathbf w$ and $\mathbf z$. Our second metric \textit{Order Rank Correlation}, is the Spearman's rank correlation coefficient between $\mathbf w$ and $\mathbf z$.
\begin{gather}
    \mathcal{D}_{ORC} \left( \mathbf w, \mathbf z \right) = 1 - 6 \cdot {\textstyle\sum}_{i = 0}^{n} \left( \mathbf w_{i} - \mathbf z_{i} \right) / \left( n^{3} - n \right)
\end{gather}
A correlation of $1$ implies that $\mathbf w$ and $\mathbf z$ are the same; a correlation of $-1$ implies that $\mathbf w$ and $\mathbf z$ are reversed; and a correlation of $0$ implies that $\mathbf w$ and $\mathbf z$ are not correlated. In Figure \ref{fig:global_stats}, we apply these metrics to analyze our models learnt through \Variational.

\input{figures/local_stats}
\paragraph{Discussion.} The experiment in Figure~\ref{fig:global_stats} confirms our model's behavior is not well explained by predefined orders. Interestingly, as the generated sequences increase in length, the \textit{Normalized Levenshtein Distance} decreases, reaching a final value of 0.57, indicating that approximately half of the tokens are already arranged according to a left-to-right generation order. However, the \textit{Order Rank Correlation} barely increases, so we can infer that while individual tokens are close to their left-to-right generation index, their relative ordering is not preserved. Our hypothesis is that certain phrases are generated from left-to-right, but their arrangement follows a \textit{best-first} strategy.

\subsection{Understanding The Model Locally}

To complement the study of our model at a global level, we perform a similar study on the micro token level. Our hope is that a per-token metric can help us understand if and when our \Variational is adaptively choosing between left-to-right and rare-first order. We also hope to evaluate our hypothesis that \Variational is following a \textit{best-first} strategy.

\paragraph{Discussion.} The experiment in Figure~\ref{fig:local_stats_gen} demonstrates that \Variational prefers decoding \textit{descriptive} tokens first---such as nouns, numerals, adverbs, verbs, and adjectives. In addition, the unknown part of speech is typically decoded first, and we find this typically corresponds to special tokens such as proper names. Our model appears to capture the \textit{salient} content first, which is illustrated by nouns ranking second in the generation order statistics. For image captioning, nouns typically correspond to focal objects, which suggests our model has an object-detection phase. Evidence of this phase supports our previous hypothesis that a \textit{best-first} strategy is learned. 

\subsection{Understanding The Model Via Perturbations}

\begin{wrapfigure}[10]{r}{0.55\textwidth}
  \vspace{-0.45cm}
  \centering
  \includegraphics[width=0.55\textwidth, trim=2.5cm 0.0cm 2.5cm 1.0cm, clip]{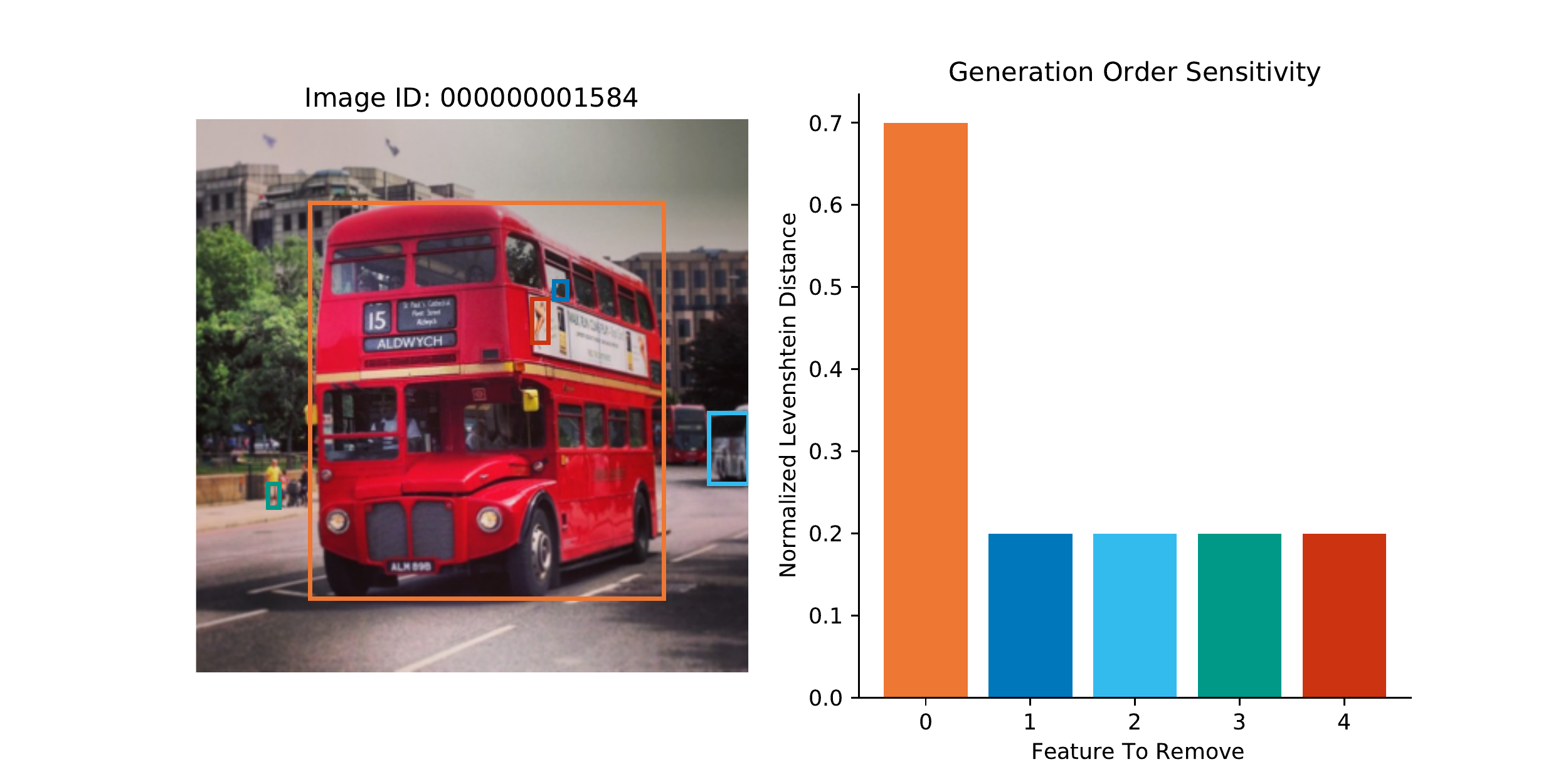}
  \label{fig:perturb_stats}
\end{wrapfigure}

In this section, we study the question: to what extent is the generation order learned by \Variational dependent on the content of the conditioning variable $\mathbf x$? This question is important because simply knowing that our model has learned a \textit{best-first} does not illuminate whether that strategy depends only on the target tokens $\mathbf y$ being generated, or if it also depends on the content of $\mathbf x$. An adaptive generation order should depend on both. 

\paragraph{Discussion.} In this experiment, we first obtain a sequence $\mathbf y$ generated by our VOI given the source image $\mathbf x$. We then freeze $\mathbf y$, which allows the model to infer a new generation order for $\mathbf y$ when different features of $\mathbf x$ are removed. The right figure shows that for a particular case, removing a single region-feature (feature number 0, which corresponds to the bus) from $\mathbf x$ changes the model-predicted generation order by as much as 0.7 \textit{Normalized Levenshtein Distance}.  These results confirm that our model appears to learn an \textit{adaptive} strategy, which depends on both the tokens $\mathbf y$ being generated and the content of the conditioning variable $\mathbf x$, which is an image in this experiment. 

%% file: figures/global_stats.tex
\begin{figure}[h]
    \centering
    \includegraphics[width=0.93\linewidth]{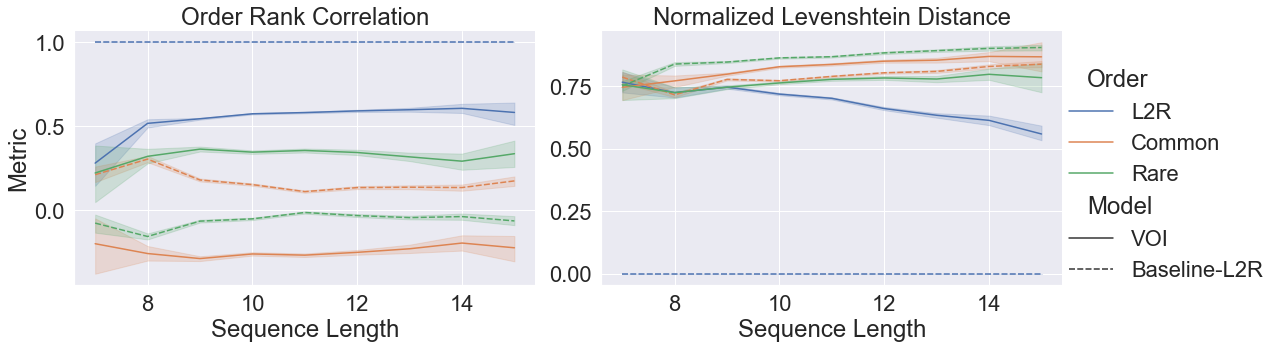}
    \vspace{-0.25cm}
    \caption{\textbf{Global statistics for learned orders.} We compare metrics as a function of the sequence length of generated captions on the COCO 2017 validation set. On the left, we compare orders learned with \Variational to a set of predefined orders (solid lines) using \textit{Order Rank Correlation}. As a reference, we provide the \textit{Order Rank Correlation} between L2R and the same set of predefined orders (dashed lines). In the right plot, with identical setup, we measure \textit{Normalized Levenshtein Distance}.
    We observe that \Variational favors left-to-right decoding above the other predefined orders---this corresponds to the blue lines. However, with a max \textit{Order Rank Correlation} of 0.6, it appears left-to-right is not a perfect explanation. The comparably high \textit{Order Rank Correlation} of 0.3 with rare-tokens-first order suggests a complex strategy.}
    \vspace{-0.3cm}
    \label{fig:global_stats}
\end{figure}

%% file: figures/local_stats.tex
\begin{figure}[!htbp]
    \centering
    \includegraphics[width=0.90\linewidth]{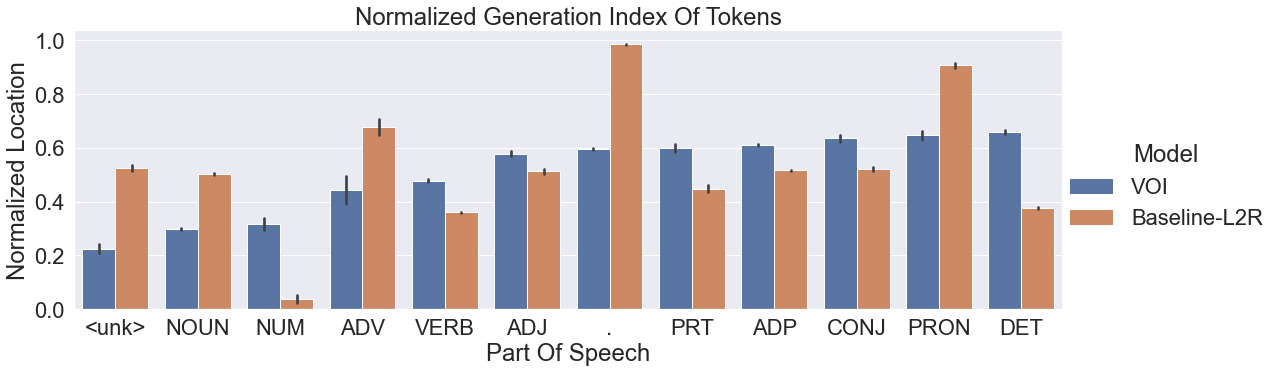}
    \vspace{-0.2cm}
    \caption{\textbf{Local statistics for learned orders.} In this figure, we evaluate the normalized generation indices for parts of speech in predicted captions on the COCO 2017 validation set. The normalized generation index is defined as the absolute generation index of a particular token, divided by the final length of predicted sequence. Parts of speech (details in Appendix \ref{app:partofspeech}) are sorted in ascending order of average normalized location. We observe that \textit{modifier} tokens, such as ``the'', tend to be decoded last, while \textit{descriptive} tokens, such as nouns and verbs, tend to be decoded first. }
    \vspace{-0.5cm}
    \label{fig:local_stats_gen}
\end{figure}

%% file: sections/07_ablations.tex
\section{Ablation Studies}
\label{sec:ablation}

In Section \ref{sec: experiments}, we introduced the specific encoder and decoder architectures we use for conditional sequence generation tasks. In this section, we present ablation studies to support the architecture design of our encoder and modeling $q_\phi$ with Gumbel-Matching distribution.

\begin{wraptable}{r}{7.5cm}
\centering
\small
\setlength{\tabcolsep}{2.1pt}
\renewcommand{\arraystretch}{1.07}
\vspace{-0.78cm}
\caption{Normalized Levenshtein Distance between the ordering learnt by the encoder and the ground truth ordering, under different positional encodings (enc) and modeling distributions of $q_\phi$ (distrib).}
\vspace{0.2cm}
\begin{tabular}{c|c|c}
\hline
Enc \textbackslash~ Distrib                & Gumbel-Matching   & Plackett-Luce \\ \hline
Sinusoid               & 0.40  &  0.62 \\
Sinusoid + Pos Attn           & 0.42  & 0.58  \\
Relative                & 0.38  & 0.53  \\
XL-Relative            & \textbf{0.25}  & 0.57  \\ \hline
\end{tabular}
\label{tab:lev_encoder_emb}
\vspace{-0.35cm}
\end{wraptable}

We consider 4 different positional encoding schemes for the encoder Transformer $\phi$: the sinusoid encoding in the original Transformer \citep{attallyouneed}, the sinusoid encoding with positional attention module \citep{gu2018nonautoregressive}, the relative positional encoding in \cite{self_attention_relative_pos}, and the relative positional encoding proposed in Transformer-XL \citep{transformer_xl}. Besides modeling $q_\phi(\cdot | \mathbf x, \mathbf y)$ as Gumbel-Matching distribution and using Bethe permanent to approximate its denominator, we also consider modeling using Plackett-Luce distribution \citep{plackett, luce59} and sample using techniques recently proposed in \cite{plackett_luce1903}. Plackett-Luce distribution has tractable density, so we can compute the exact $q_\phi$ efficiently without using approximation techniques.

To analyze the encoder's ability to learn autoregressive orderings, we first train a decoder with Common-First order on one batch of MS-COCO until it perfectly generates each sentence. We then fix the decoder and initialize an encoder. We train the encoder for 15k gradient steps using the procedure in Algorithm \ref{alg: VOI} to recover the ground truth Common-First order, and we report the final Normalized Levenshtein Distance against the ground truth in Table \ref{tab:lev_encoder_emb}. We observe that modeling $q_\phi$ with Gumbel-Matching distribution significantly outperforms modeling with Plackett-Luce, despite the former requiring denominator approximation. We also observe that under Gumbel-Matching modeling distribution, the relative position encoding in Transformer-XL significantly outperforms other encoding schemes. Thus we combine these two techniques in our architecture design.

\begin{wraptable}{r}{6.2cm}
\centering
\small
\setlength{\tabcolsep}{2.1pt}
\renewcommand{\arraystretch}{1.07}
\vspace{-0.76cm}
\caption{Normalized Levenshtein Distance between the encoder ordering and the ground truth with respect to the choice of $K$.}
\vspace{0.2cm}
\begin{tabular}{c|ccccc}
\hline
$K$ & 2 & 3 & 4 & 10 & 20 \\
\hline
$\mathcal{D}_{NLD}$ & 0.31 & 0.28 & 0.25 & 0.21 & 0.21 \\ \hline
\end{tabular}
\label{tab:ablation_K}
\vspace{-0.25cm}
\end{wraptable}

In addition, we analyze how choices of $K$, the number of latents per training sample, affects model performance. We use the same setting as above and apply Transformer-XL relative position encoding, and we report the results in Table \ref{tab:ablation_K}. We observe that the encoder more accurately fits to the ground truth order as $K$ increases, until a value of around 10. Since a very large $K$ can slow the model down while only bringing marginal improvements, we find a good choice of $K$ to be between 4 and 10.

%% file: sections/08_conclusion.tex
\section{Conclusion}

We propose, to our best knowledge, the first unsupervised learner that learns high-quality autoregressive orders through fully-parallelizable end-to-end training without domain-specific tuning. We propose a procedure named \Variational that uses the Variational Lower Bound with the space of autoregressive orderings as latent. Building on techniques in combinatorical optimization, we develop a practical policy gradient algorithm to optimize the encoder of the variational objective, and we propose an encoder architecture that conditions on training examples to output autoregressive orders. Empirical results demonstrate that our model is capable of discovering autoregressive orders that are competitive with or even better than fixed and predefined orders. In addition, the global and local analysis of the orderings learned through \Variational suggest that they resemble a type of \textit{best-first} generation order, characterized by prioritizing the generation of \textit{descriptive} tokens and deprioritizing the generation of \textit{modifier} tokens.

%% file: sections/09_appendix.tex
\section*{\LARGE{Appendix}}
\appendix
\section{Gumbel-Matching Distribution and its Sampling}
\label{app:gumbel_sinkhorn_details}

In Section \ref{sec:vans}, We model the distribution of $q_\phi(\cdot | \mathbf y, \mathbf x)$ as a Gumbel-Matching distribution $\mathcal{G.M.}(X)$ over $\mathcal{P}_{n\times n}$, where $X=\phi(\mathbf y, \mathbf x) \in \mathbb{R}^{n\times n}$ is the latent output. 

To obtain samples in $\mathcal{P}_{n\times n}$ from the Gumbel-Matching distribution, \cite{Mena2018sinkhorn} relaxes $\mathcal{P}_{n\times n}$ to $\mathcal{B}_{n\times n}$ by defining the Gumbel-Sinkhorn distribution $\mathcal{G.S.}(X, \tau): \tau > 0$ over $\mathcal{B}_{n\times n}$. Here we reproduce the following definitions and theorems with similar notations from \cite{sinkhorn1964} and \cite{Mena2018sinkhorn}:

\begin{definition} 
Let $X \in \mathbb R_{n\times n}$ and $A \in \mathbb R_{n\times n}^+$. The \textbf{Sinkhorn Operator} $S$ is defined as 
\begin{gather}\label{eq:sinkhorn_op}
    \mathcal{T}_r(A) = A \oslash{(A\mathbf{1}_n\mathbf{1}_n^T)} \\
    \mathcal{T}_c(A) = A \oslash{(\mathbf{1}_n\mathbf{1}_n^T A)} \\    
    \mathcal{T}(A) = \mathcal{T}_r(\mathcal{T}_c(A)) \\
    S(X) = \lim_{n \to \infty} \mathcal{T}^n(\exp(X))
\end{gather}
\end{definition}

Here, $\oslash$ is the element-wise division between two matrices, and $\mathcal{T}_r$ and $\mathcal{T}_c$ are row and column normalizations of a non-negative matrix, respectively. Therefore, iteratively applying $\mathcal{T}$ is equivalent to iteratively normalizing a non-negative matrix by column and row.

\begin{theorem} \textnormal{\citep{sinkhorn1964}} The range of $S$ is $\mathcal{B}_{n\times n}$.
\end{theorem}

\begin{theorem} \textnormal{\citep{Mena2018sinkhorn}}
Let $X \in \mathbb R_{n\times n}, \tau > 0$. The \textbf{Gumbel-Sinkhorn distribution} $\mathcal{G.S.}(X, \tau)$ is defined as follows:
\begin{gather}\label{eq:sinkhorn_op}
    \mathcal{G.S.}(X, \tau) = S(\frac{X + \epsilon}{\tau})
\end{gather}
where $\epsilon$ is a matrix of i.i.d. standard Gumbel noise. Moreover, $\mathcal{G.S.}(X, \tau)$ converges almost surely to $\mathcal{G.M.}(X)$ as $\tau \to 0^+$. 
\end{theorem}

To approximately sample from $\mathcal{G.M.}(X)$, we first sample from $\mathcal{G.S.}(X, \tau)$. Even though theoretically, $\mathcal{T}$ needs to be applied infinite number of times to obtain a matrix in $\mathcal{B}_{n\times n}$, \cite{Mena2018sinkhorn} reports that $20$ iterations of $\mathcal{T}$ are enough in practice. We find that in our experiments, $20$ iterations are not enough to obtain a matrix in $\mathcal{B}_{n\times n}$, but $100 - 200$ iterations are enough. After we obtain the matrix in $\mathcal{B}_{n\times n}$, we apply Hungarian algorithm \citep{hungarian_alg} to obtain $P \in \mathcal{G.M.}(X)$.

\section{Matrix Permanent and its Approximation with Bethe Permanent}
\label{app:perm_proof}
In this section, we present details about matrix permanent and bethe permanent, which we use as an approximation to the denominator of $q_\phi(\cdot | \mathbf y, \mathbf x)$.
\begin{definition}
Let $A \in \mathbb{R}_{n\times n}$. The \textbf{permanent} of $A$ is defined as follows:
\begin{gather}\label{eq:permanent}
    \textrm{perm}(A) = \sum_{\sigma \in S_n} \prod_{i=1}^n A_{i, \sigma_i}
\end{gather}
\end{definition}

\begin{theorem}\label{theorem: perm}
The denominator of $q_\phi(\cdot|\mathbf{y}, \mathbf{x})$ equals $\textrm{perm}(\exp(X))$.
\end{theorem}
\begin{proof}
\begin{align*}
    \sum_{P\in \mathcal{P}_{n\times n}} \exp{\langle X,P\rangle_F}
    &= \sum_{\sigma \in S_n} \exp({\sum_{i=1}^n X_{i,\sigma(i)}}) \\
    &= \sum_{\sigma \in S_n} \prod_{i=1}^n (\exp(X))_{i,\sigma(i)} \\
    &= \textrm{perm}(\exp(X))
\end{align*}
\end{proof}

\begin{definition}\label{def:bethe}
\textnormal{\citep{bethe2011, bethe2018}}
Let $A \in \mathbb{R}_{n\times n}^+$. The \textbf{bethe permanent} of $A$ is defined as follows:
\begin{gather}\label{eq:permanent}
    \textrm{perm}_B (A) = \exp{(\max_{\gamma \in \mathcal{B}_{n\times n}}\sum_{i, j} (\gamma_{i,j}\log{A_{i,j}} - \gamma_{i,j} \log{\gamma_{i,j}} + (1-\gamma_{i,j}) \log{(1-\gamma_{i,j})}))}
\end{gather}
\end{definition}

\begin{theorem} \textnormal{\citep{bethe2018}}
Let $A \in \mathbb{R}_{n\times n}^+$. Then, $\sqrt{2}^{-n}\textrm{perm}(A) \le \textrm{perm}_B (A) \le \textrm{perm}(A)$.
\end{theorem}

The $\gamma$ in Definition \ref{def:bethe} can be calculated using the message passing algorithm in Lemma 29 of \cite{bethe2011}. An efficient implementation has recently been introduced in Appendix C of \cite{mena2020sinkhornvariational}. Therefore, we can use $\textrm{perm}_B (\exp{(X)})$ to approximate the denominator of $q_\phi(\cdot|\mathbf{y}, \mathbf{x})$, and we can then use policy gradient to compute $\nabla_\phi L(\phi)$ in Equation (\ref{eq:L_phi}).





    

\section{Detailed Training Process and Hyperparameter Settings}
\label{app:hyperparams}

For all experiments, we apply dropout = 0.1 \citep{dropout} and label smoothing = 0.1. We apply Adam Optimizer \citep{adamopt} with $\beta_1=0.99, \beta_2=0.999$ for MS-COCO, and $\beta_1=0.99, \beta_2=0.98$ for all other tasks. For baseline experiments, we use a batch size of 64 for Django and MS-COCO, and 128 for Gigaword and WMT. We apply 8000 warmup steps and decrease the learning rate linearly from 1e-4 to zero afterwards. We train the baseline until the performance plateaus.
 
For our VOI model, we train on Django for a total of 300 epochs (150k gradient steps), MS-COCO for 20 epochs (350k gradient steps), Gigaword for 16 epochs (1M gradient steps), and WMT16 Ro-En for 110 epochs (1.2M gradient steps). We use a batch size of 32 for MS-COCO and Django, 50 for Gigaword, and 54 for WMT. We sample $K=4$ latents per training sample for the first three datasets, and $K=3$ for WMT. Due to constraints in computational resource (we did not have access to industry-level infrastructures to utilize a large amount of GPU memory on multiple machines), we were unable to scale WMT to larger batch size and larger $K$ (as the WMT model with batch size of 54 and $K=3$ would already take up 90G of GPU memory). We also did not experiment with a large batch size (e.g. 128+) for other datasets. We leave the discovery of better training schemes and exhaustive hyperparameter search for future work.

We set the initial decoder autoregressive language model learning rate to be 5e-5 and the encoder permutation generator learning rate to be 5e-6. To facilitate encoder learning, we hope to propagate information from the decoder to the encoder. Therefore, we train the encoder permutation generator and the decoder autoregressive language model with a shared Transformer encoder and a shared token embedding for about the first 15-20\% of steps (i.e. 4 epochs for COCO, 50 epochs for Django, 3 epochs for Gigaword, and 20 epochs for WMT). We then separate the Transformer encoders and the token embeddings, and we add a token embedding cosine alignment loss between the encoder permutation generator and the decoder autoregressive language model with a coefficient of 100.0 for the rest of the training steps. When the Transformer encoder is shared, we set the entropy coefficient $\beta=0.5$ for non-WMT tasks and $0.3$ for WMT. 

After we separate the Transformer encoders, for MS-COCO, we anneal $\beta$ with a log-linear schedule from 0.5 to 0.1 for 10 epochs. We then anneal $\beta$ from 0.1 to 0.03 for 5 epochs and decrease the learning rate to (2e-5, 2e-6) for the decoder and the encoder respectively, as the encoder starts to sample very similar orderings for a single training data. We observe that training either VOI or the fixed ordering models for too long leads to overfitting. Finetuning VOI with the encoder fixed does not help and causes the performance to slightly drop. 

For Django, we set the learning rates to be (3e-5, 3e-6). We log-anneal $\beta$ from 0.5 to 0.002 for 200 epochs. We then fix the encoder and finetune the decoder autoregressive model for 50 epochs, with learning rate linearly decreasing to zero and a batch size of 64.

For Gigaword, we anneal $\beta$ log-linearly from 0.5 to 0.03 in 8 epochs (500k gradient steps). We then fix the encoder and fine-tune the decoder with a batch size of 128 for 5 epochs with learning rate linearly decreasing from 7e-5 to 0. We observe that, compared to COCO and Django, this finetuning step significantly improves VOI's performance and raises the ROUGE score by around 1.5 to 2.0. 

For WMT, we anneal $\beta$ log-linearly from 0.3 to 0.01 in 40 epochs. We then anneal $\beta$ to 5e-4 in 30 epochs with learning rate decrased from (5e-5, 5e-6) to (2e-5, 2e-6) as the encoder starts sampling similar permutations. We then fix the encoder and finetune the decoder with a batch size of 128 for 20 epochs with learning rate linearly decreasing from 2e-5 to 0. We observe that this finetuning step also significantly benefits VOI's performance and improves the BLEU score by around 1.5 points.

\section{Parts Of Speech Mappings}
\label{app:partofspeech}
The parts of speech used in our Order Analysis section correspond to the NLTK Universal Tagset. In the below table, we provide mappings for the tag identifiers used in our main paper. More information about the specific NLTK tags can be found at the following url: \url{http://www.nltk.org/book/ch05.html}.

\begin{table}[!htbp]
    \centering
    \begin{tabular}{l|ll}
    \textbf{Tag}	& \textbf{Meaning}	& \textbf{English Examples} \\
    \hline
    ADJ	& adjective	& new, good, high, special, big, local \\
    ADP	& adposition	& on, of, at, with, by, into, under \\
    ADV	& adverb	& really, already, still, early, now \\
    CONJ	& conjunction	& and, or, but, if, while, although \\
    DET	& determiner, article	& the, a, some, most, every, no, which \\
    NOUN	& noun	& year, home, costs, time, Africa \\
    NUM	& numeral	& twenty-four, fourth, 1991, 14:24 \\
    PRT	& particle	& at, on, out, over per, that, up, with \\
    PRON	& pronoun	& he, their, her, its, my, I, us \\
    VERB	& verb	& is, say, told, given, playing, would \\
    .	& punctuation marks	& . , ; ! \\
    X	& other	& ersatz, esprit, dunno, gr8, univeristy \\
    \end{tabular}
    \caption{NLTK Universal Tagset.}
    \label{tab:tagset}
\end{table}

\clearpage
\section{Visualizations of Sequence Generation}

\subsection{COCO}

We visualize the generation order inferred by \Variational for COCO. Sequences are generated using beam search over both tokens and their insertion positions, using a beam size of 3. Bounding boxes that correspond to region-features calculated using bottom-up attention are superimposed on the image, with an opacity value proportional to the magnitude of their softmax attention value in the final cross-attention layer in the language model.

\begin{figure}[h]
    \centering
    \includegraphics[width=0.8\linewidth, trim=3.0cm 1.75cm 1.5cm 1.5cm, clip]{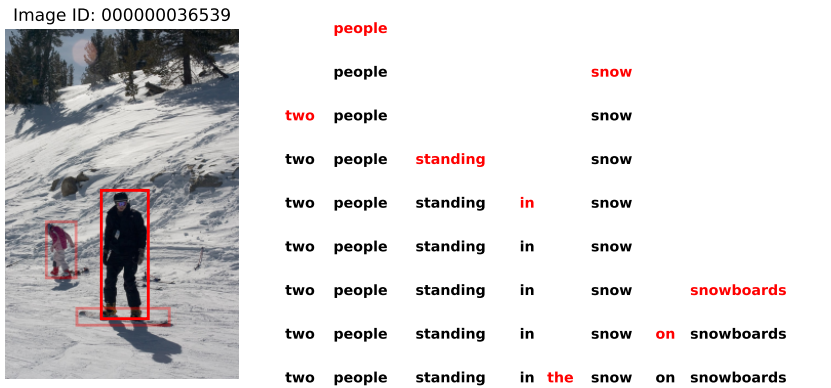}
    \caption{Generation order inferred by \textbf{Variational Order Inference}. Without supervision over its generation order, nor a domain-specific initialization, nor a prior to aid learning, the model learns an adaptive strategy that prioritizes object names---in this case, \textit{people} and \textit{snow}.}
    \label{fig:000000036539_demo}
\end{figure}

\begin{figure}[h]
    \centering
    \includegraphics[width=\linewidth]{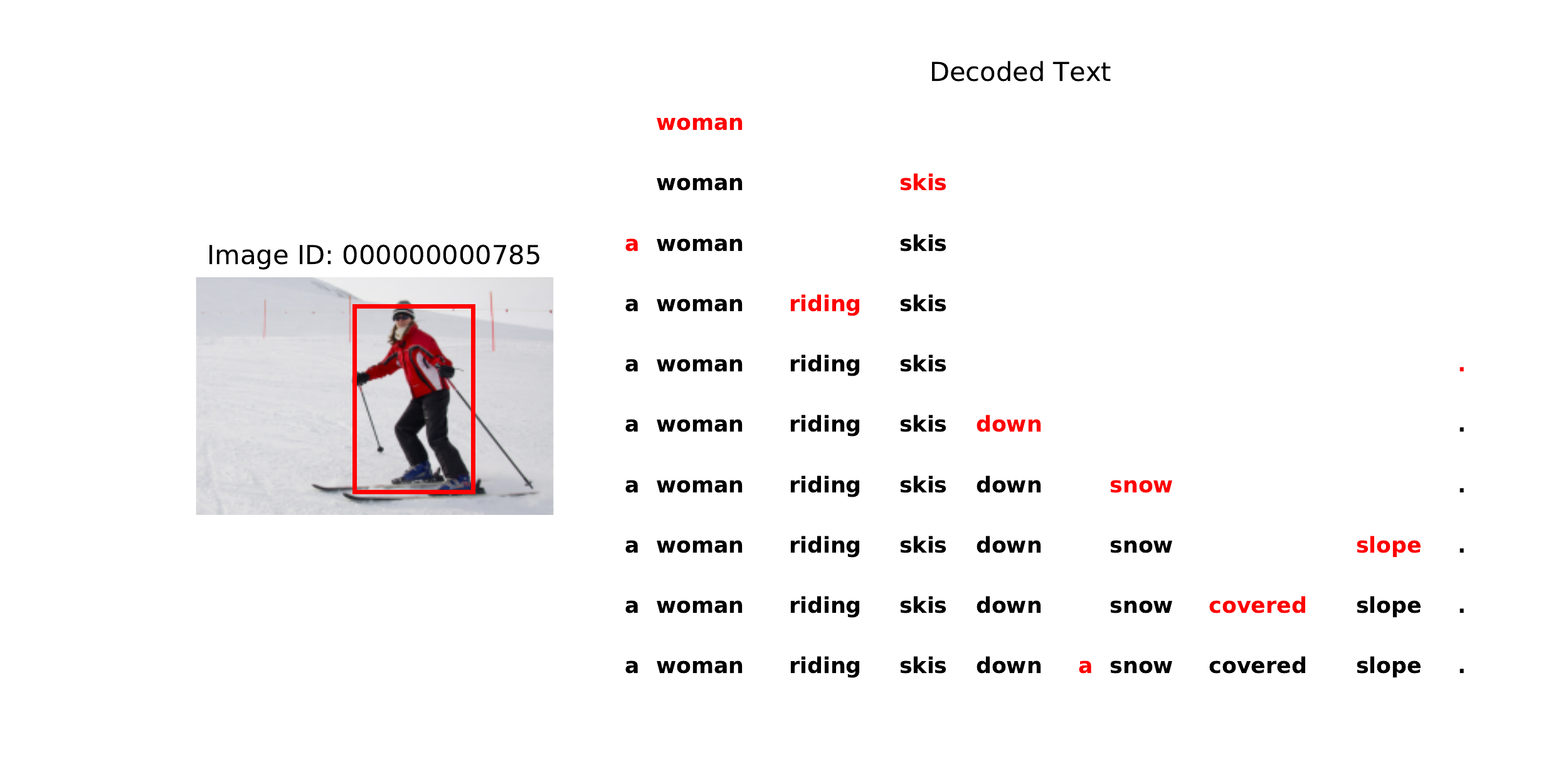}
    \caption{Generation order inferred by \textbf{Ours-VOI} for an image from the COCO 2017 validation set with the image identifier $\mathbf{000000000785}$. }
    \label{fig:000000000785_gen_order_ours}
\end{figure}

\begin{figure}
    \centering
    \includegraphics[width=\linewidth]{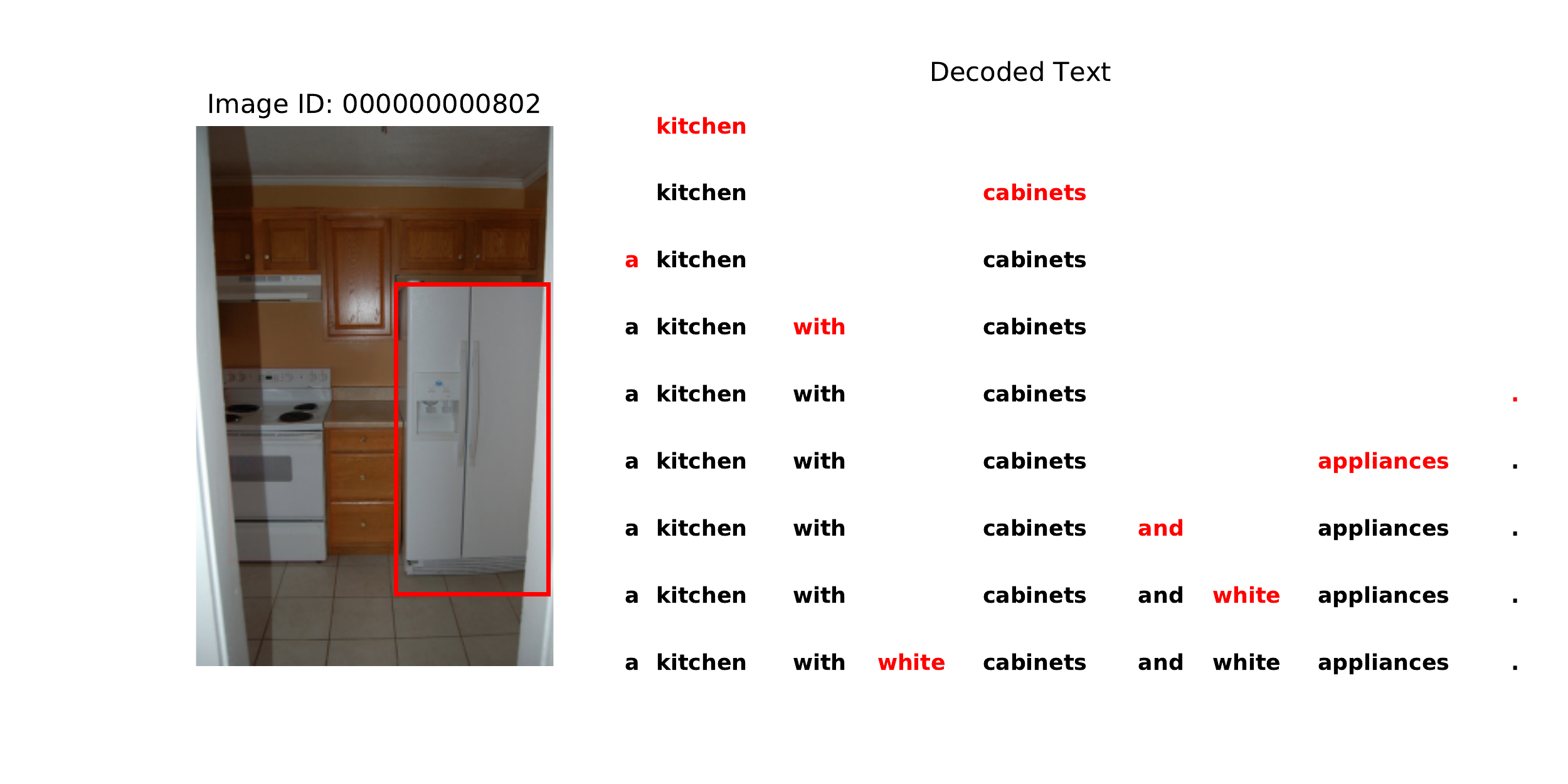}
    \caption{Generation order inferred by \textbf{Ours-VOI} for an image from the COCO 2017 validation set with the image identifier $\mathbf{000000000802}$.}
    \label{fig:000000000802_gen_order_ours}
\end{figure}

\begin{figure}
    \centering
    \includegraphics[width=\linewidth]{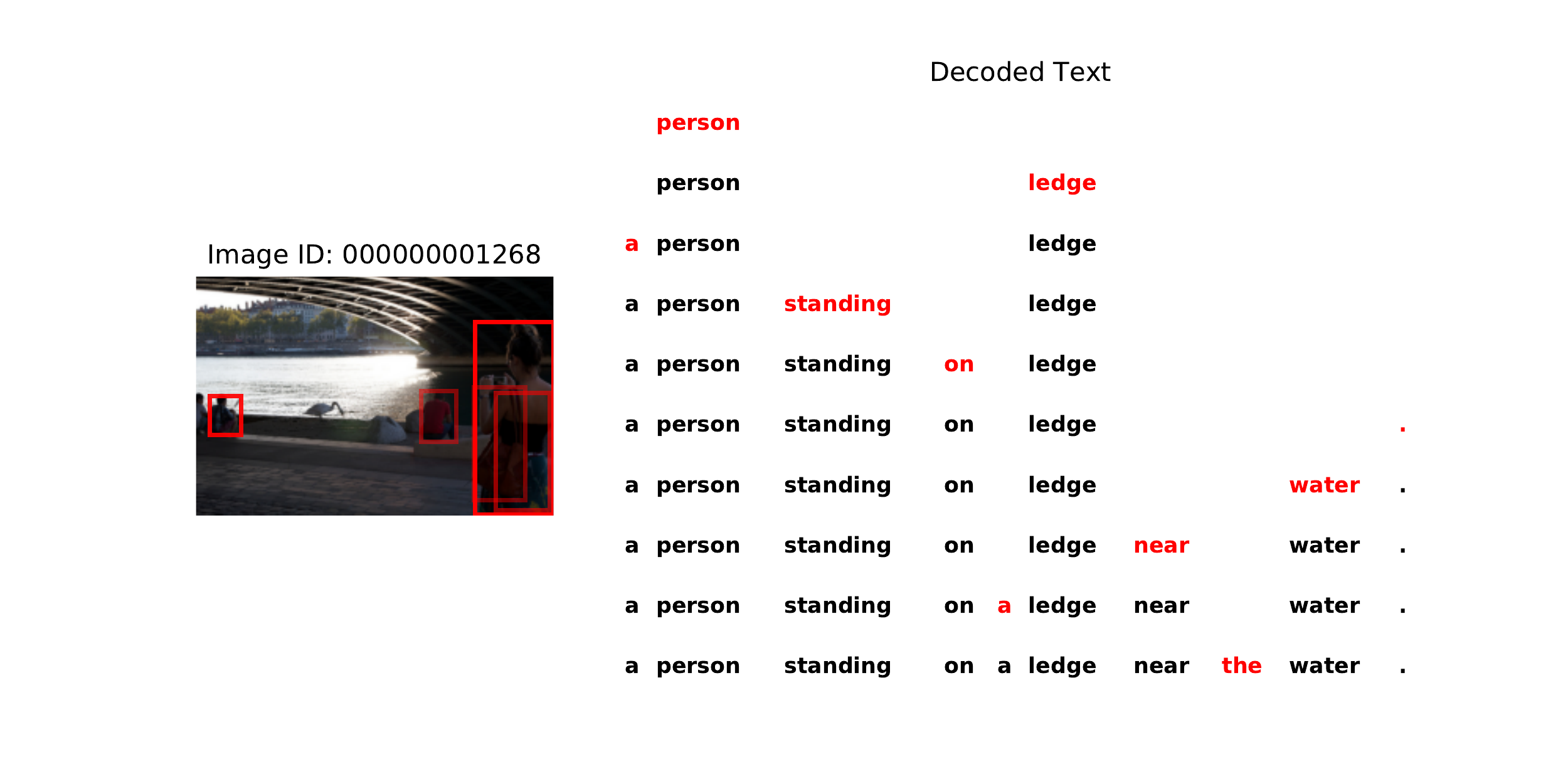}
    \caption{Generation order inferred by \textbf{Ours-VOI} for an image from the COCO 2017 validation set with the image identifier $\mathbf{000000001268}$.}
    \label{fig:000000001268_gen_order_ours}
\end{figure}

\begin{figure}
    \centering
    \includegraphics[width=\linewidth]{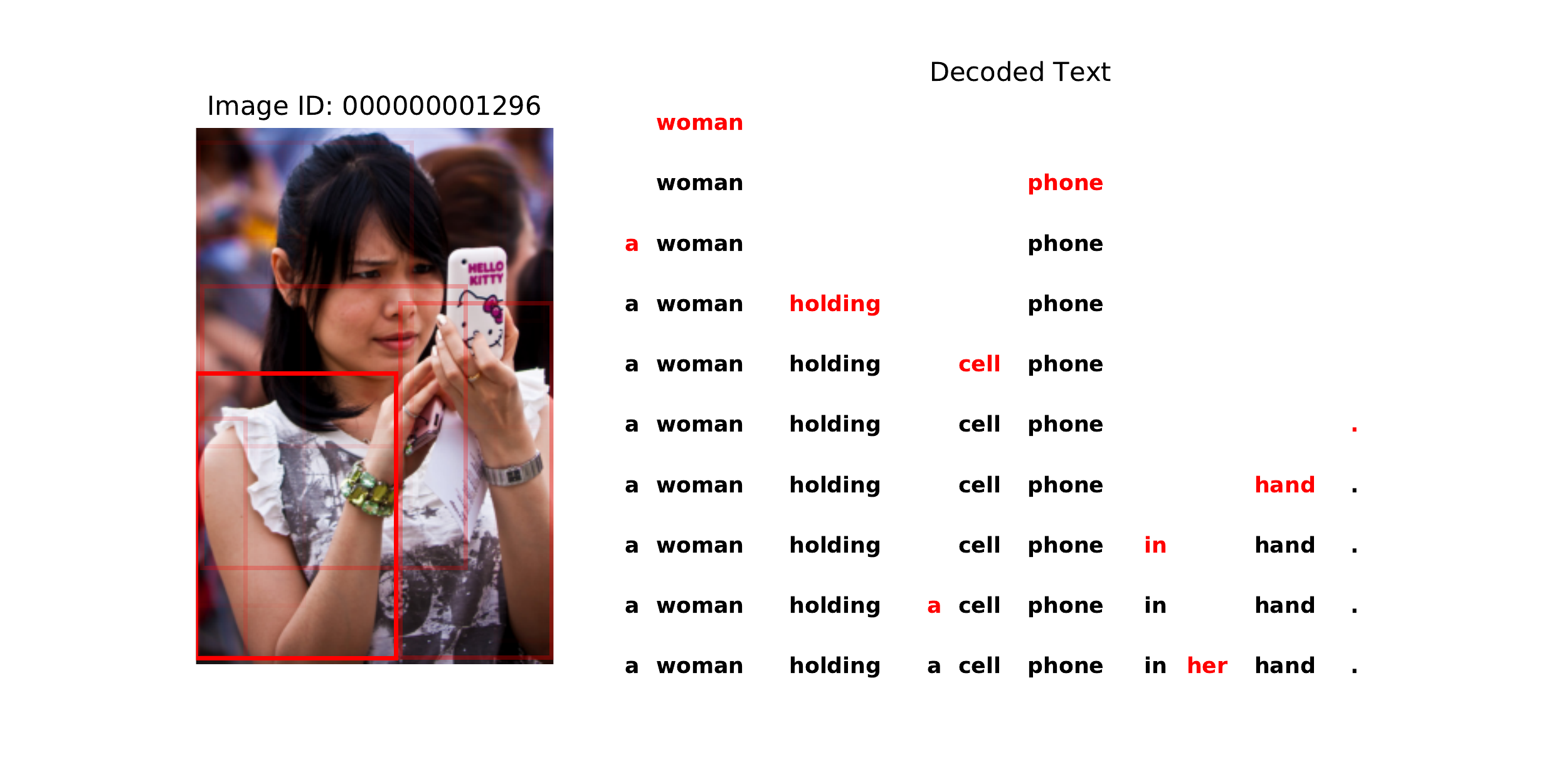}
    \caption{Generation order inferred by \textbf{Ours-VOI} for an image from the COCO 2017 validation set with the image identifier $\mathbf{000000001296}$.}
    \label{fig:000000001296_gen_order_ours}
\end{figure}

\begin{figure}
    \centering
    \includegraphics[width=\linewidth]{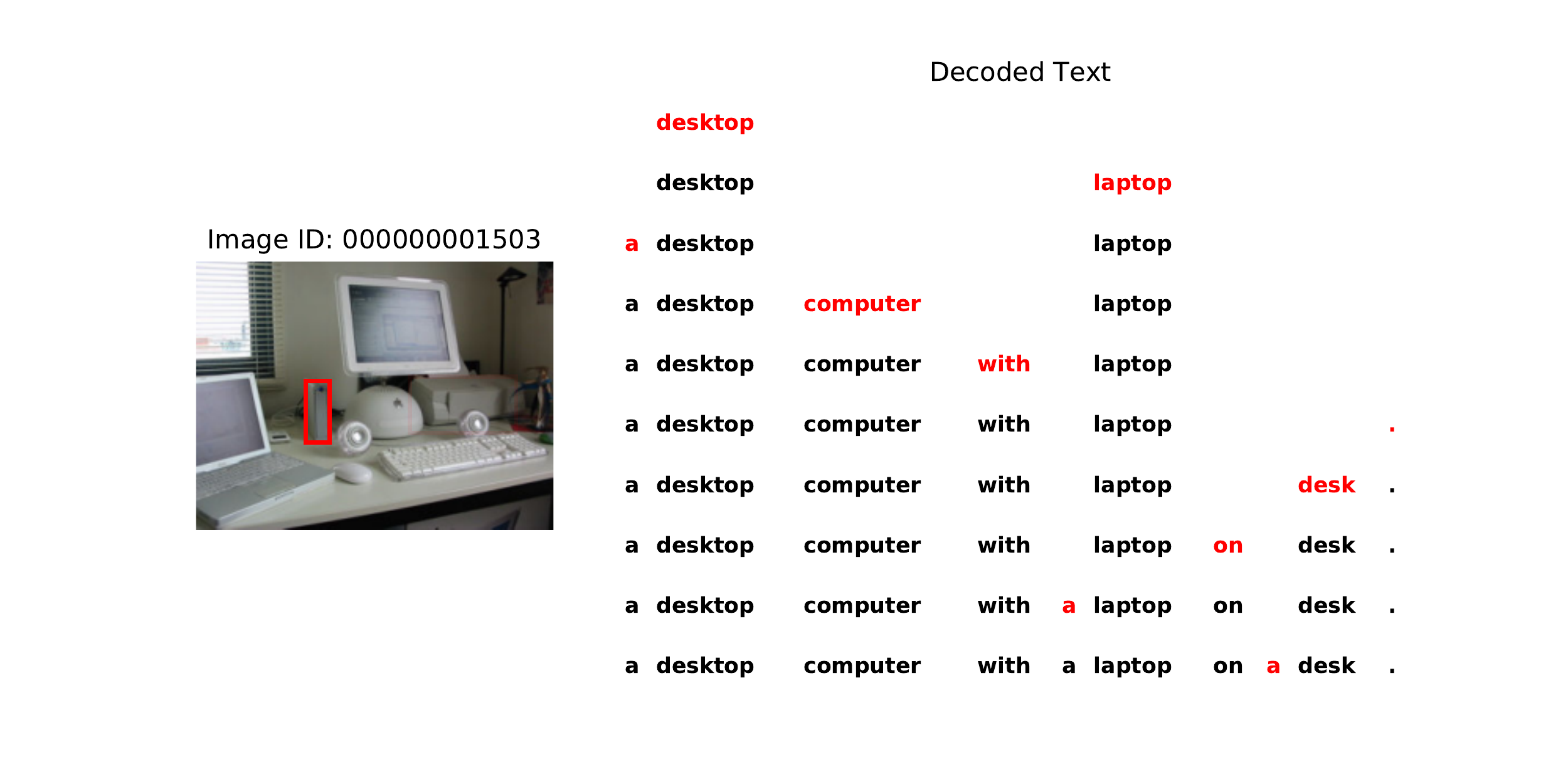}
    \caption{Generation order inferred by \textbf{Ours-VOI} for an image from the COCO 2017 validation set with the image identifier $\mathbf{000000001503}$.}
    \label{fig:000000001503_gen_order_ours}
\end{figure}

\begin{figure}
    \centering
    \includegraphics[width=\linewidth]{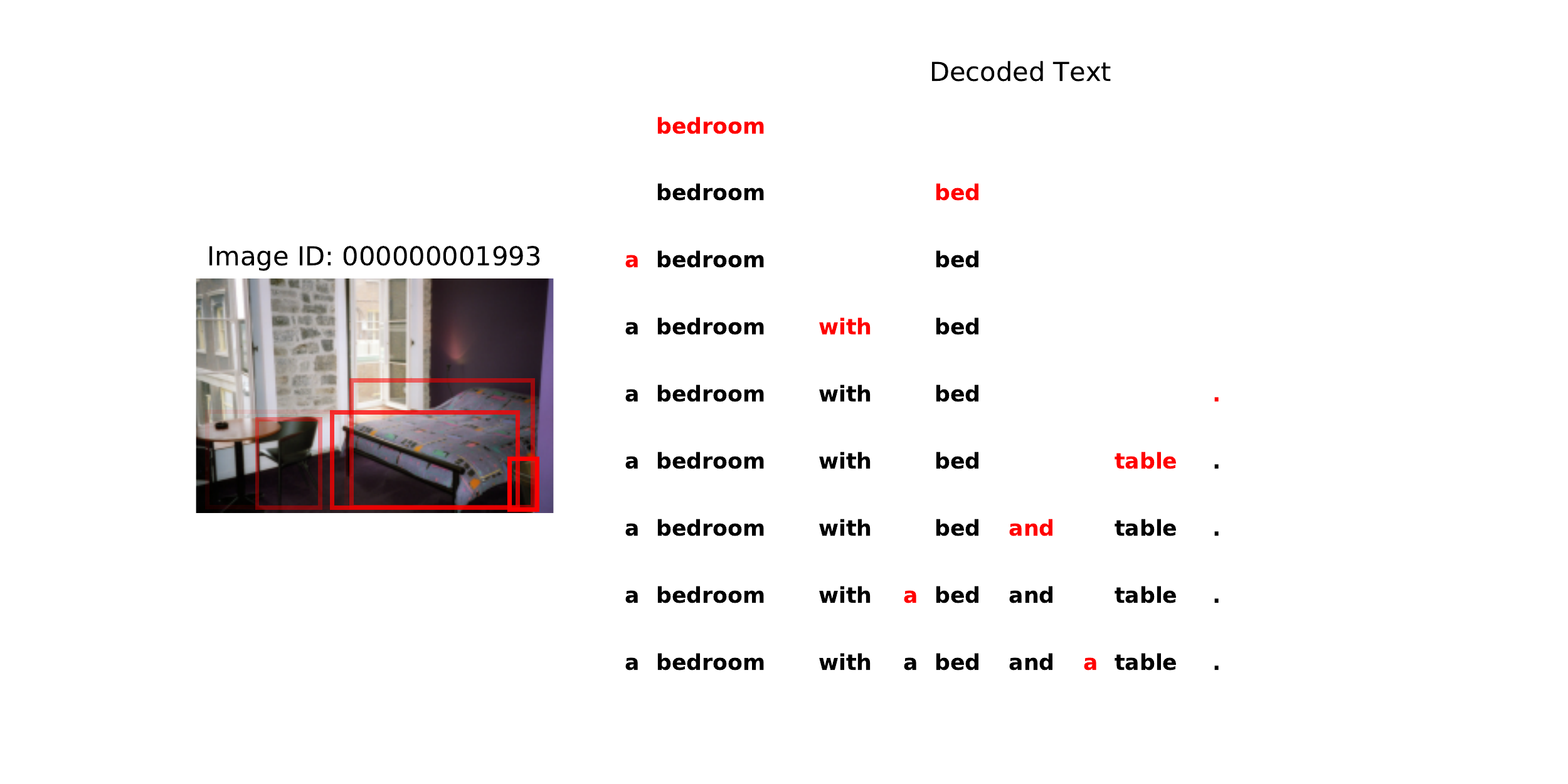}
    \caption{Generation order inferred by \textbf{Ours-VOI} for an image from the COCO 2017 validation set with the image identifier $\mathbf{000000001993}$. }
    \label{fig:000000001993_gen_order_ours}
\end{figure}

\begin{figure}
    \centering
    \includegraphics[width=\linewidth]{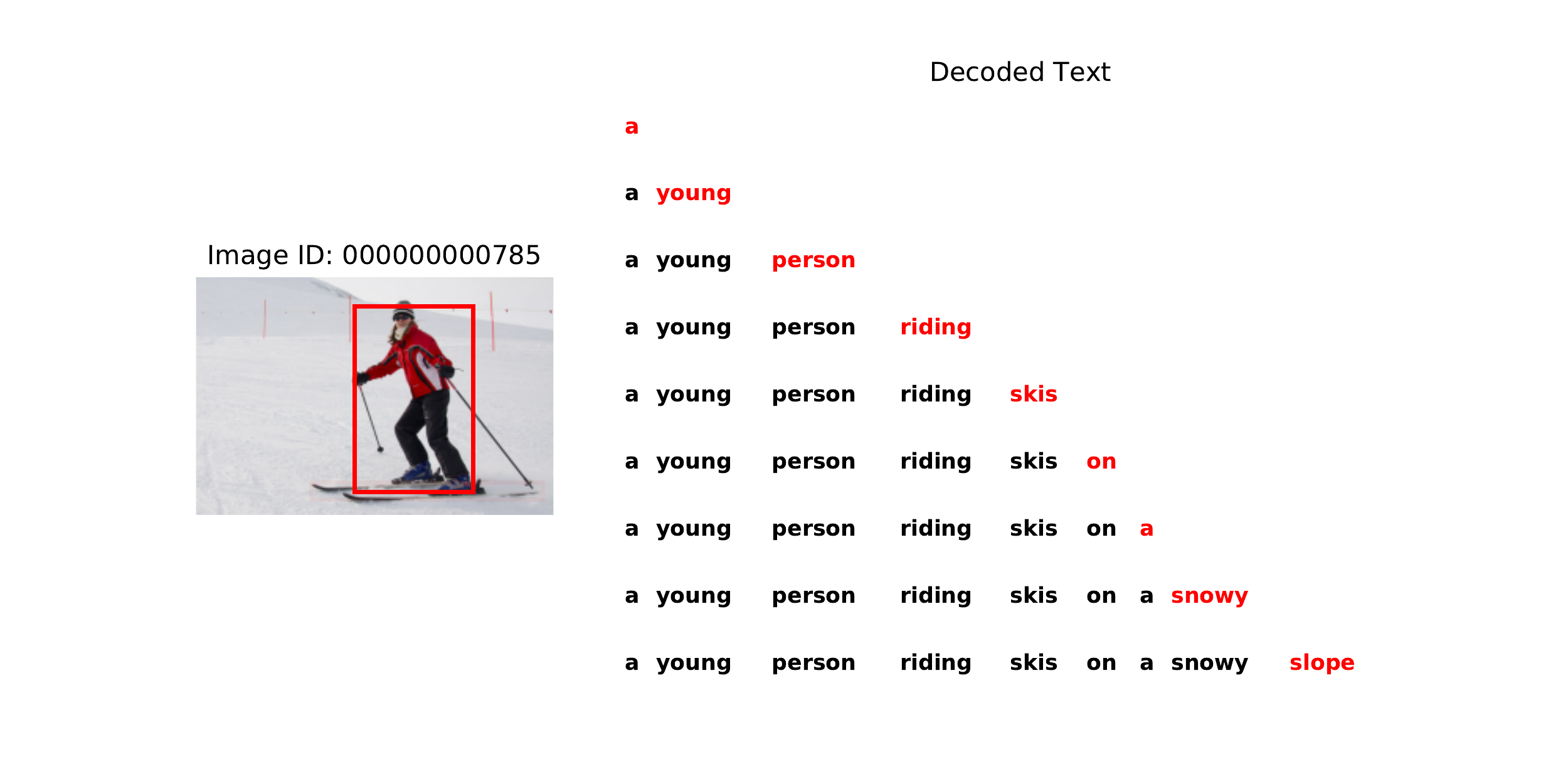}
    \caption{Generation order inferred by \textbf{Ours-L2R} for an image from the COCO 2017 validation set with the image identifier $\mathbf{000000000785}$. }
    \label{fig:000000000785_gen_order_l2r}
\end{figure}

\begin{figure}
    \centering
    \includegraphics[width=\linewidth]{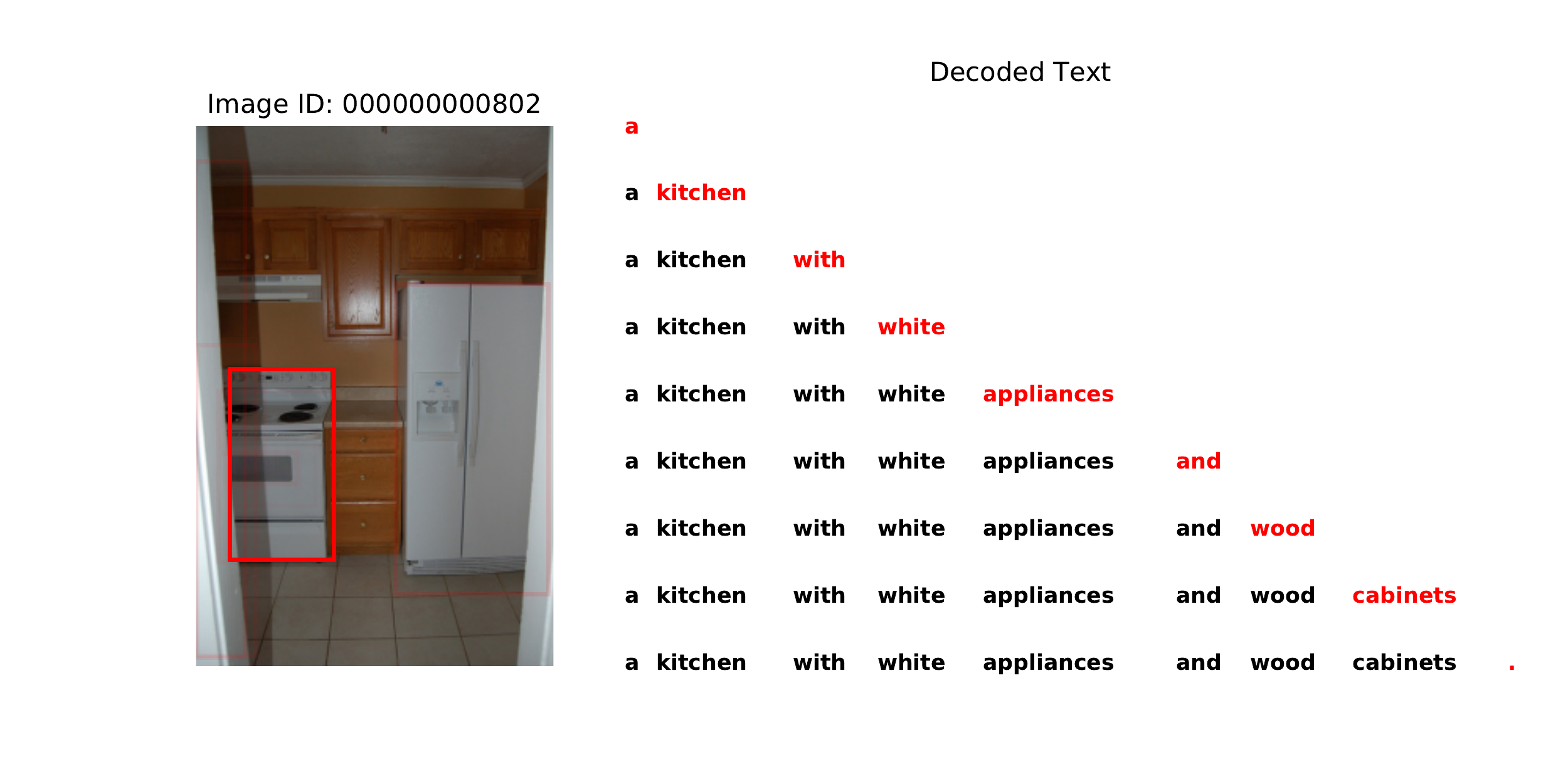}
    \caption{Generation order inferred by \textbf{Ours-L2R} for an image from the COCO 2017 validation set with the image identifier $\mathbf{000000000802}$.}
    \label{fig:000000000802_gen_order_l2r}
\end{figure}

\begin{figure}
    \centering
    \includegraphics[width=\linewidth]{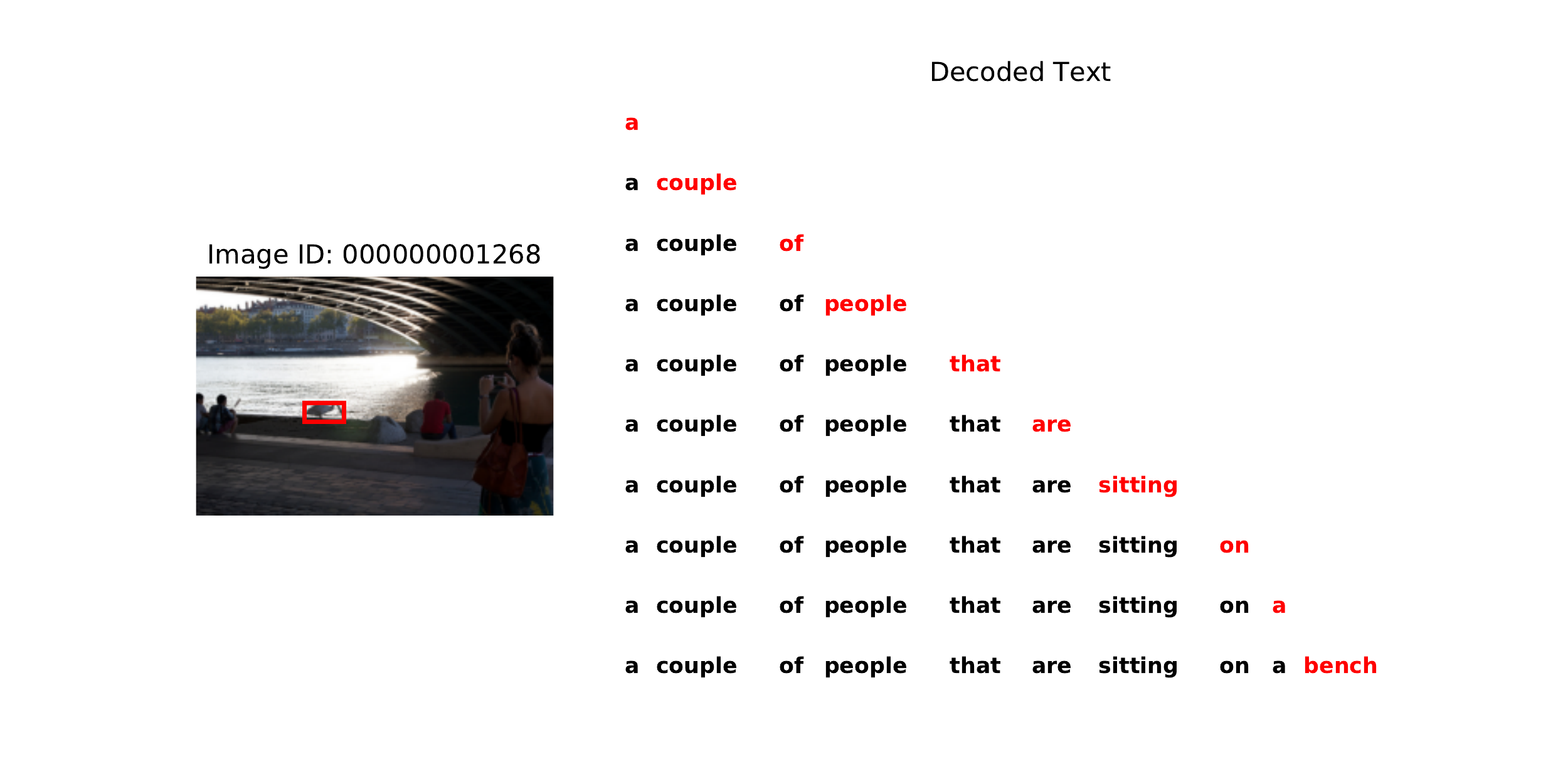}
    \caption{Generation order inferred by \textbf{Ours-L2R} for an image from the COCO 2017 validation set with the image identifier $\mathbf{000000001268}$.}
    \label{fig:000000001268_gen_order_l2r}
\end{figure}

\begin{figure}
    \centering
    \includegraphics[width=\linewidth]{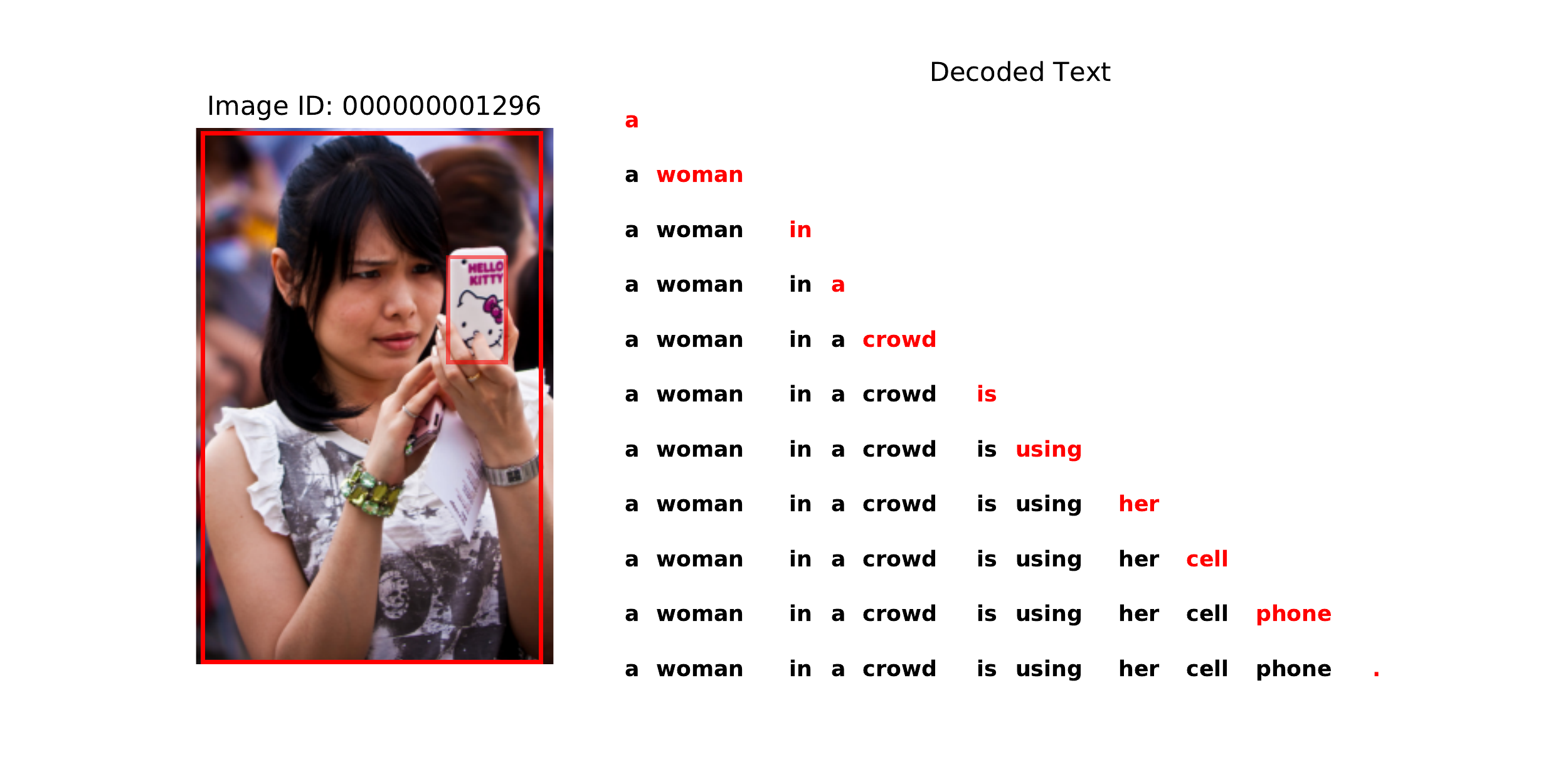}
    \caption{Generation order inferred by \textbf{Ours-L2R} for an image from the COCO 2017 validation set with the image identifier $\mathbf{000000001296}$.}
    \label{fig:000000001296_gen_order_l2r}
\end{figure}

\begin{figure}
    \centering
    \includegraphics[width=\linewidth]{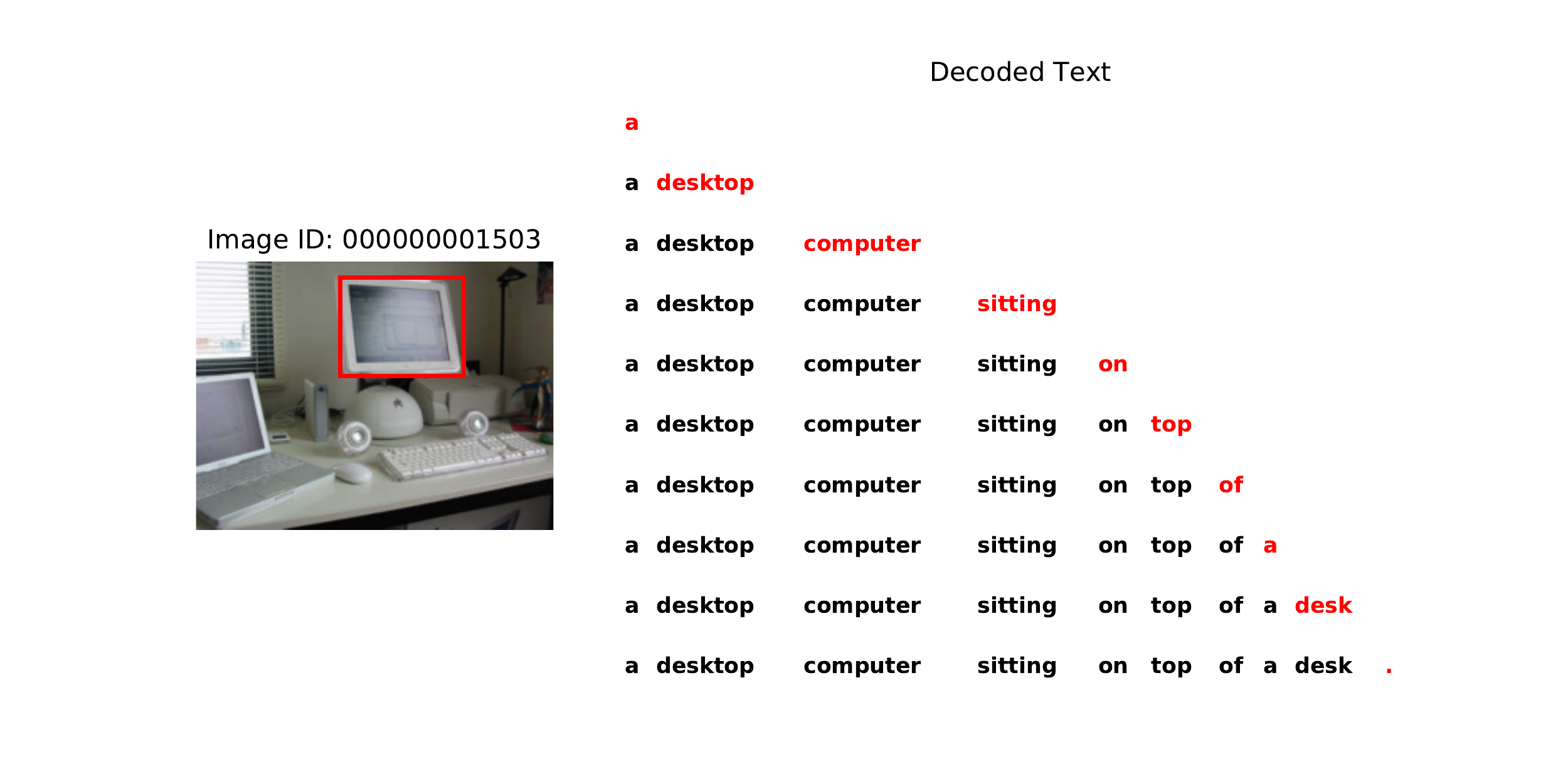}
    \caption{Generation order inferred by \textbf{Ours-L2R} for an image from the COCO 2017 validation set with the image identifier $\mathbf{000000001503}$.}
    \label{fig:000000001503_gen_order_l2r}
\end{figure}

\begin{figure}
    \centering
    \includegraphics[width=\linewidth]{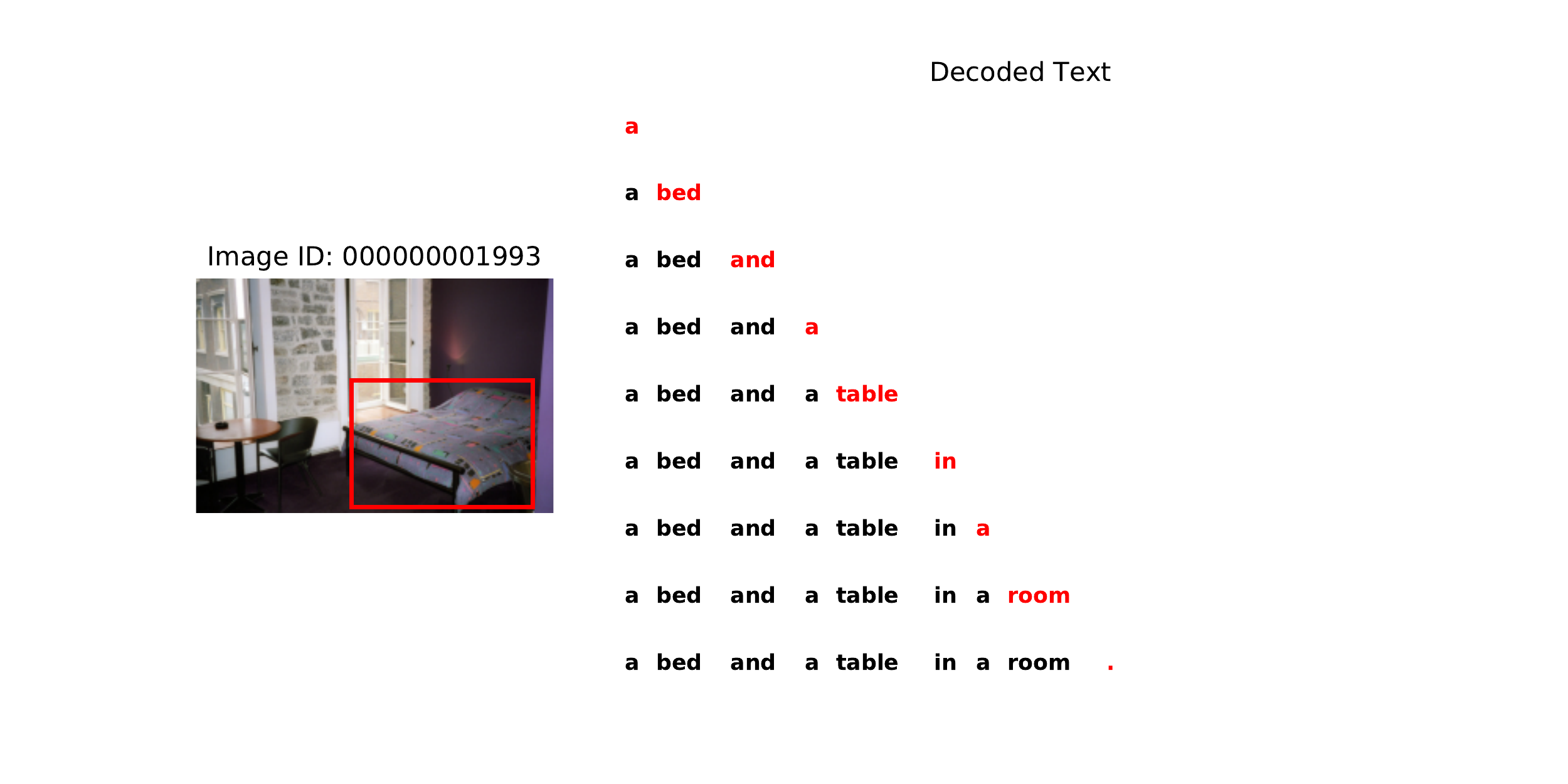}
    \caption{Generation order inferred by \textbf{Ours-L2R} for an image from the COCO 2017 validation set with the image identifier $\mathbf{000000001993}$. }
    \label{fig:000000001993_gen_order_l2r}
\end{figure}

\begin{figure}
    \centering
    \includegraphics[width=\linewidth]{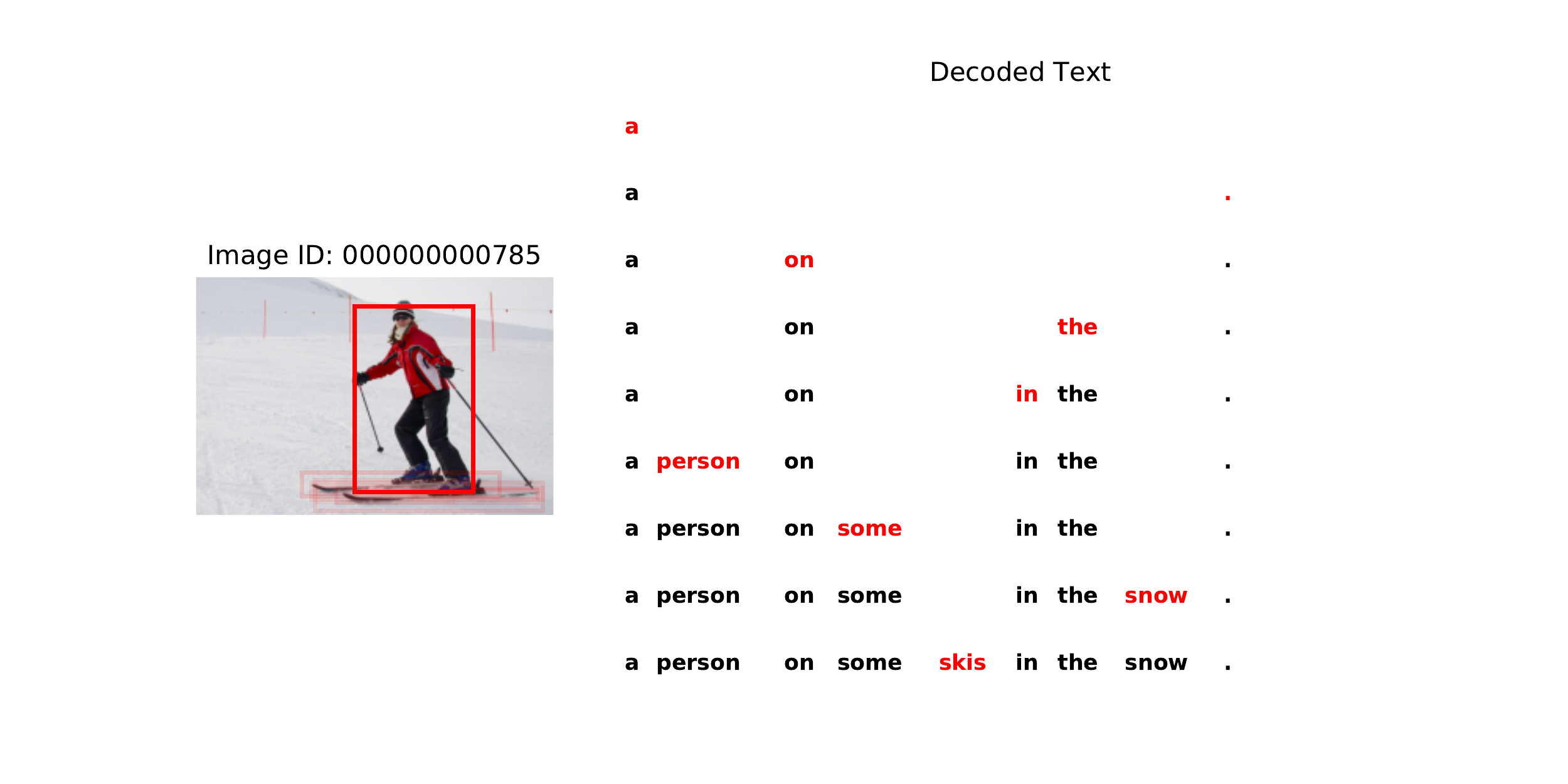}
    \caption{Generation order inferred by \textbf{Ours-Common} for an image from the COCO 2017 validation set with the image identifier $\mathbf{000000000785}$. }
    \label{fig:000000000785_gen_order_common}
\end{figure}

\begin{figure}
    \centering
    \includegraphics[width=\linewidth]{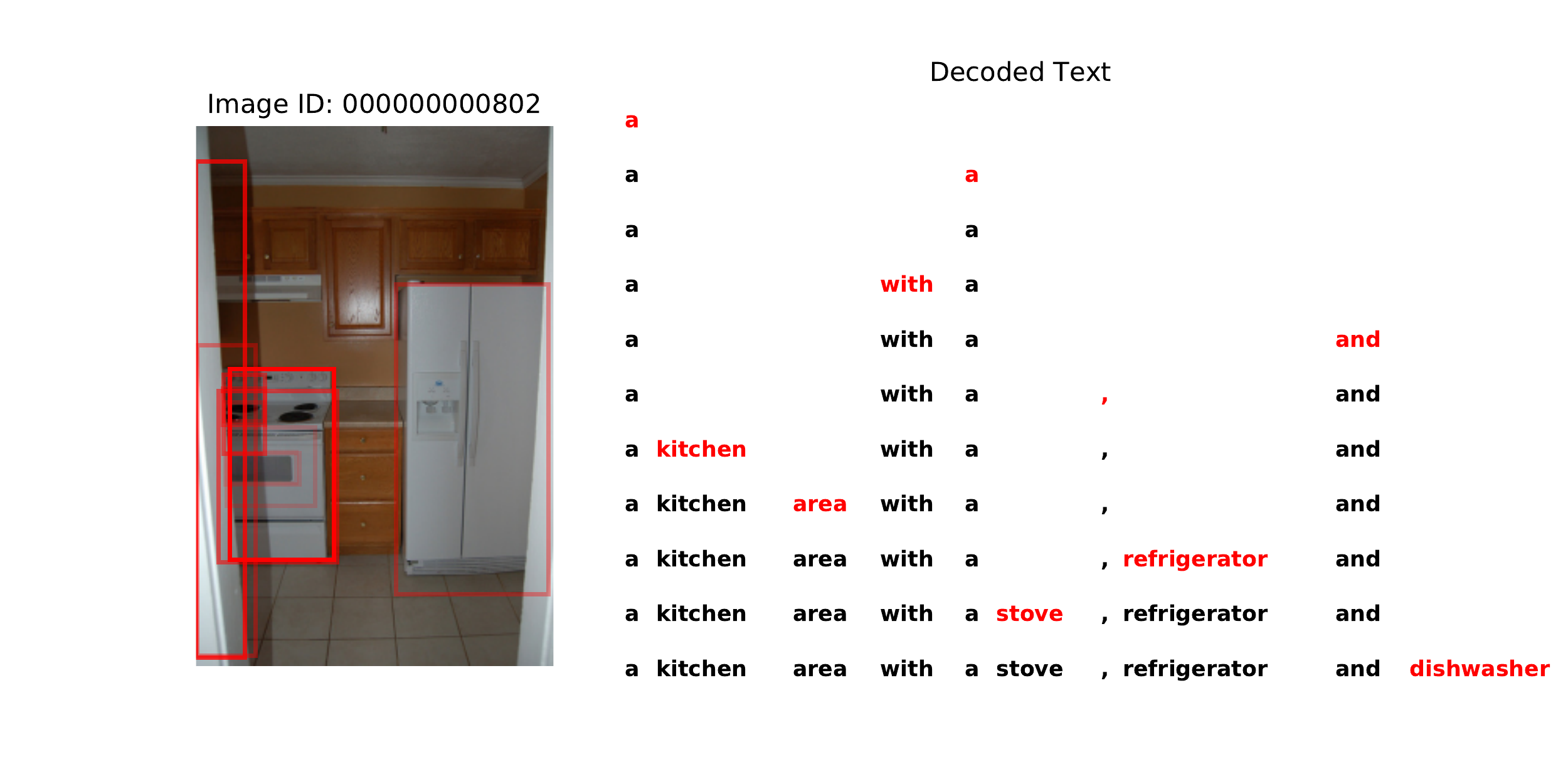}
    \caption{Generation order inferred by \textbf{Ours-Common} for an image from the COCO 2017 validation set with the image identifier $\mathbf{000000000802}$.}
    \label{fig:000000000802_gen_order_common}
\end{figure}

\begin{figure}
    \centering
    \includegraphics[width=\linewidth]{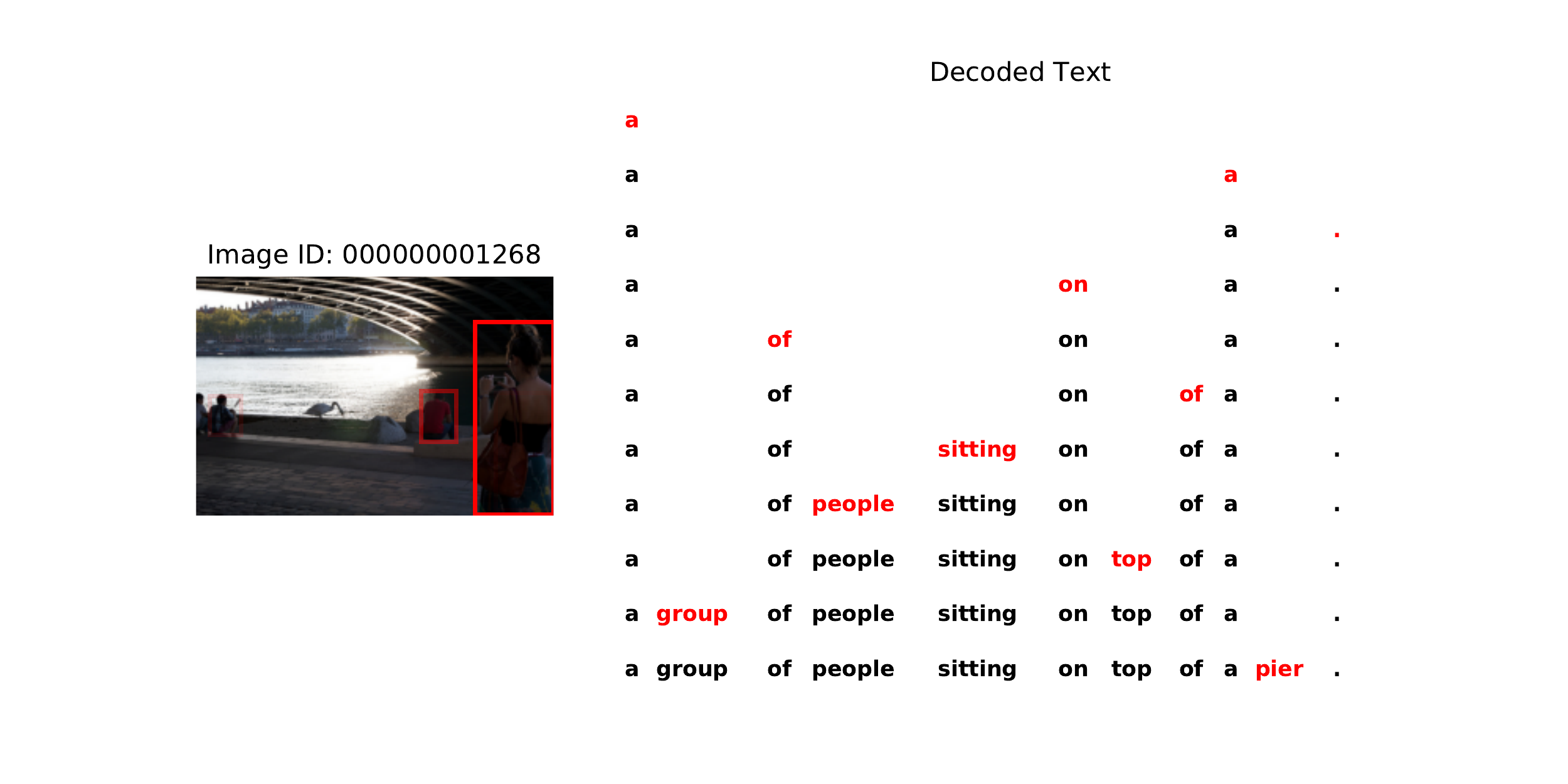}
    \caption{Generation order inferred by \textbf{Ours-Common} for an image from the COCO 2017 validation set with the image identifier $\mathbf{000000001268}$.}
    \label{fig:000000001268_gen_order_common}
\end{figure}

\begin{figure}
    \centering
    \includegraphics[width=\linewidth]{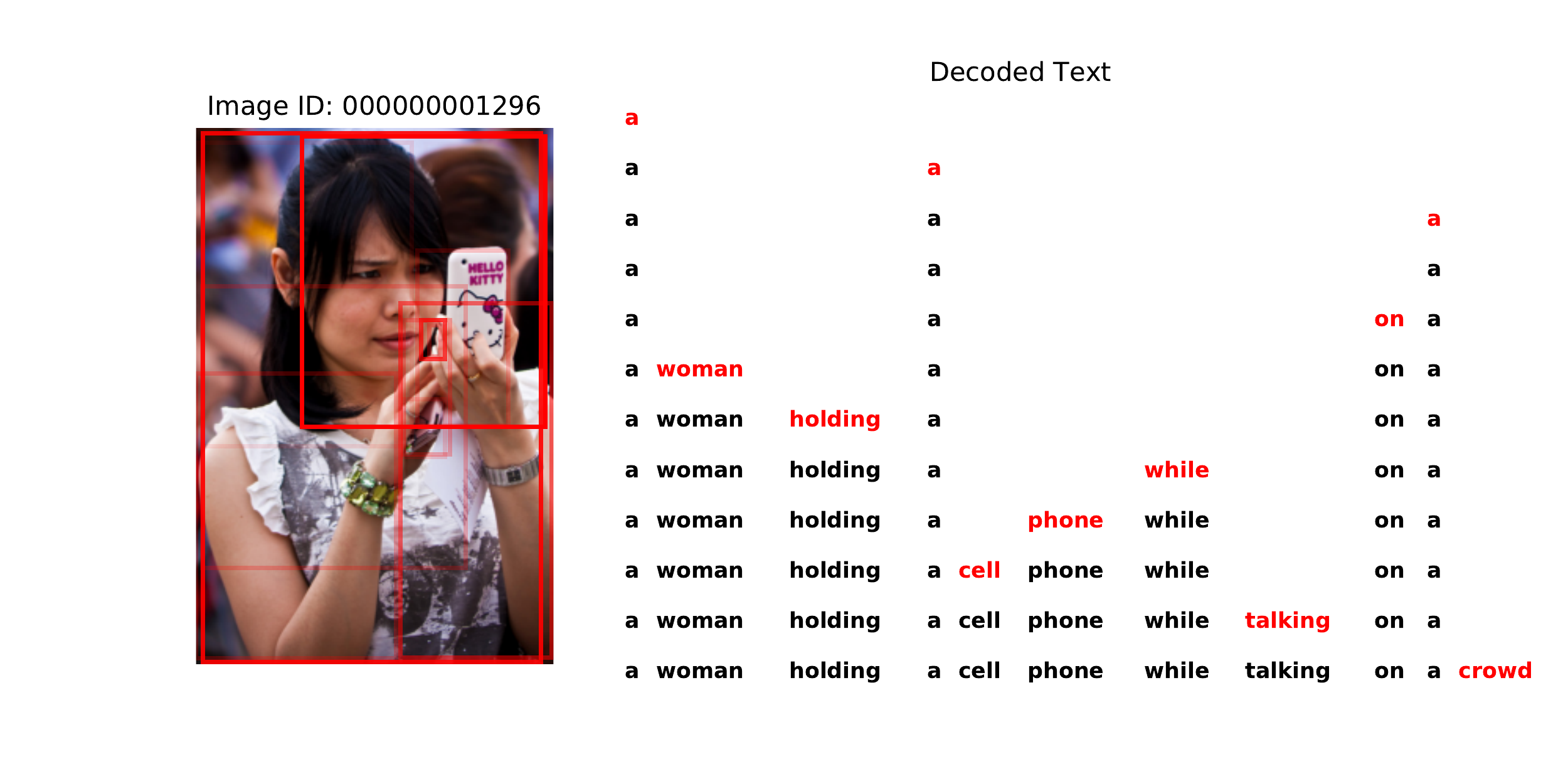}
    \caption{Generation order inferred by \textbf{Ours-Common} for an image from the COCO 2017 validation set with the image identifier $\mathbf{000000001296}$.}
    \label{fig:000000001296_gen_order_common}
\end{figure}

\begin{figure}
    \centering
    \includegraphics[width=\linewidth]{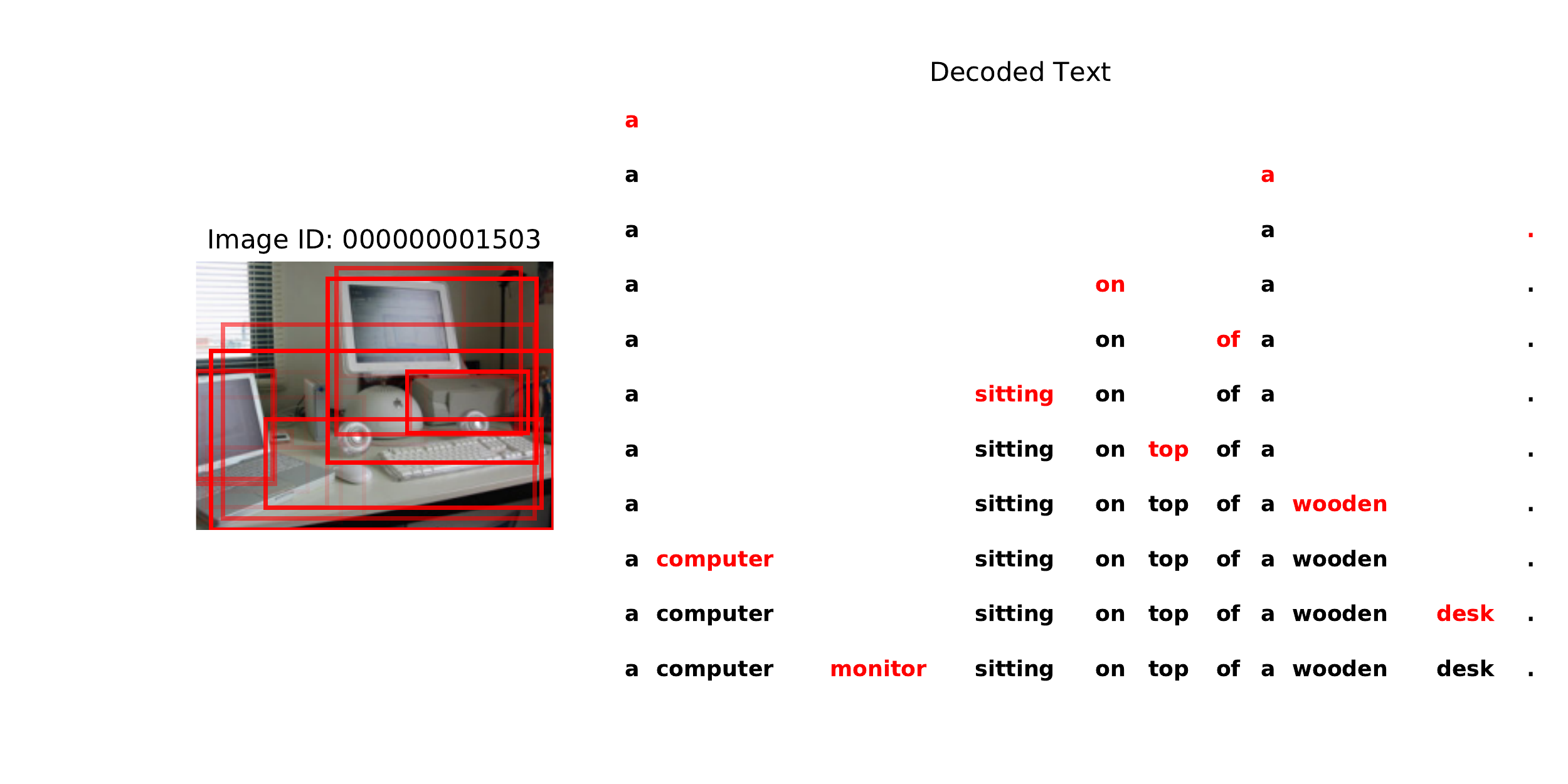}
    \caption{Generation order inferred by \textbf{Ours-Common} for an image from the COCO 2017 validation set with the image identifier $\mathbf{000000001503}$.}
    \label{fig:000000001503_gen_order_common}
\end{figure}

\begin{figure}
    \centering
    \includegraphics[width=\linewidth]{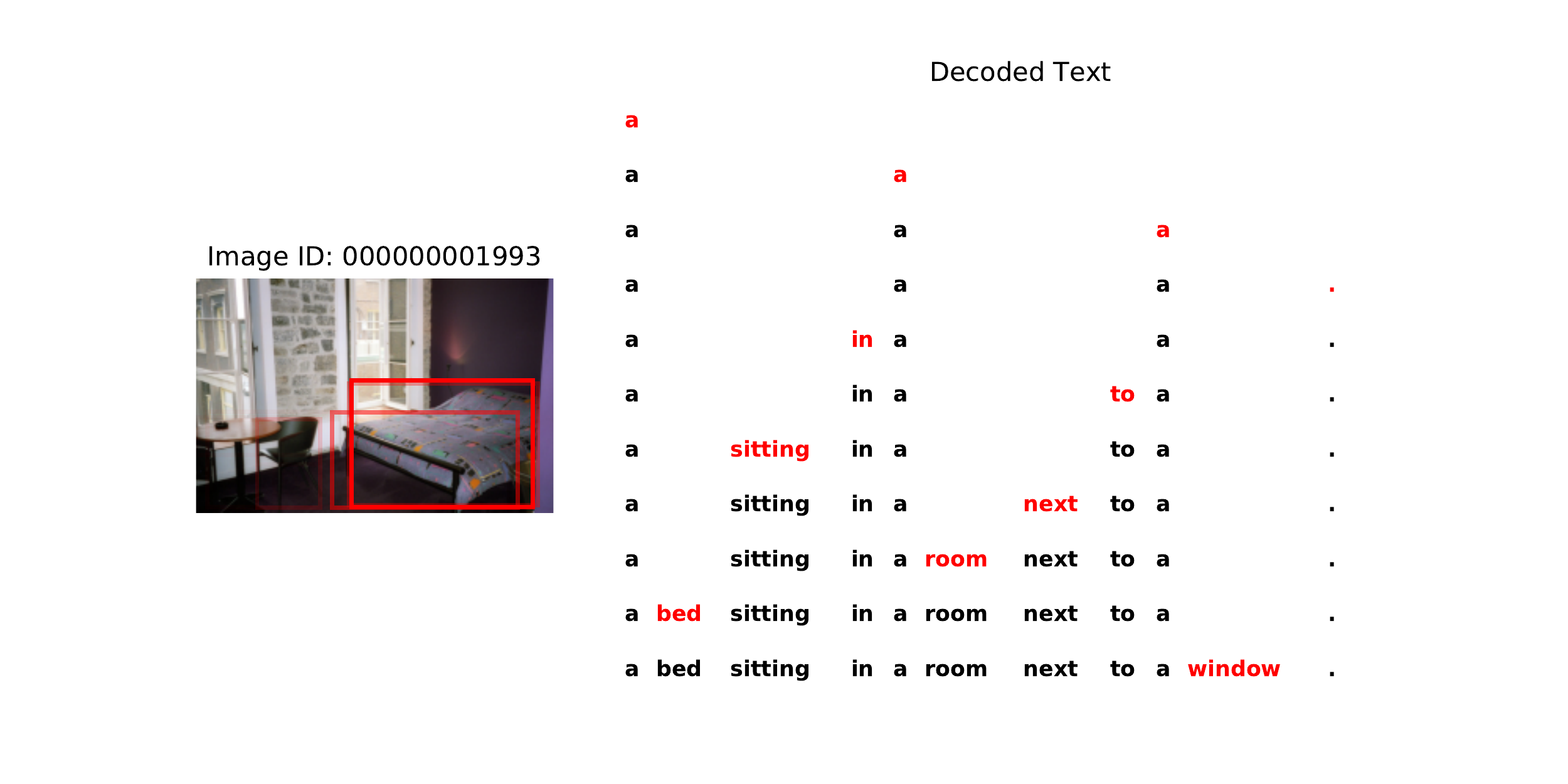}
    \caption{Generation order inferred by \textbf{Ours-Common} for an image from the COCO 2017 validation set with the image identifier $\mathbf{000000001993}$. }
    \label{fig:000000001993_gen_order_common}
\end{figure}

\begin{figure}
    \centering
    \includegraphics[width=\linewidth]{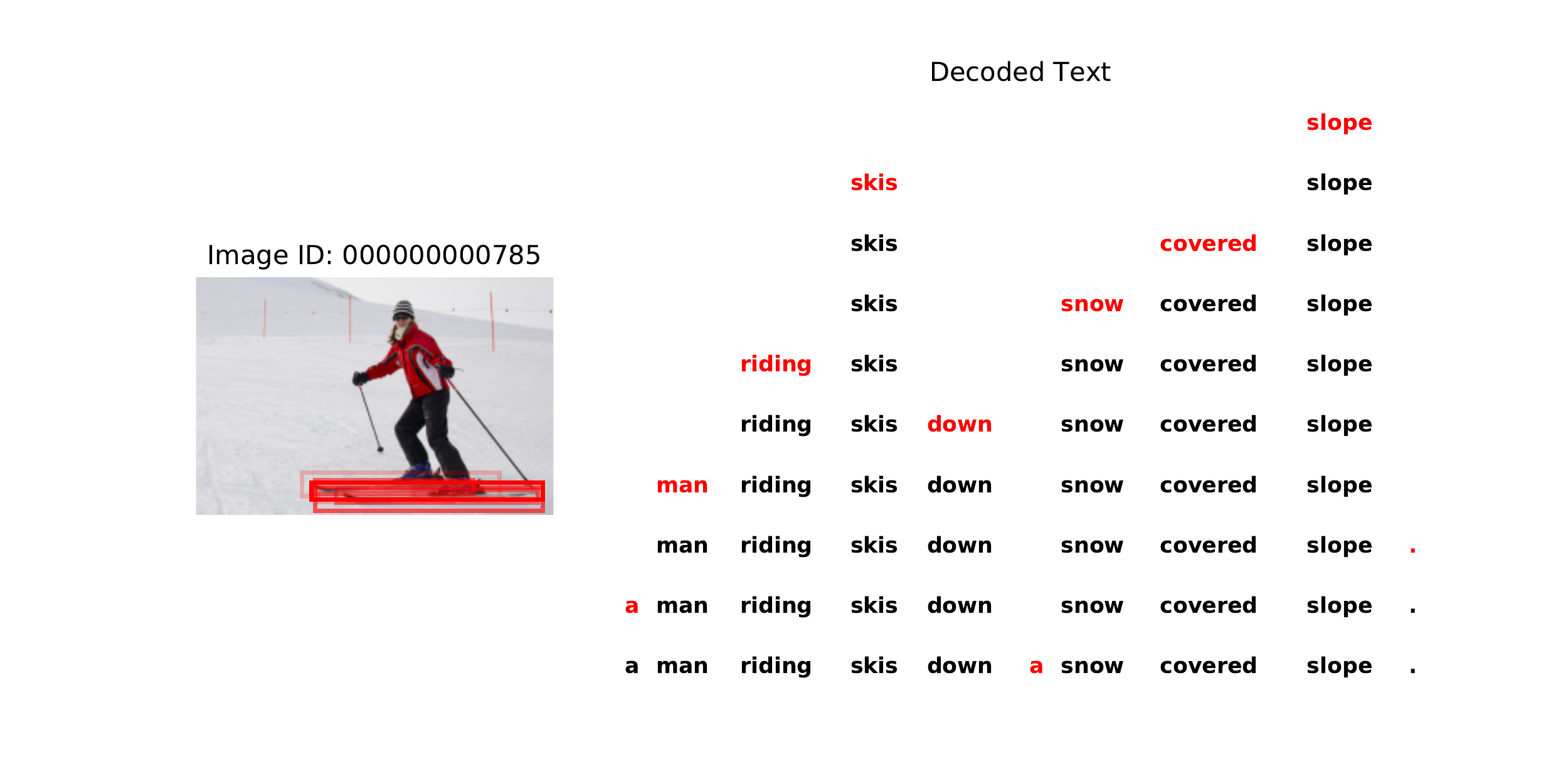}
    \caption{Generation order inferred by \textbf{Ours-Rare} for an image from the COCO 2017 validation set with the image identifier $\mathbf{000000000785}$. }
    \label{fig:000000000785_gen_order_rare}
\end{figure}

\begin{figure}
    \centering
    \includegraphics[width=\linewidth]{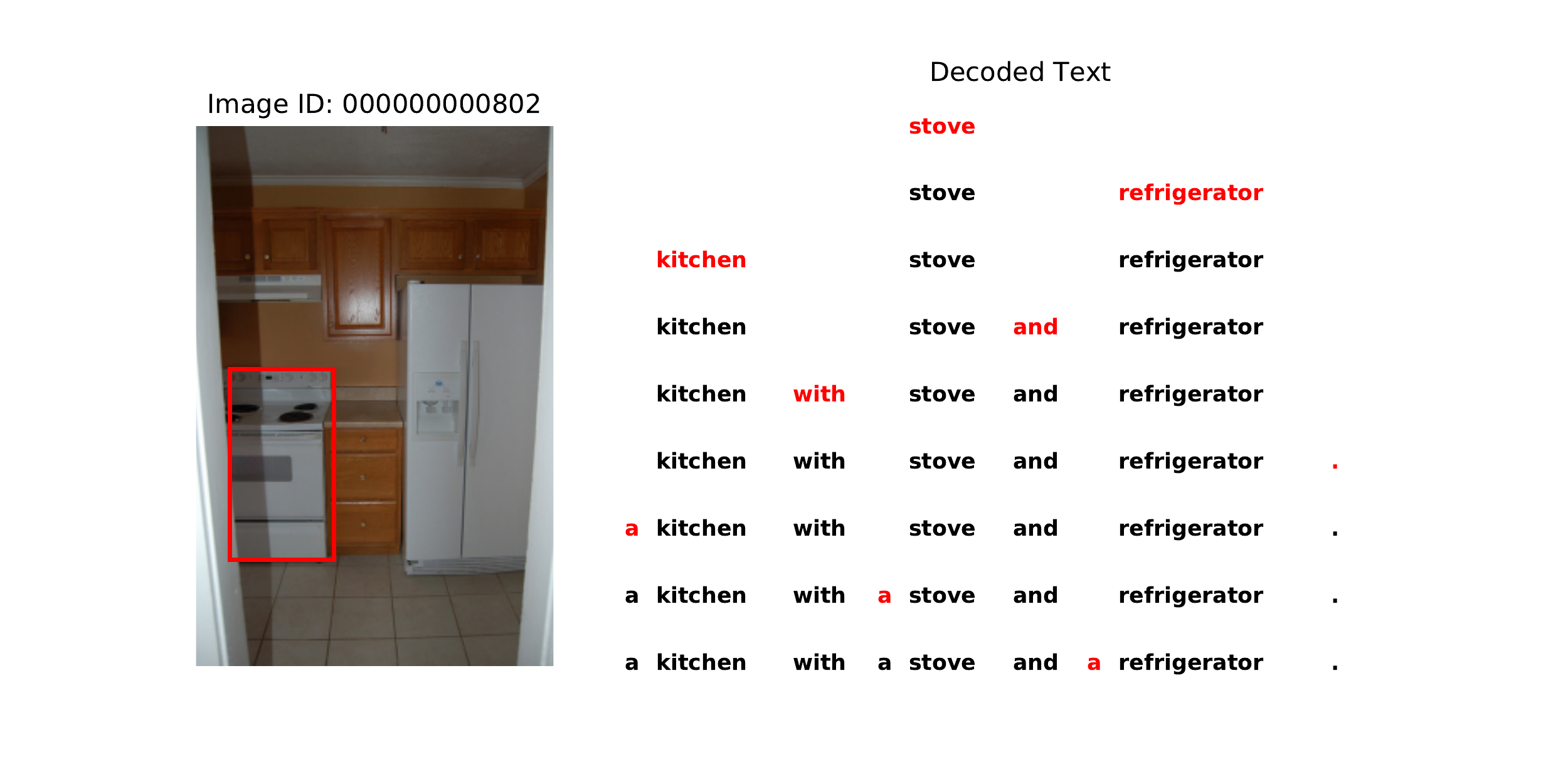}
    \caption{Generation order inferred by \textbf{Ours-Rare} for an image from the COCO 2017 validation set with the image identifier $\mathbf{000000000802}$.}
    \label{fig:000000000802_gen_order_rare}
\end{figure}

\begin{figure}
    \centering
    \includegraphics[width=\linewidth]{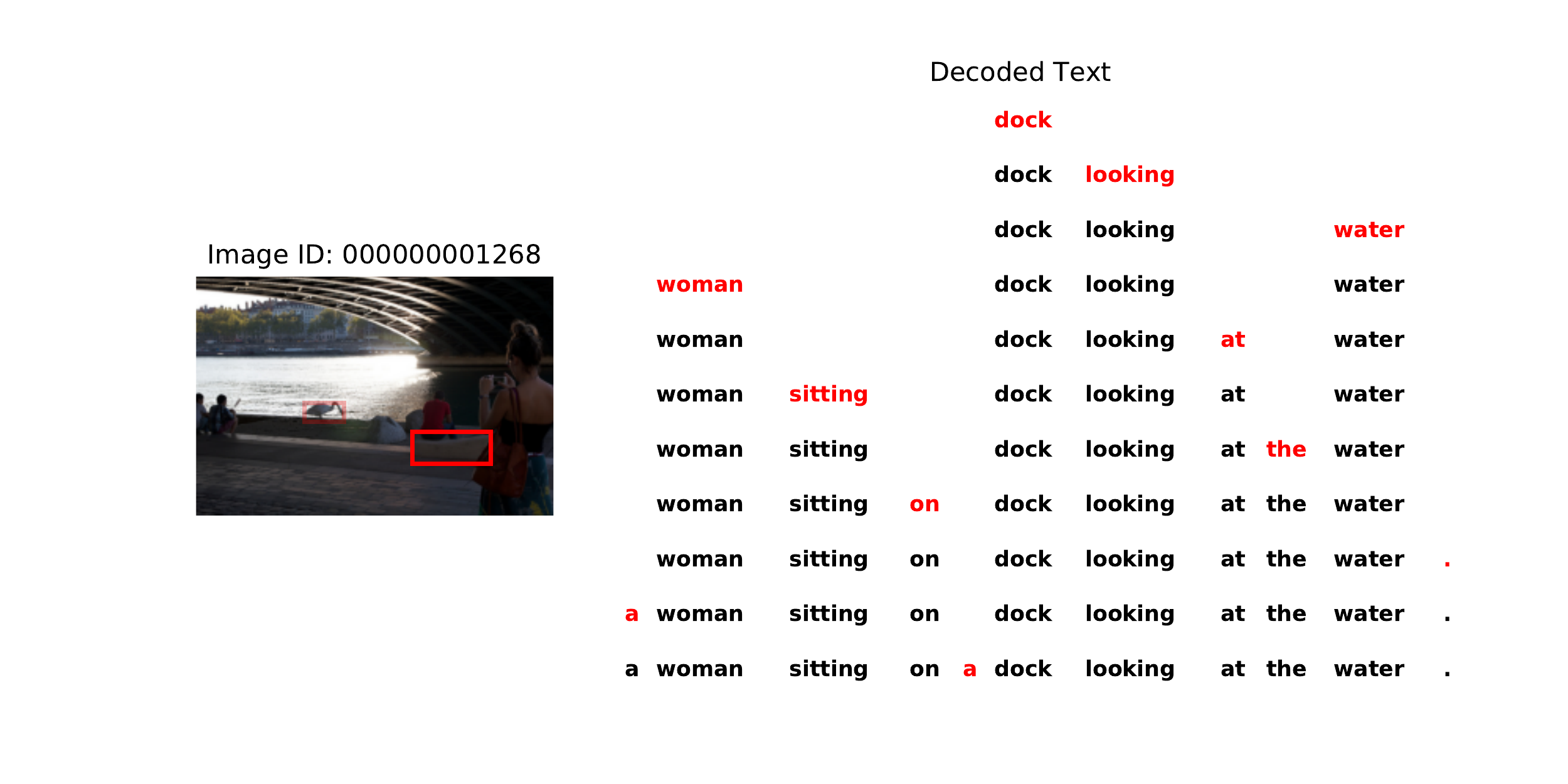}
    \caption{Generation order inferred by \textbf{Ours-Rare} for an image from the COCO 2017 validation set with the image identifier $\mathbf{000000001268}$.}
    \label{fig:000000001268_gen_order_rare}
\end{figure}

\begin{figure}
    \centering
    \includegraphics[width=\linewidth]{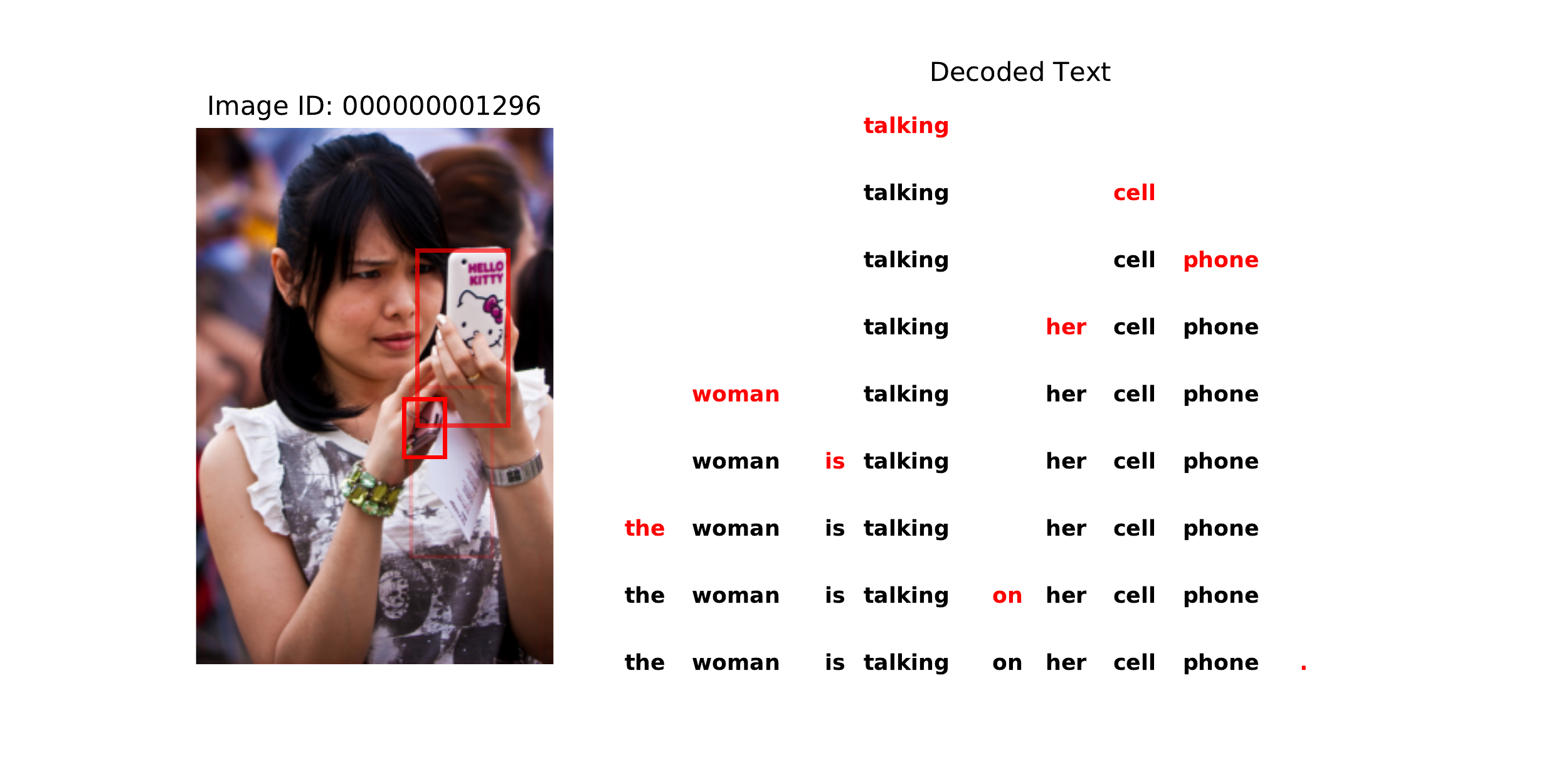}
    \caption{Generation order inferred by \textbf{Ours-Rare} for an image from the COCO 2017 validation set with the image identifier $\mathbf{000000001296}$.}
    \label{fig:000000001296_gen_order_rare}
\end{figure}

\begin{figure}
    \centering
    \includegraphics[width=\linewidth]{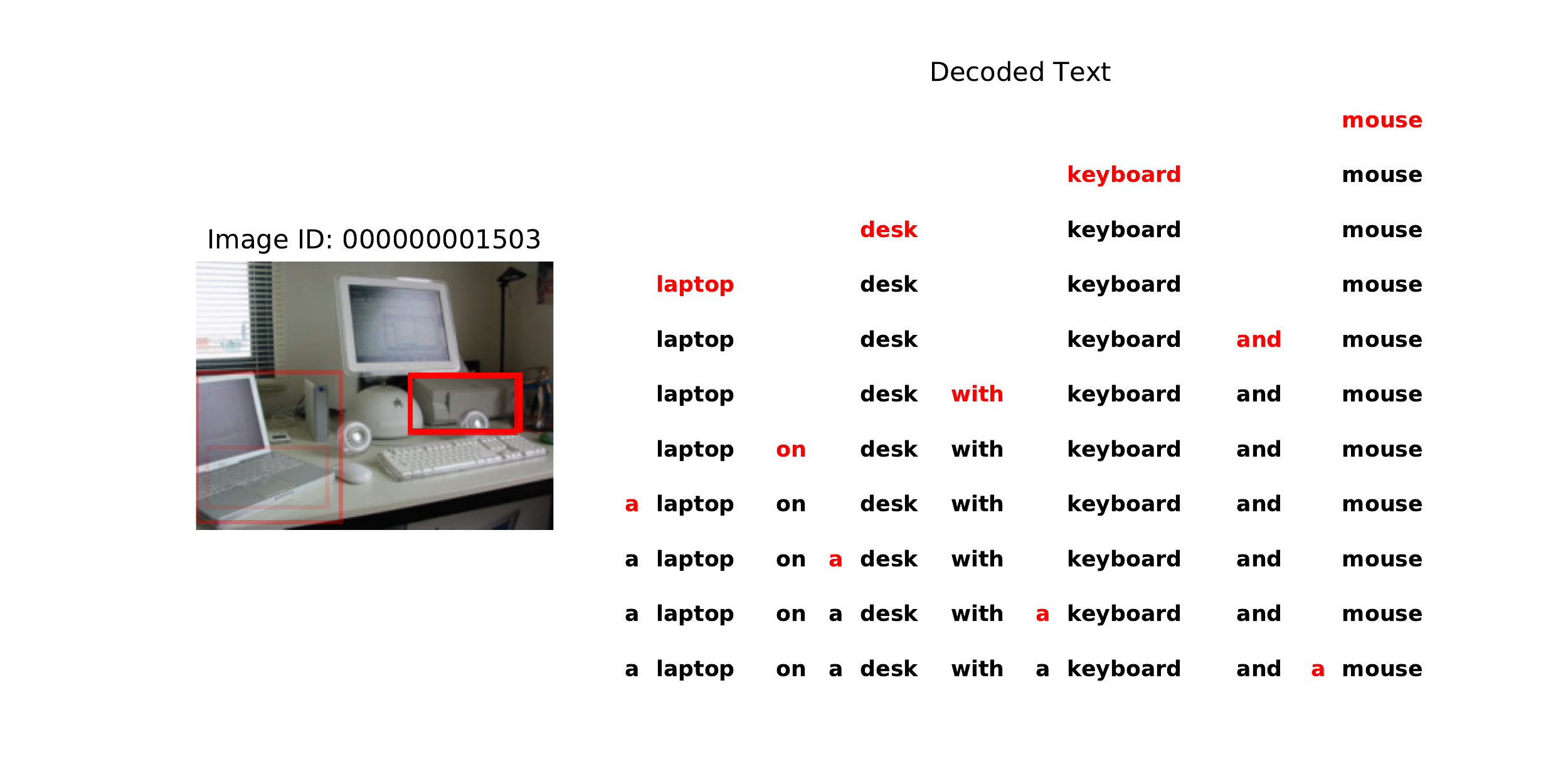}
    \caption{Generation order inferred by \textbf{Ours-Rare} for an image from the COCO 2017 validation set with the image identifier $\mathbf{000000001503}$.}
    \label{fig:000000001503_gen_order_rare}
\end{figure}

\begin{figure}
    \centering
    \includegraphics[width=\linewidth]{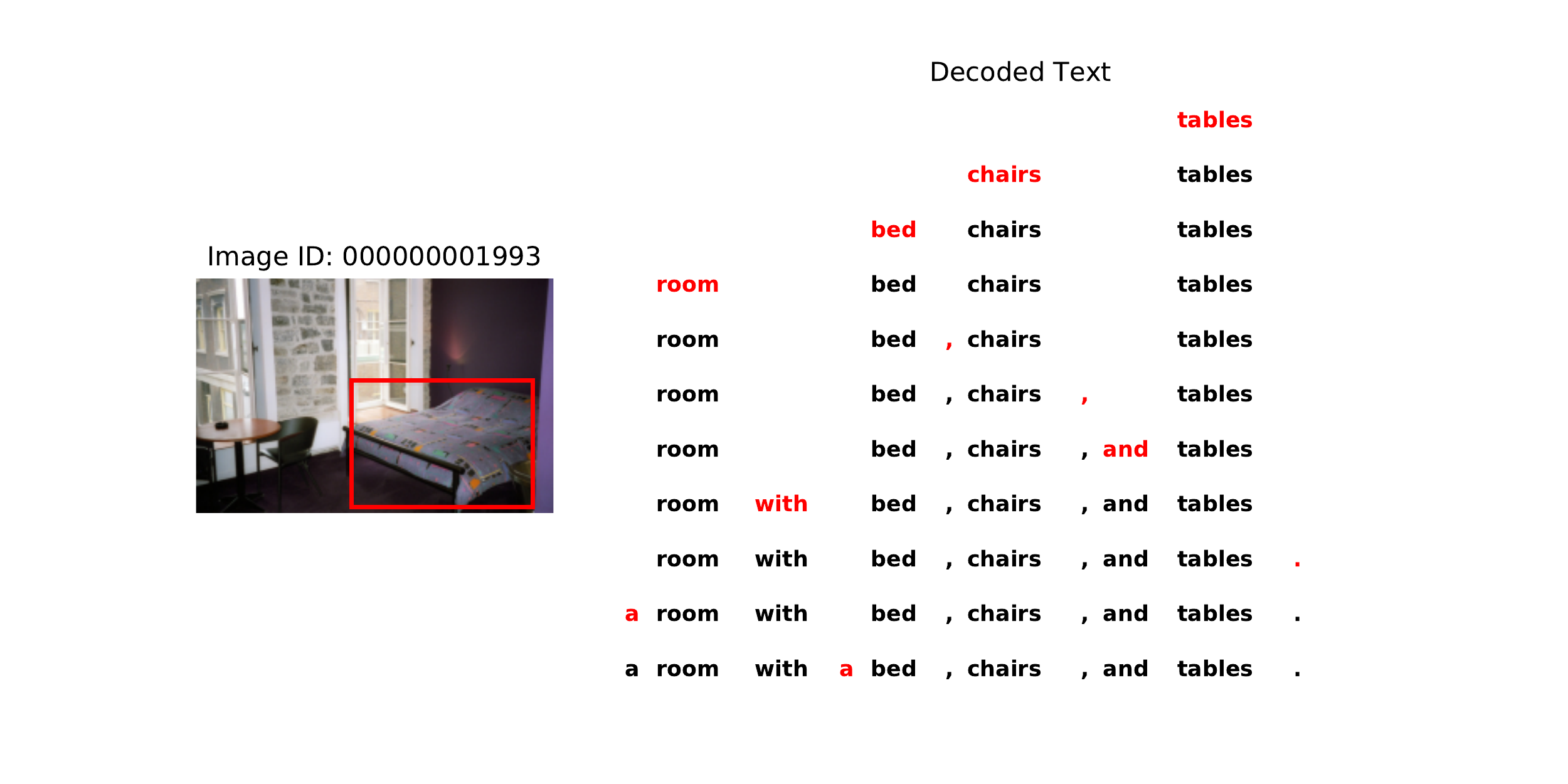}
    \caption{Generation order inferred by \textbf{Ours-Rare} for an image from the COCO 2017 validation set with the image identifier $\mathbf{000000001993}$. }
    \label{fig:000000001993_gen_order_rare}
\end{figure}

\clearpage
\subsection{Django}

We visualize the latent generation order inferred by Variational Order Inference for Django. Sequences are generated using a beam search over both the tokens and their insertion positions, using a beam size of 3. Text on which the model is conditioned is provided on the left for each example.

\begin{figure}[!htbp]
    \centering
    \includegraphics[width=\linewidth]{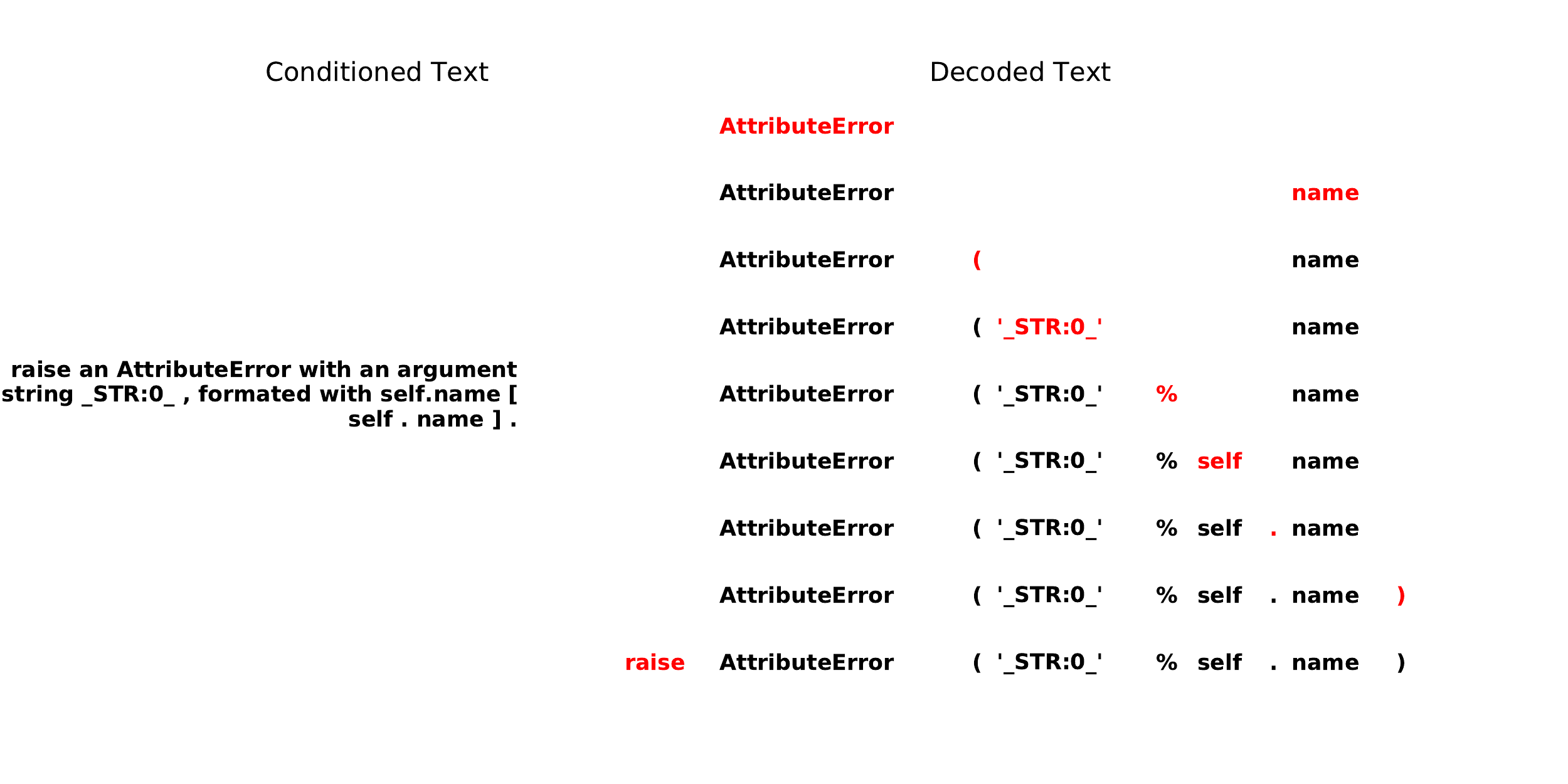}
    \caption{Generation order inferred by \textbf{Ours-VOI} for a pseudocode sample from the Django natural language to code test set with the sample id $\mathbf{154}$.}
    \label{fig:django/154_ours}
\end{figure}

\begin{figure}[!htbp]
    \centering
    \includegraphics[width=\linewidth]{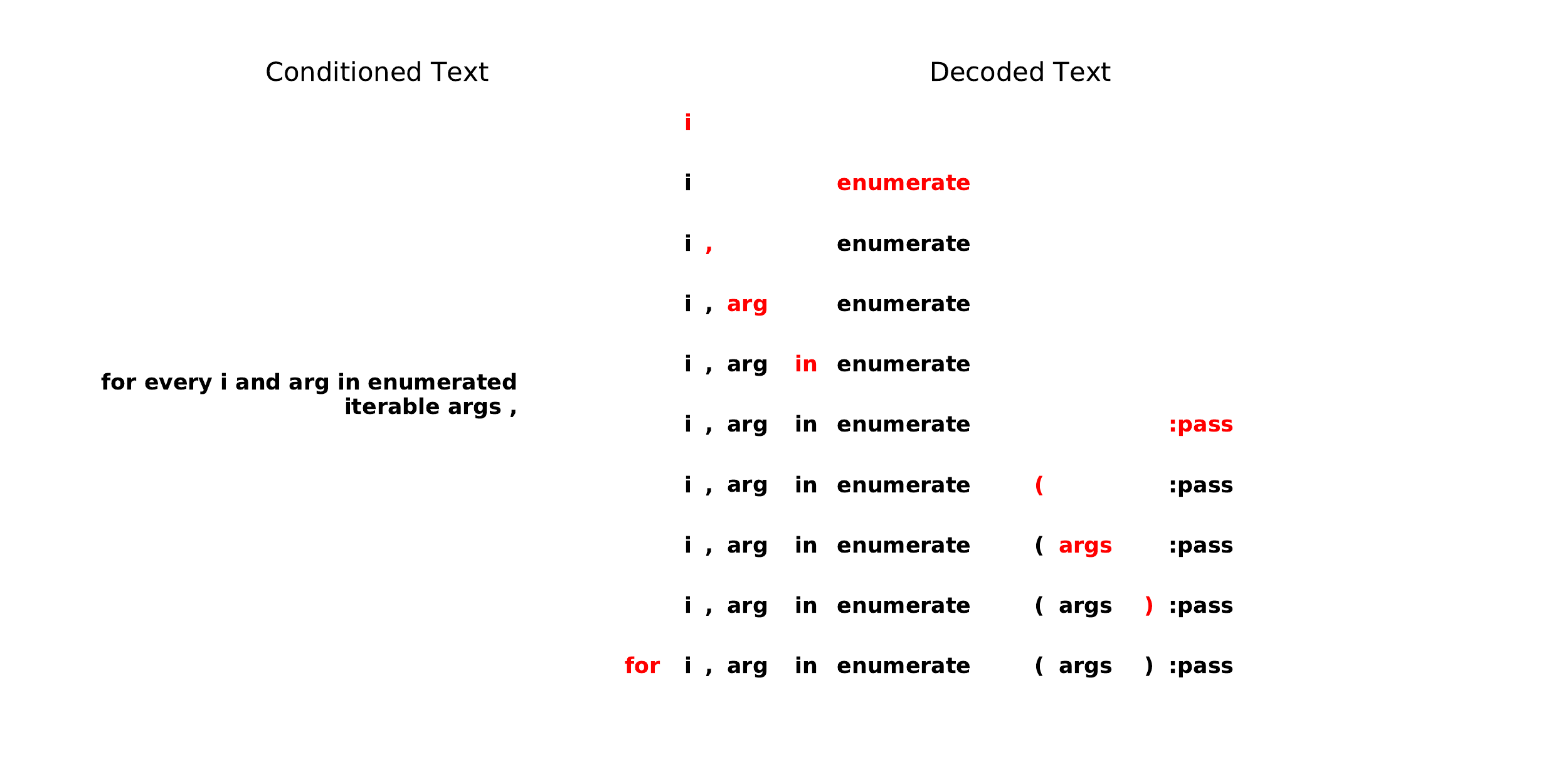}
    \caption{Generation order inferred by \textbf{Ours-VOI} for a pseudocode sample from the Django natural language to code test set with the sample id $\mathbf{431}$.}
    \label{fig:django/431_ours}
\end{figure}

\begin{figure}[!htbp]
    \centering
    \includegraphics[width=\linewidth]{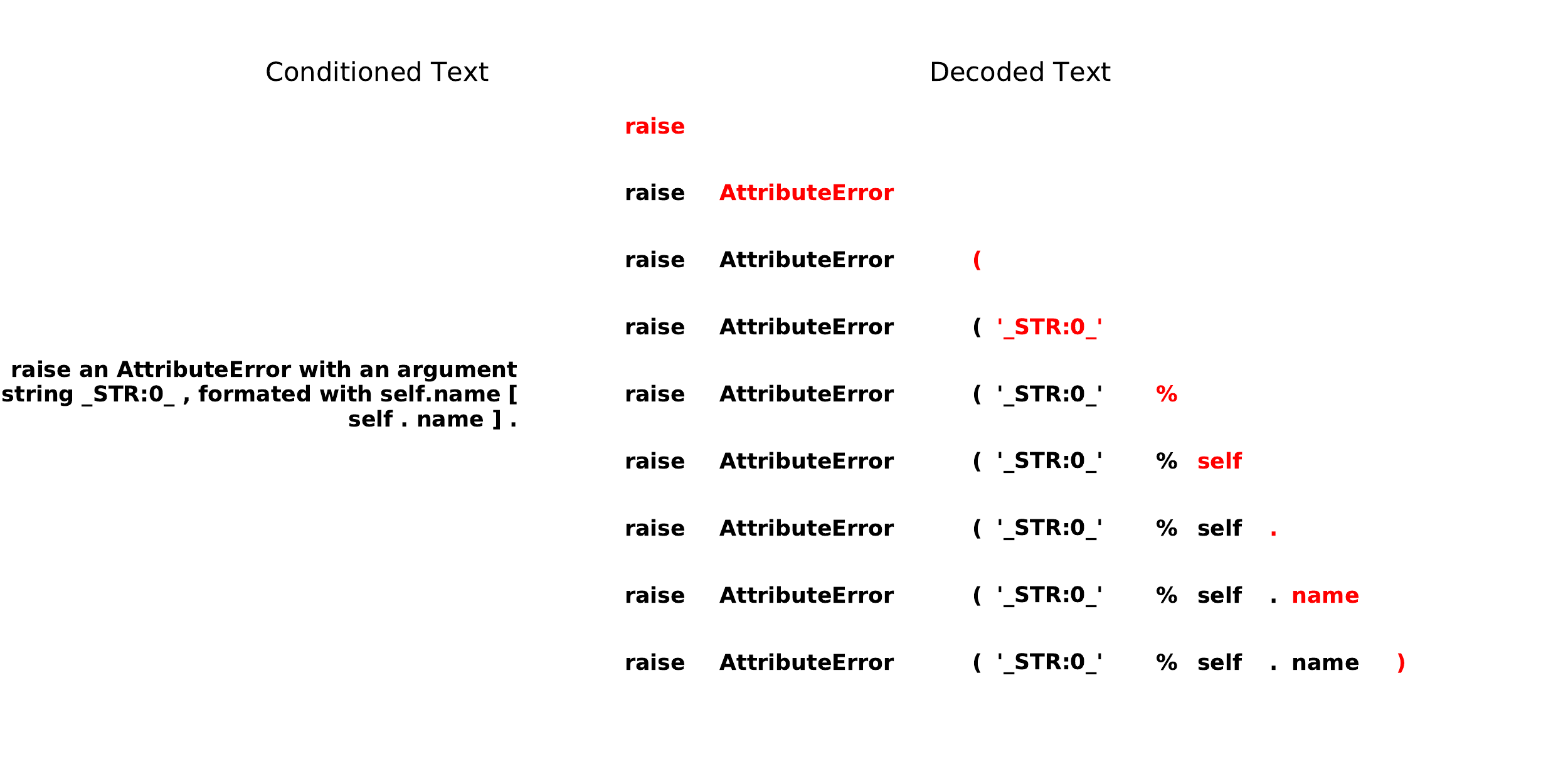}
    \caption{Generation order inferred by \textbf{Ours-L2R} for a pseudocode sample from the Django natural language to code test set with the sample id $\mathbf{154}$.}
    \label{fig:django/154_l2r}
\end{figure}

\begin{figure}[!htbp]
    \centering
    \includegraphics[width=\linewidth]{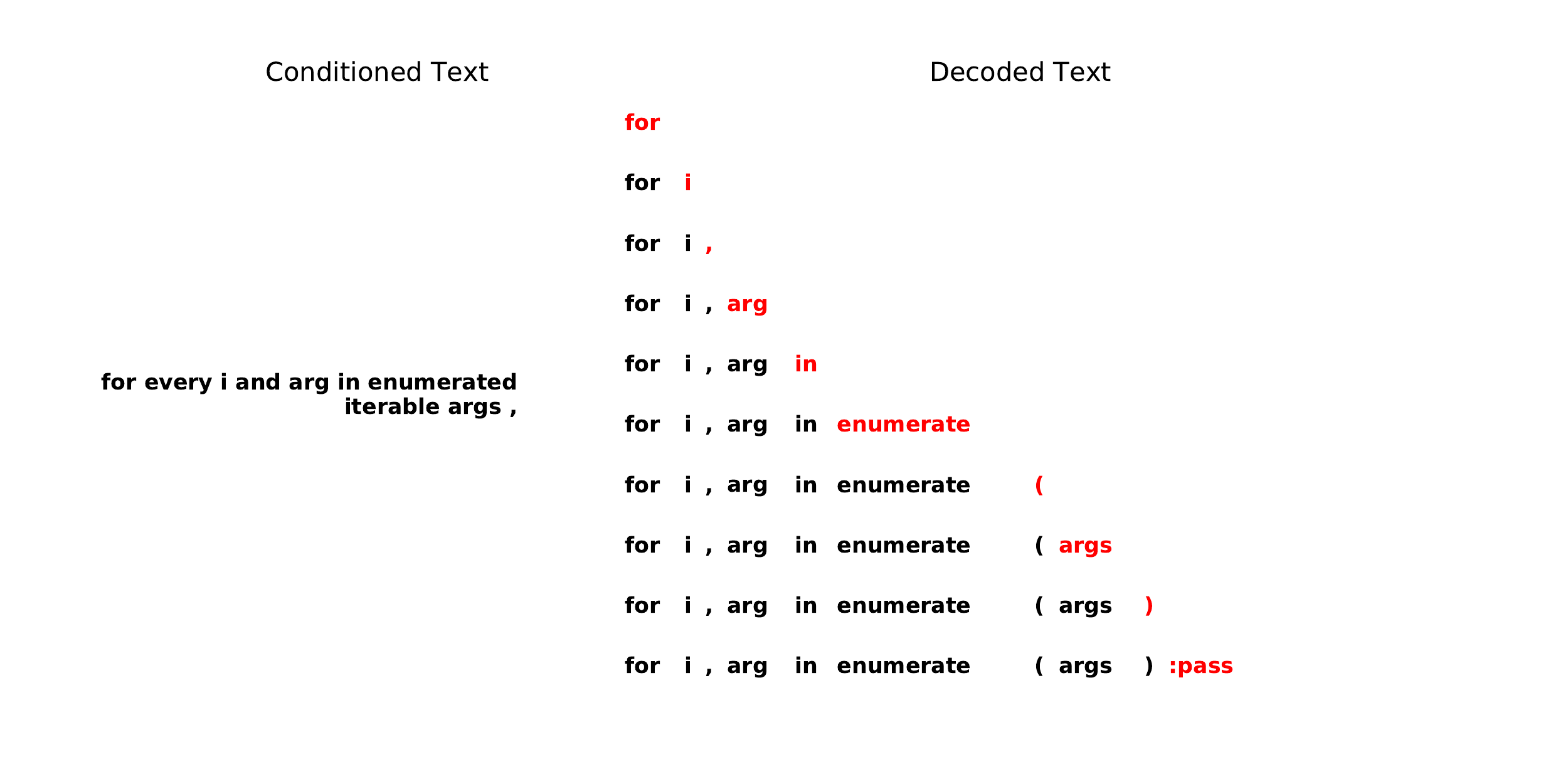}
    \caption{Generation order inferred by \textbf{Ours-L2R} for a pseudocode sample from the Django natural language to code test set with the sample id $\mathbf{431}$.}
    \label{fig:django/431_l2r}
\end{figure}

\begin{figure}[!htbp]
    \centering
    \includegraphics[width=\linewidth]{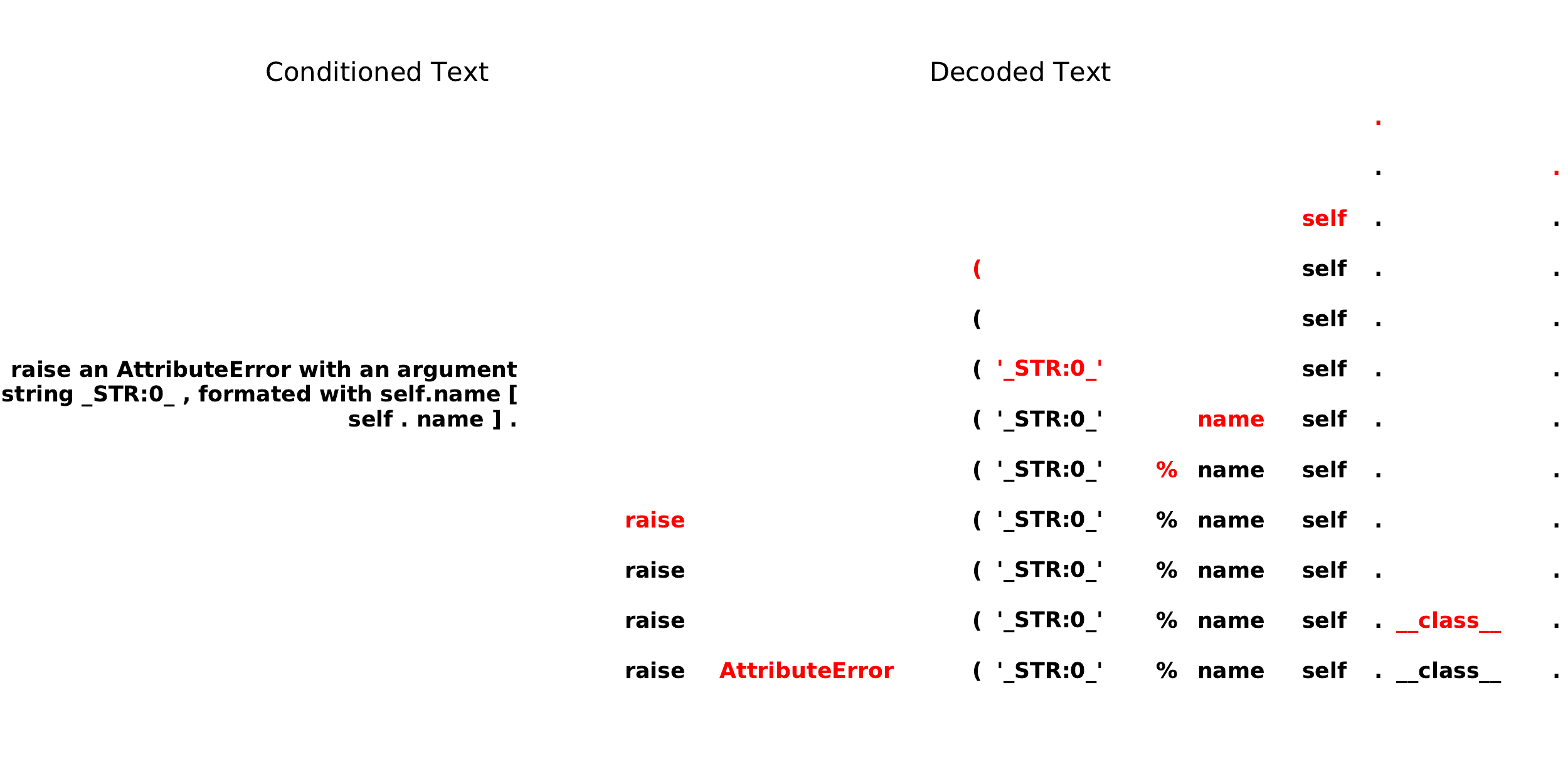}
    \caption{Generation order inferred by \textbf{Ours-Common} for a pseudocode sample from the Django natural language to code test set with the sample id $\mathbf{154}$.}
    \label{fig:django/154_common}
\end{figure}

\begin{figure}[!htbp]
    \centering
    \includegraphics[width=\linewidth]{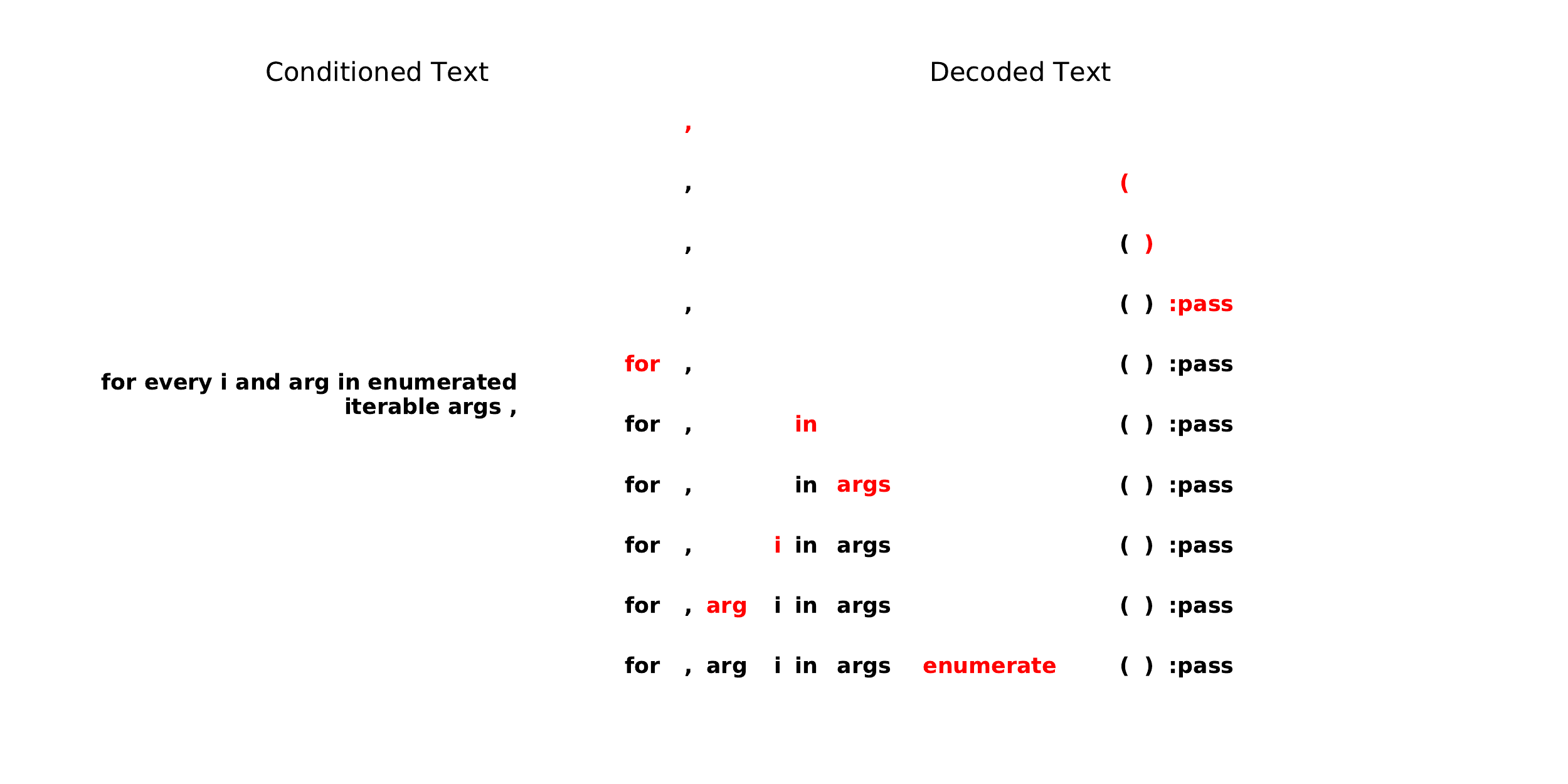}
    \caption{Generation order inferred by \textbf{Ours-Common} for a pseudocode sample from the Django natural language to code test set with the sample id $\mathbf{431}$.}
    \label{fig:django/431_common}
\end{figure}

\begin{figure}[!htbp]
    \centering
    \includegraphics[width=\linewidth]{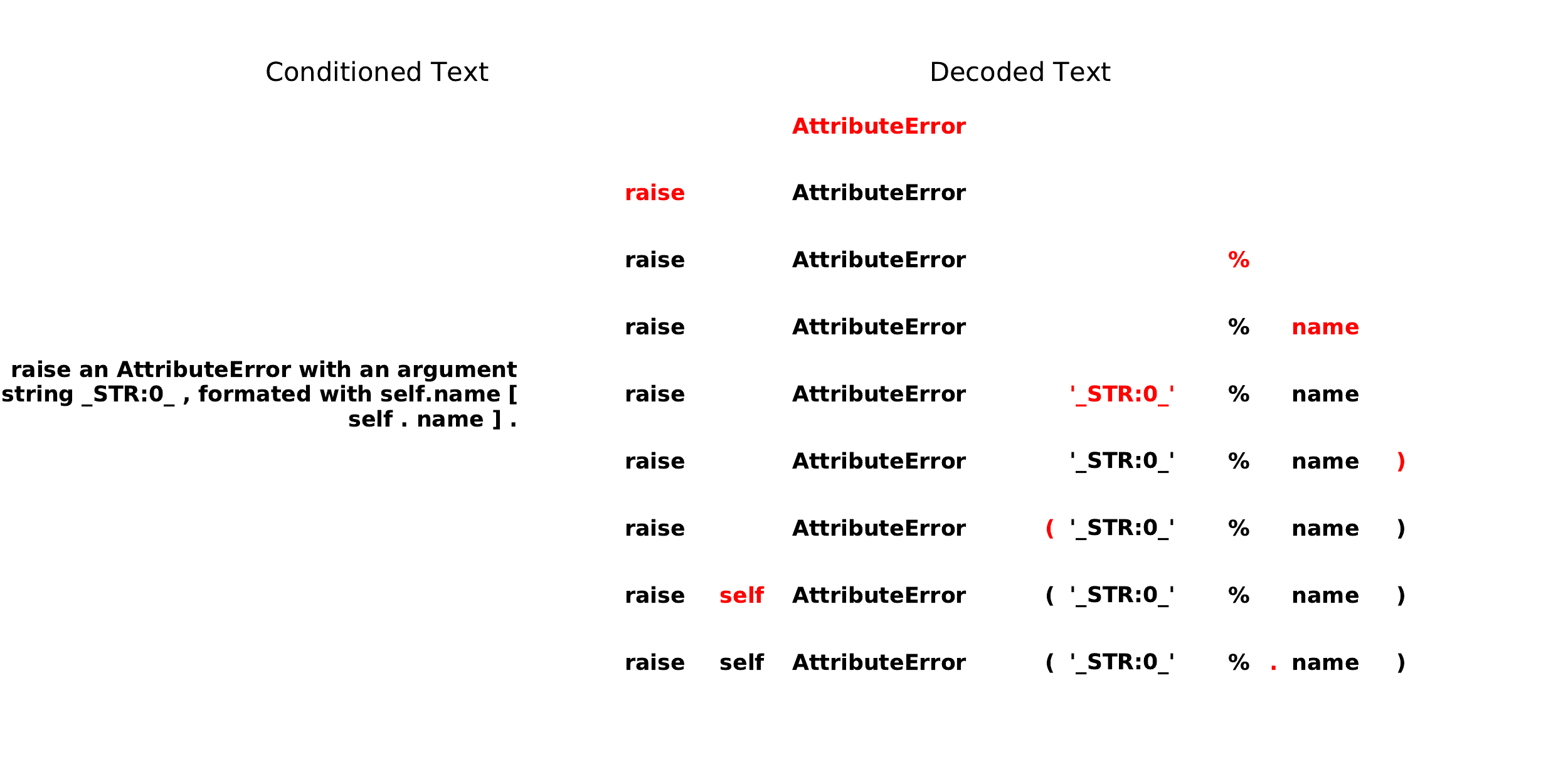}
    \caption{Generation order inferred by \textbf{Ours-Rare} for a pseudocode sample from the Django natural language to code test set with the sample id $\mathbf{154}$.}
    \label{fig:django/154_rare}
\end{figure}

\begin{figure}[!htbp]
    \centering
    \includegraphics[width=\linewidth]{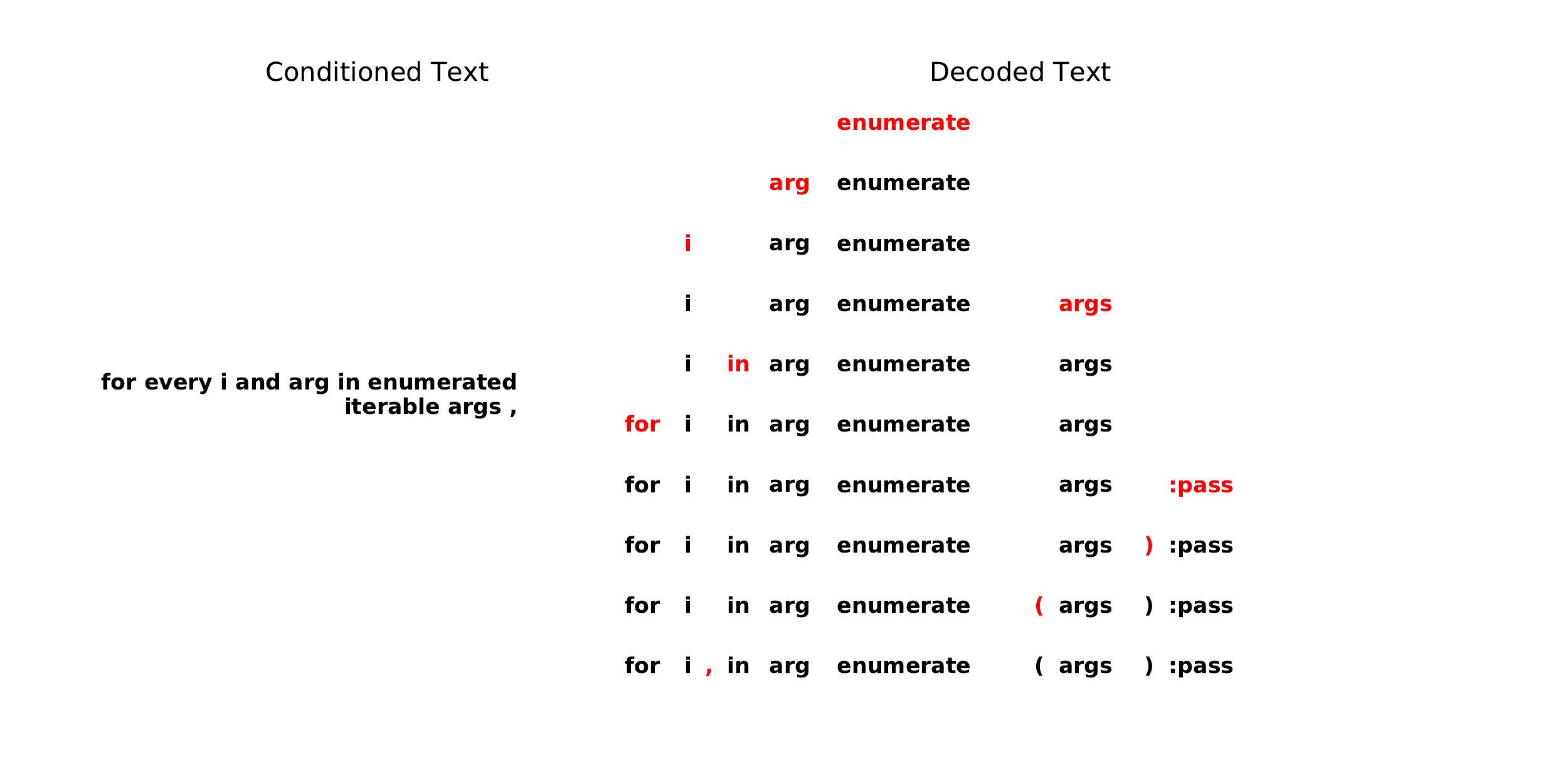}
    \caption{Generation order inferred by \textbf{Ours-Rare} for a pseudocode sample from the Django natural language to code test set with the sample id $\mathbf{431}$.}
    \label{fig:django/431_rare}
\end{figure}

\clearpage
\subsection{Gigaword}

We visualize the latent generation order inferred by Variational Order Inference for Gigaword. Sequences are generated using a beam search over both the tokens and their insertion positions, using a beam size of 1. Text on which the model is conditioned is provided on the left for each example.

\begin{figure}[!htbp]
    \centering
    \includegraphics[width=\linewidth]{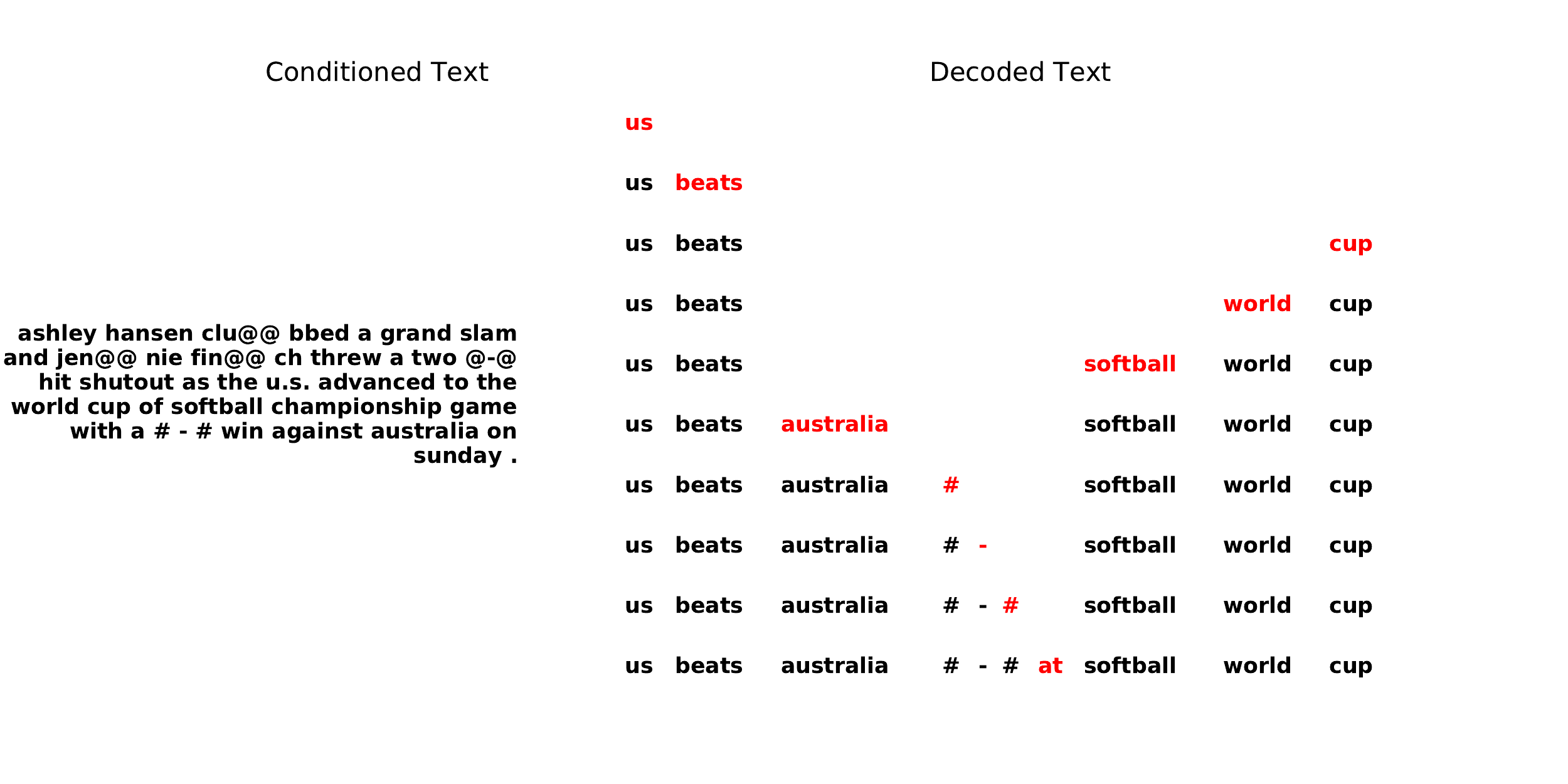}
    \caption{Generation order inferred by \textbf{Ours-VOI} for a text sample from the Gigaword text summarization test set with the sample id $\mathbf{15}$.}
    \label{fig:giga/15_ours}
\end{figure}

\begin{figure}[!htbp]
    \centering
    \includegraphics[width=\linewidth]{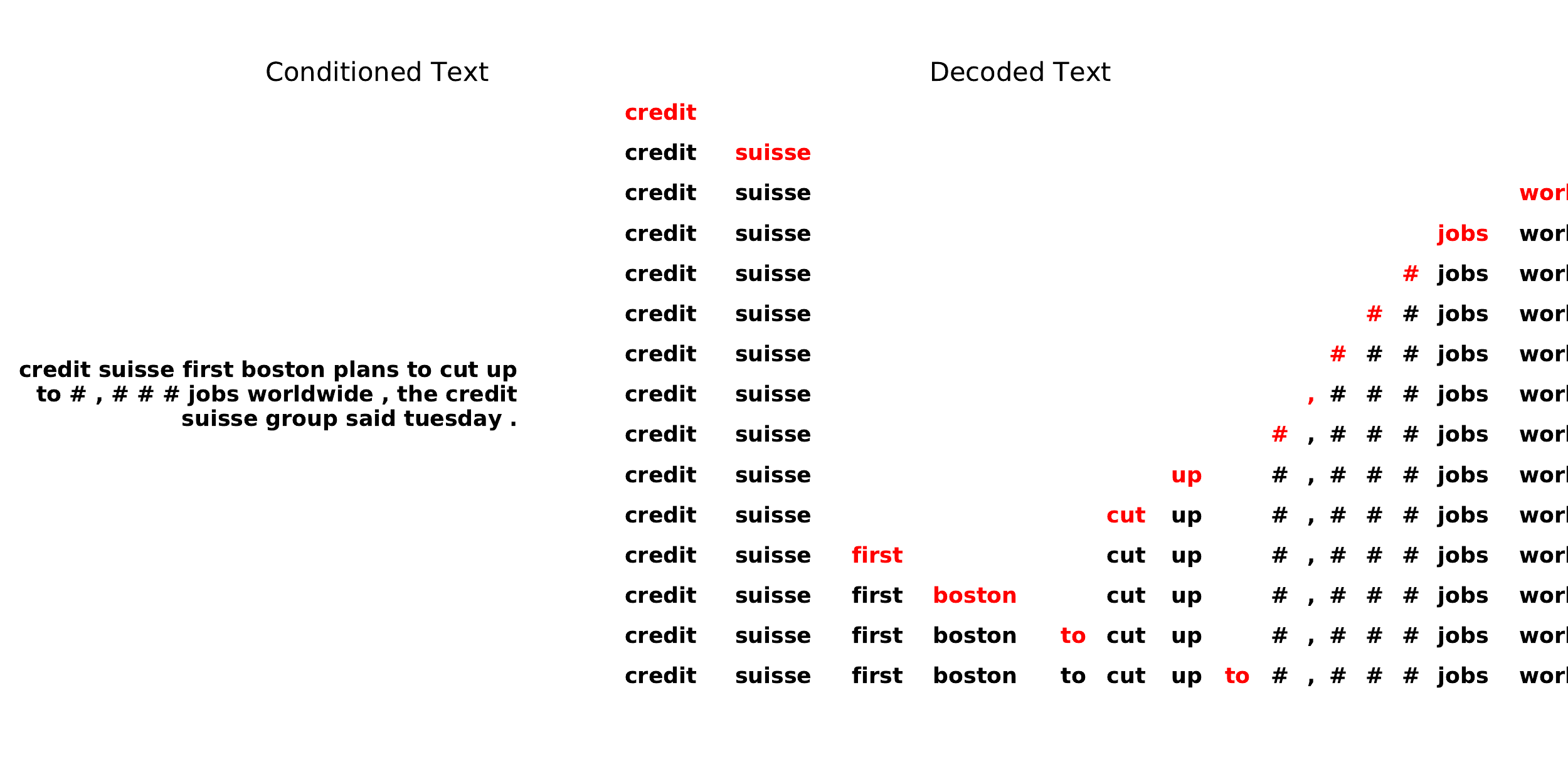}
    \caption{Generation order inferred by \textbf{Ours-VOI} for a text sample from the Gigaword text summarization test set with the sample id $\mathbf{33}$.}
    \label{fig:giga/33_ours}
\end{figure}

\begin{figure}[!htbp]
    \centering
    \includegraphics[width=\linewidth]{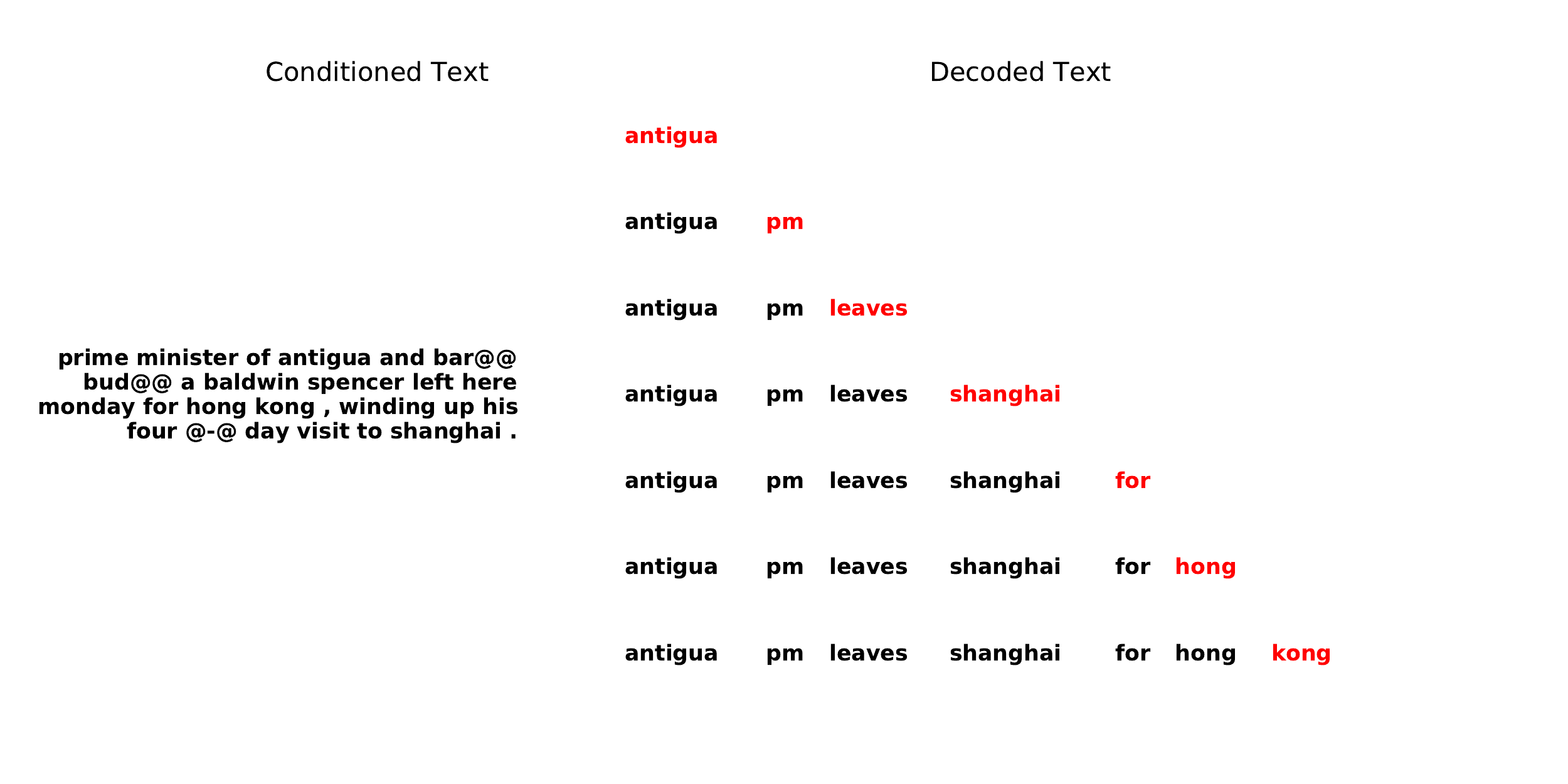}
    \caption{Generation order inferred by \textbf{Ours-L2R} for a text sample from the Gigaword text summarization test set with the sample id $\mathbf{15}$.}
    \label{fig:giga/15_l2r}
\end{figure}

\begin{figure}[!htbp]
    \centering
    \includegraphics[width=\linewidth]{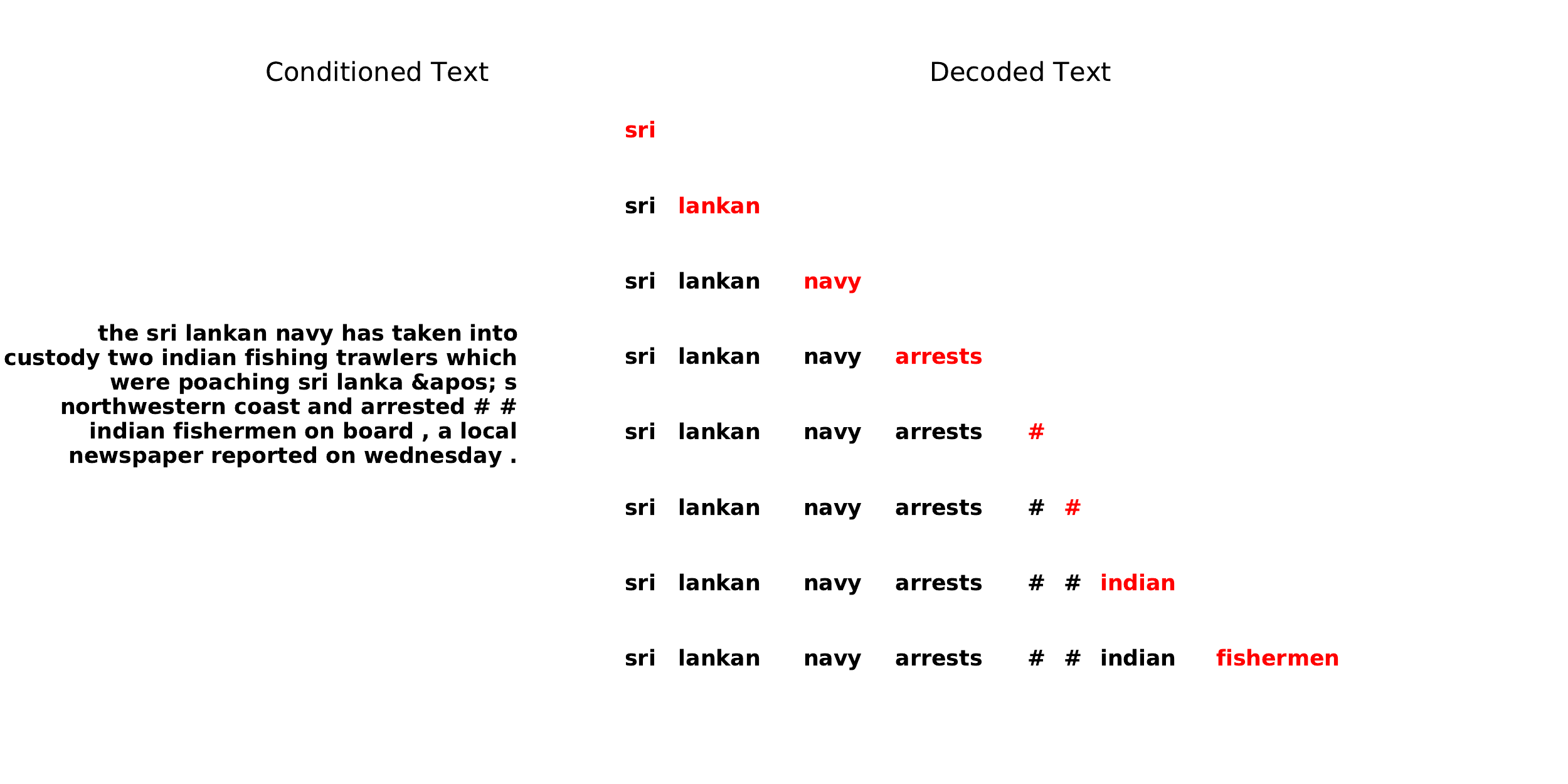}
    \caption{Generation order inferred by \textbf{Ours-L2R} for a text sample from the Gigaword text summarization test set with the sample id $\mathbf{33}$.}
    \label{fig:giga/33_l2r}
\end{figure}

\begin{figure}[!htbp]
    \centering
    \includegraphics[width=\linewidth]{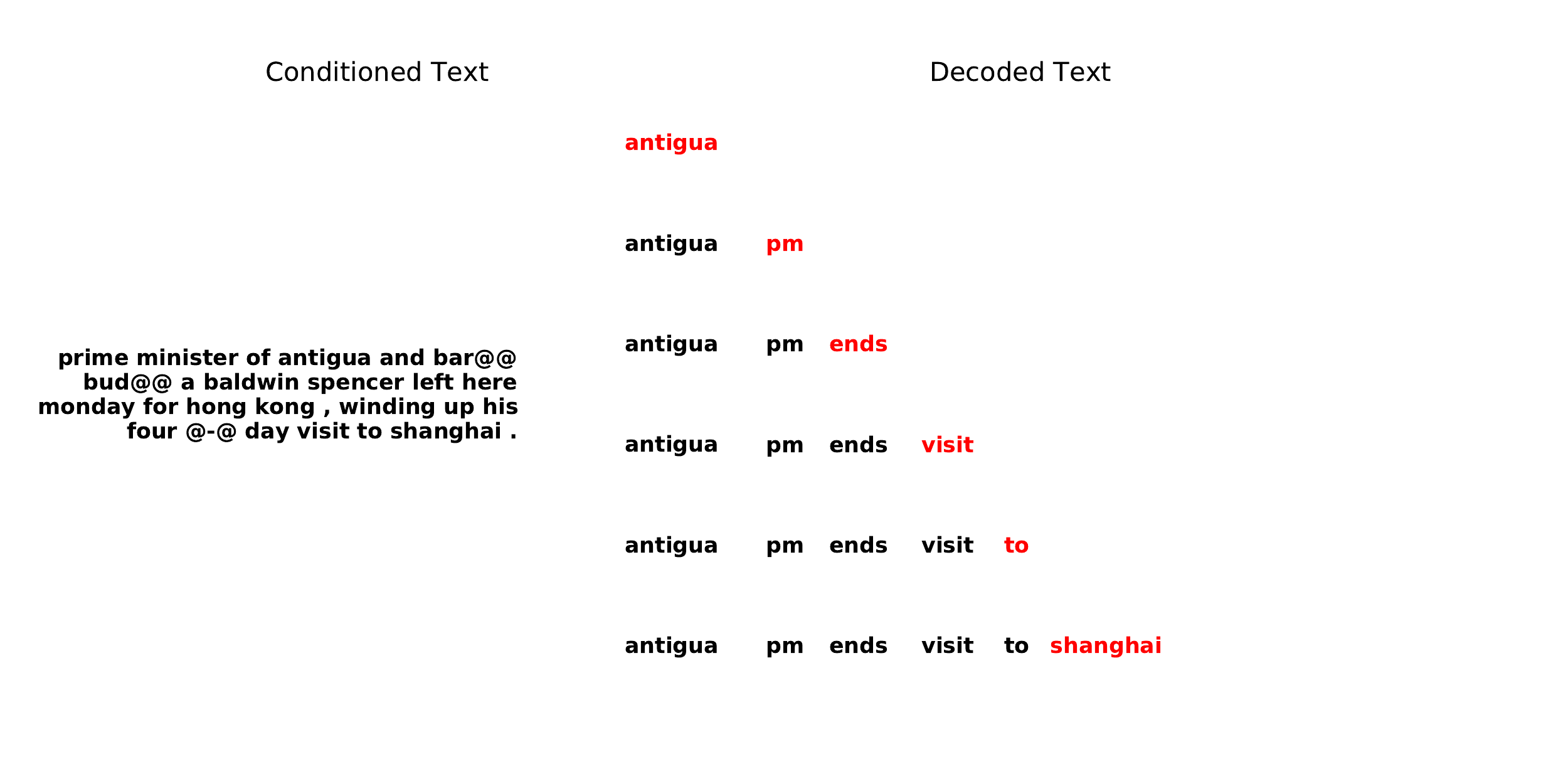}
    \caption{Generation order inferred by \textbf{Ours-Common} for a text sample from the Gigaword text summarization test set with the sample id $\mathbf{15}$.}
    \label{fig:giga/15_common}
\end{figure}

\begin{figure}[!htbp]
    \centering
    \includegraphics[width=\linewidth]{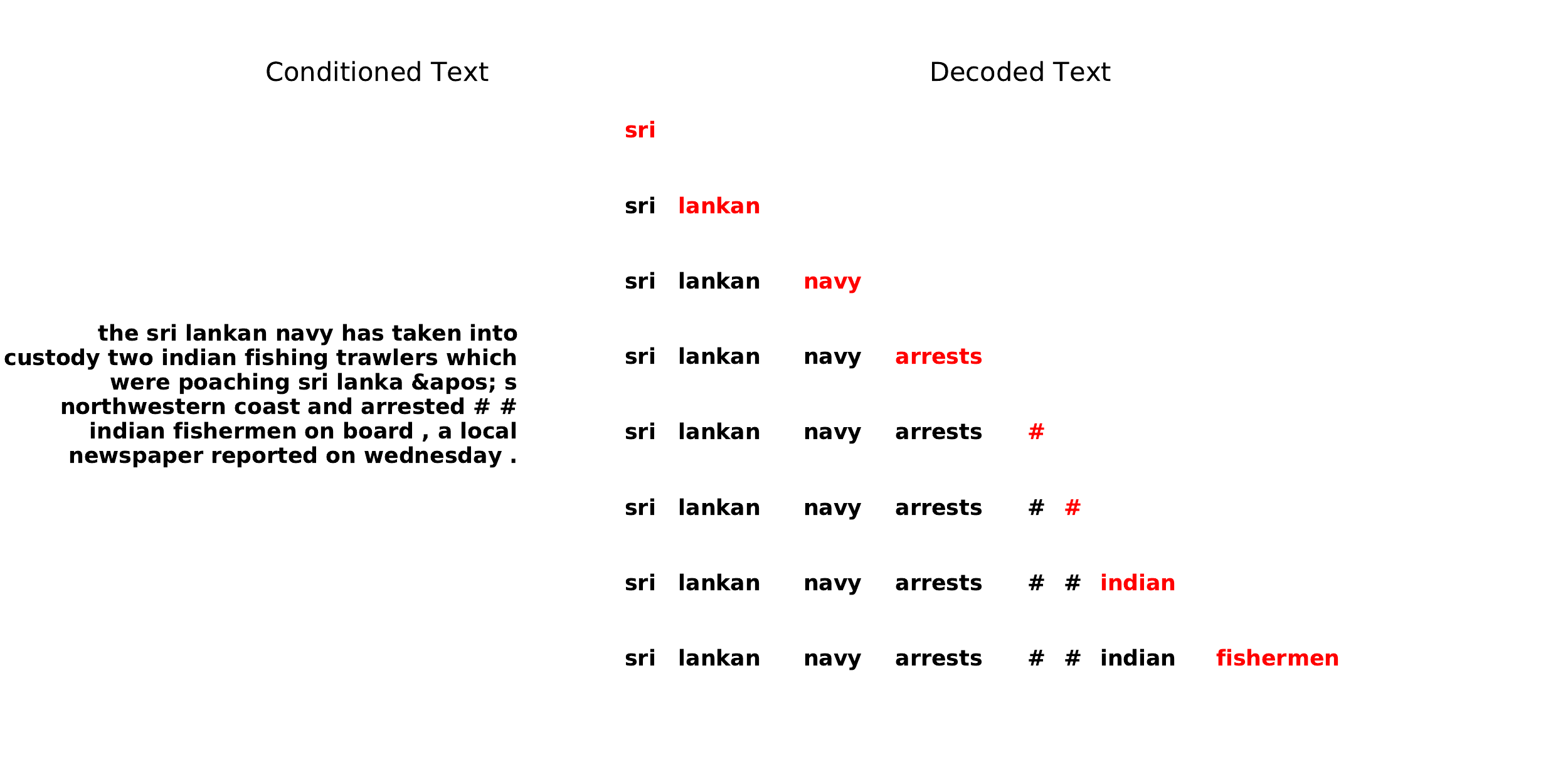}
    \caption{Generation order inferred by \textbf{Ours-Common} for a text sample from the Gigaword text summarization test set with the sample id $\mathbf{33}$.}
    \label{fig:giga/33_common}
\end{figure}

\begin{figure}[!htbp]
    \centering
    \includegraphics[width=\linewidth]{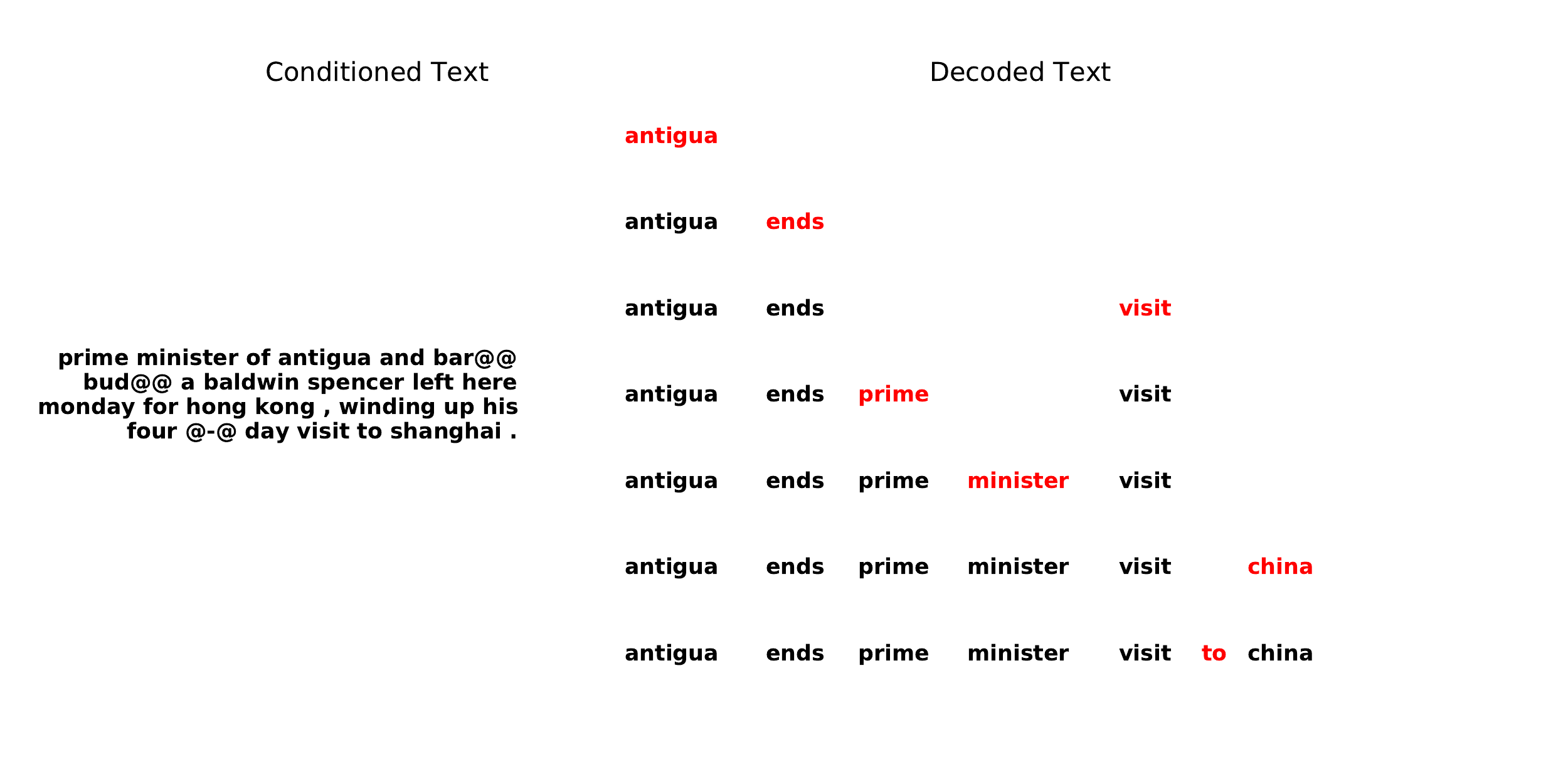}
    \caption{Generation order inferred by \textbf{Ours-Rare} for a text sample from the Gigaword text summarization test set with the sample id $\mathbf{15}$.}
    \label{fig:giga/15_rare}
\end{figure}

\begin{figure}[!htbp]
    \centering
    \includegraphics[width=\linewidth]{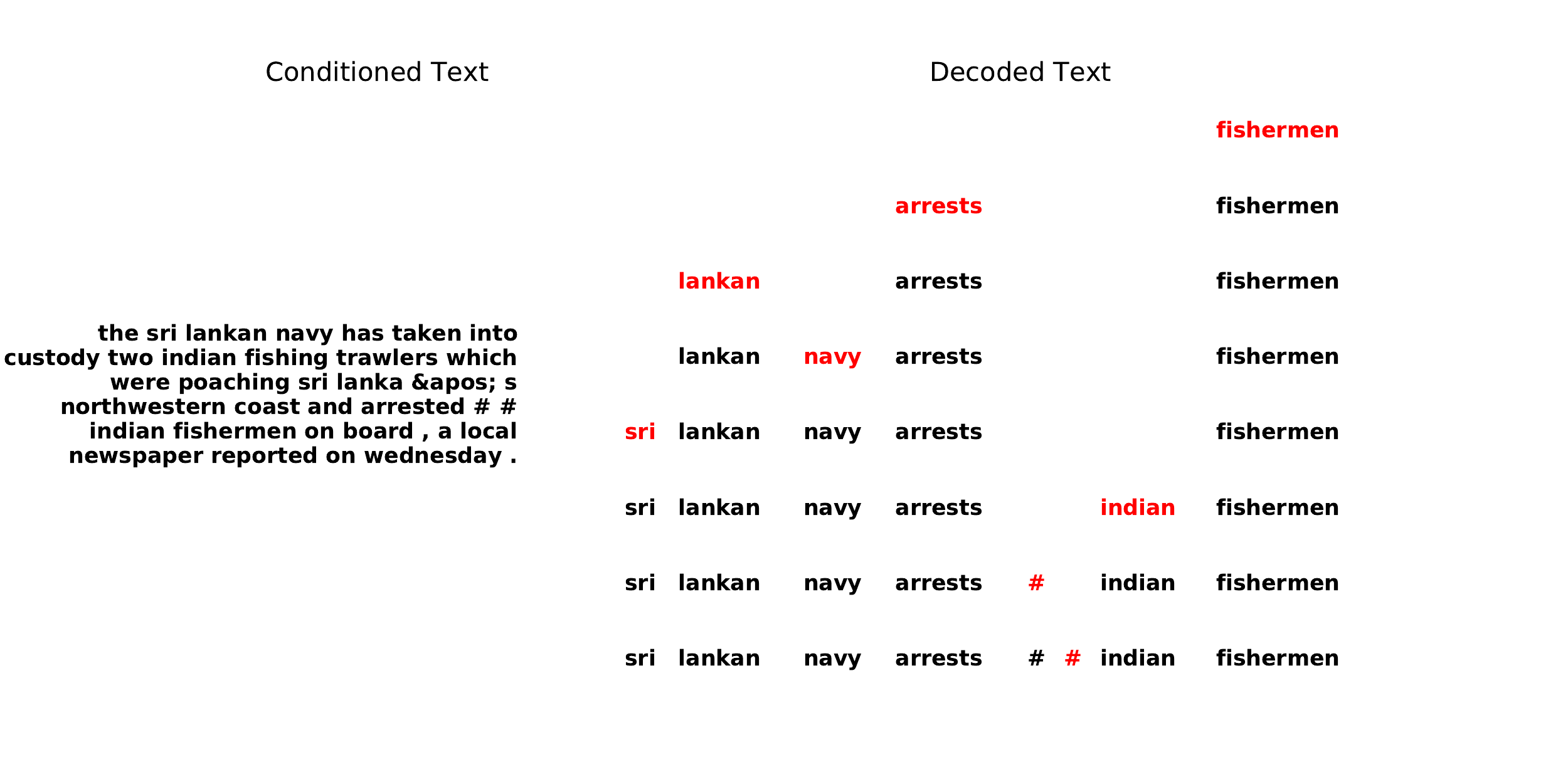}
    \caption{Generation order inferred by \textbf{Ours-Rare} for a text sample from the Gigaword text summarization test set with the sample id $\mathbf{33}$.}
    \label{fig:giga/33_rare}
\end{figure}

\clearpage
\subsection{WMT 16 Ro-En}

We visualize the generation order inferred by Variational Order Inference for WMT 16 Ro-En.  Sequences are generated using a beam search over both the tokens and their insertion positions, using a beam size of 5. Text on which the model is conditioned is provided on the left.

\begin{figure}[!htbp]
    \centering
    \includegraphics[width=\linewidth]{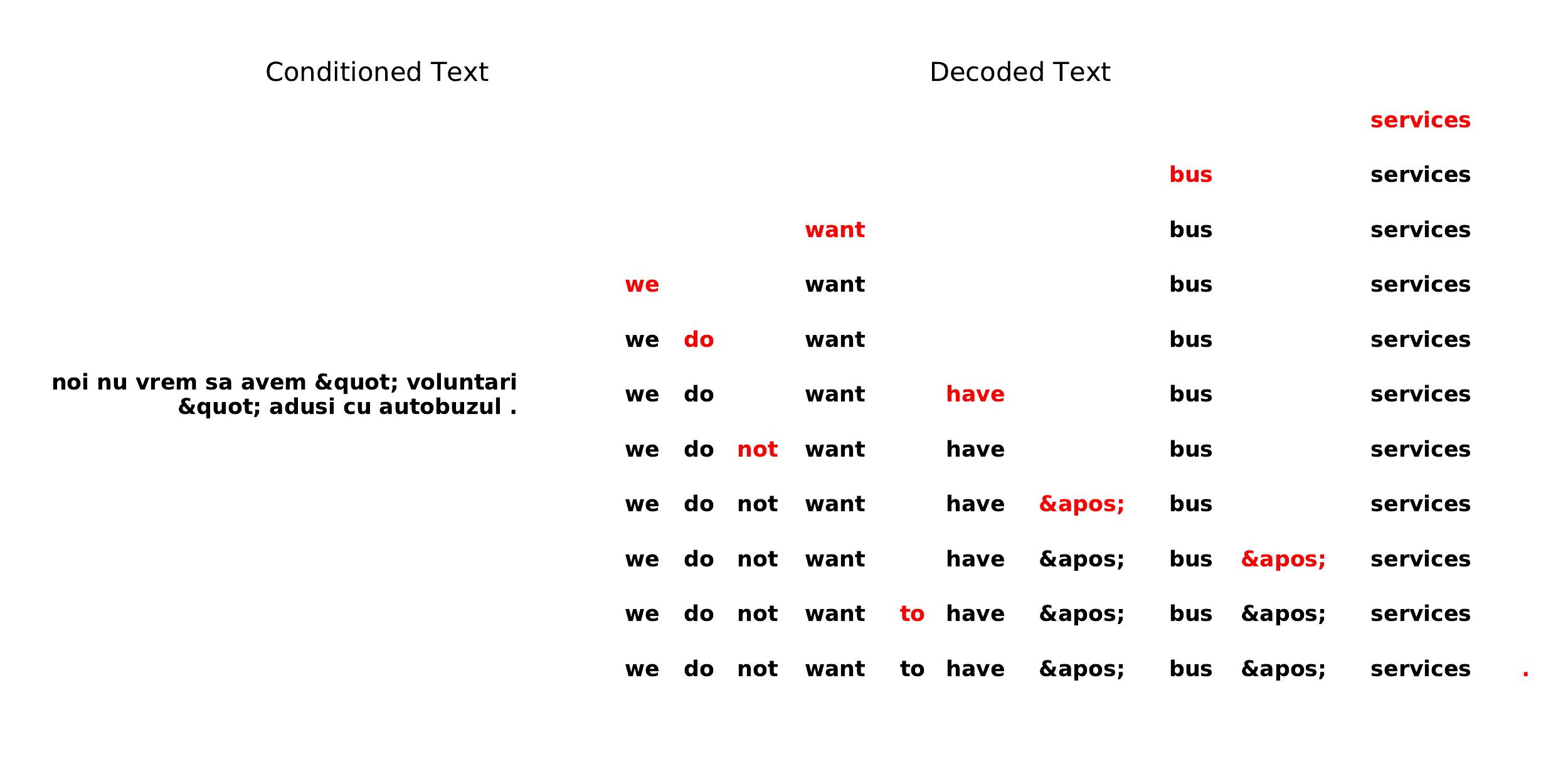}
    \caption{Generation order inferred by \textbf{Ours-VOI} for a text sample from the WMT validation set with the sample id $\mathbf{71}$.}
    \label{fig:wmt/71_ours}
\end{figure}

\begin{figure}[!htbp]
    \centering
    \includegraphics[width=\linewidth]{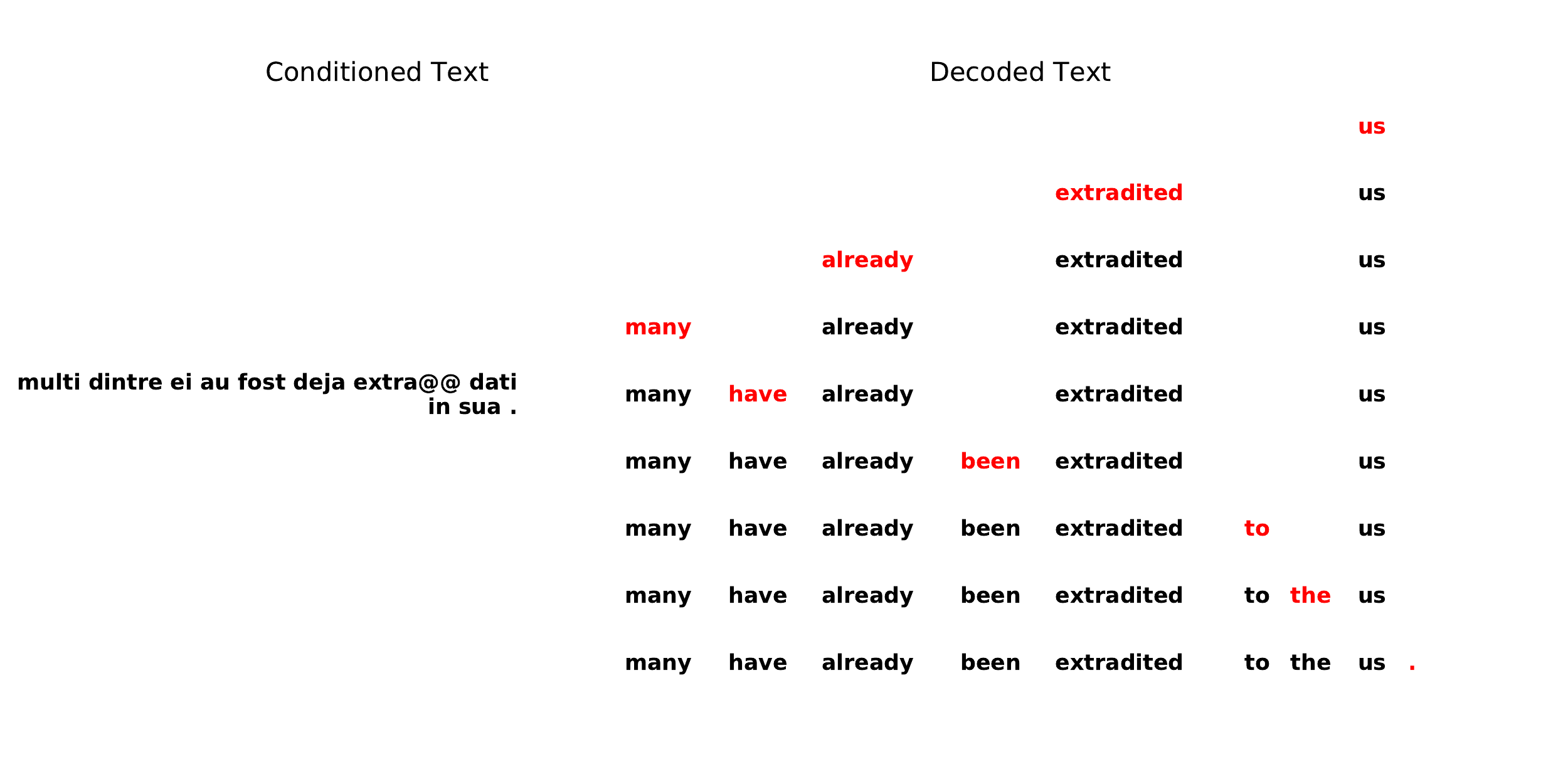}
    \caption{Generation order inferred by \textbf{Ours-VOI} for a text sample from the WMT validation set with the sample id $\mathbf{72}$.}
    \label{fig:wmt/72_ours}
\end{figure}

\begin{figure}[!htbp]
    \centering
    \includegraphics[width=\linewidth]{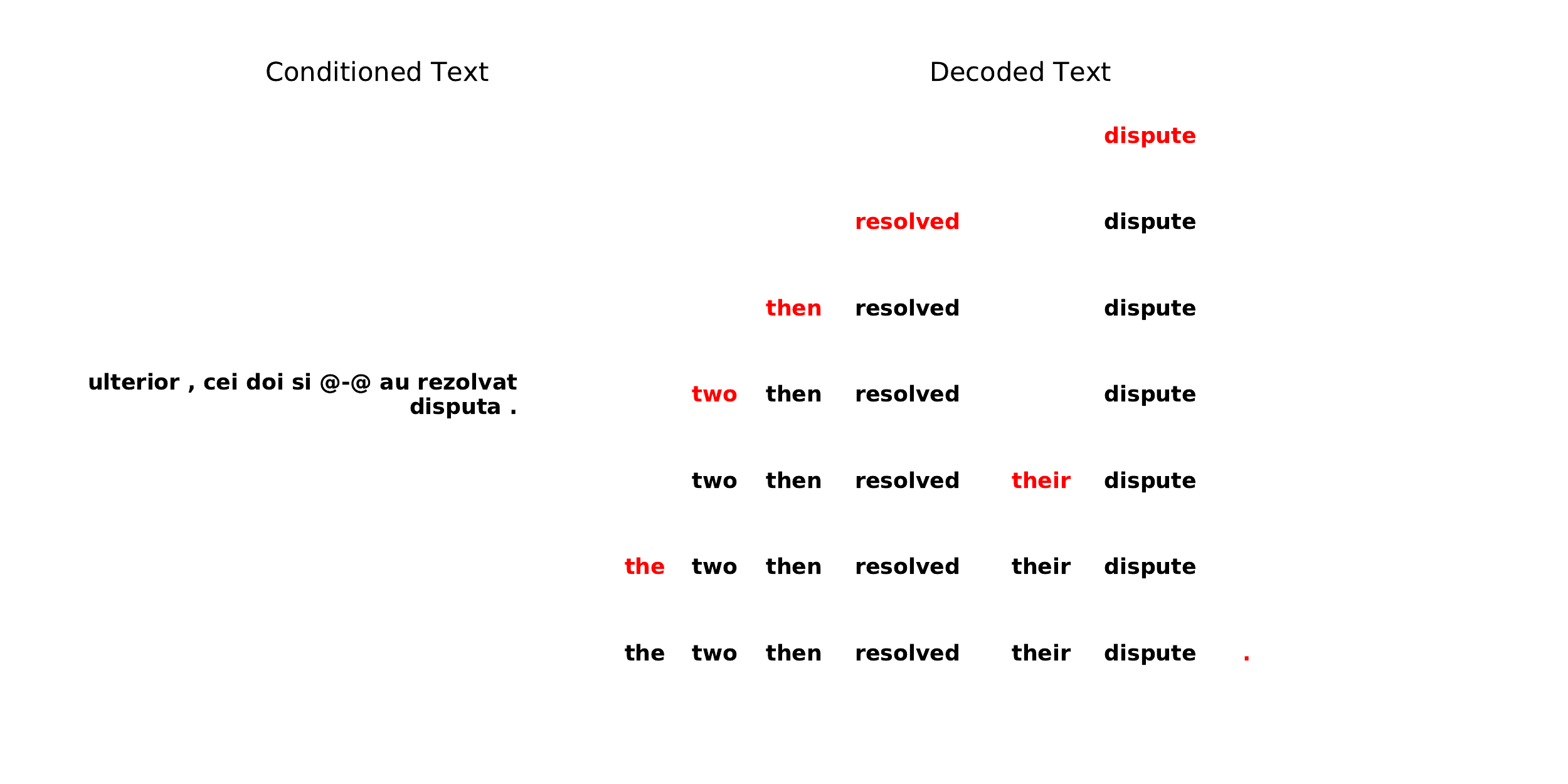}
    \caption{Generation order inferred by \textbf{Ours-VOI} for a text sample from the WMT validation set with the sample id $\mathbf{74}$.}
    \label{fig:wmt/74_ours}
\end{figure}

\begin{figure}[!htbp]
    \centering
    \includegraphics[width=\linewidth]{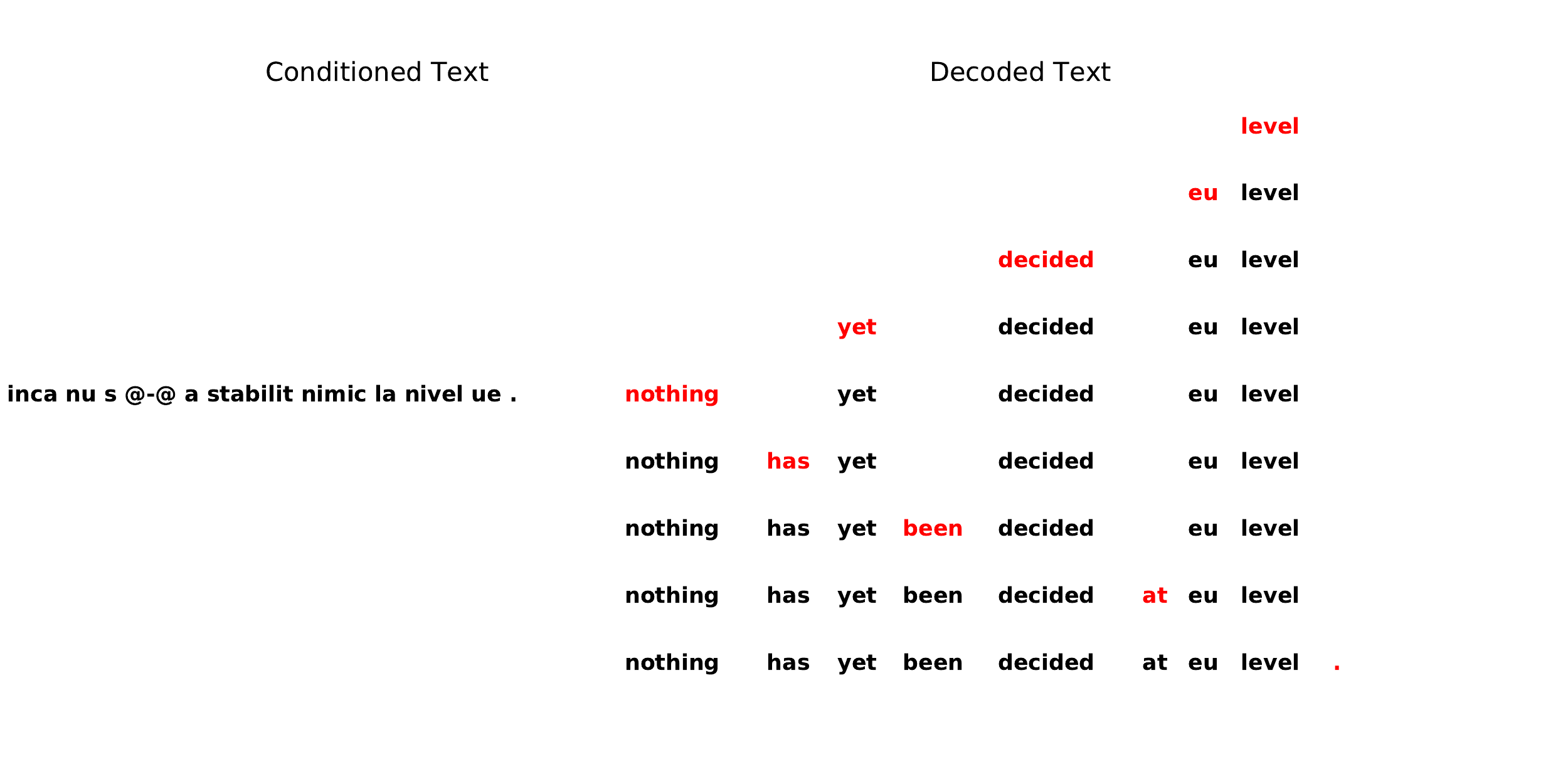}
    \caption{Generation order inferred by \textbf{Ours-VOI} for a text sample from the WMT validation set with the sample id $\mathbf{110}$.}
    \label{fig:wmt/110_ours}
\end{figure}

\begin{figure}[!htbp]
    \centering
    \includegraphics[width=\linewidth]{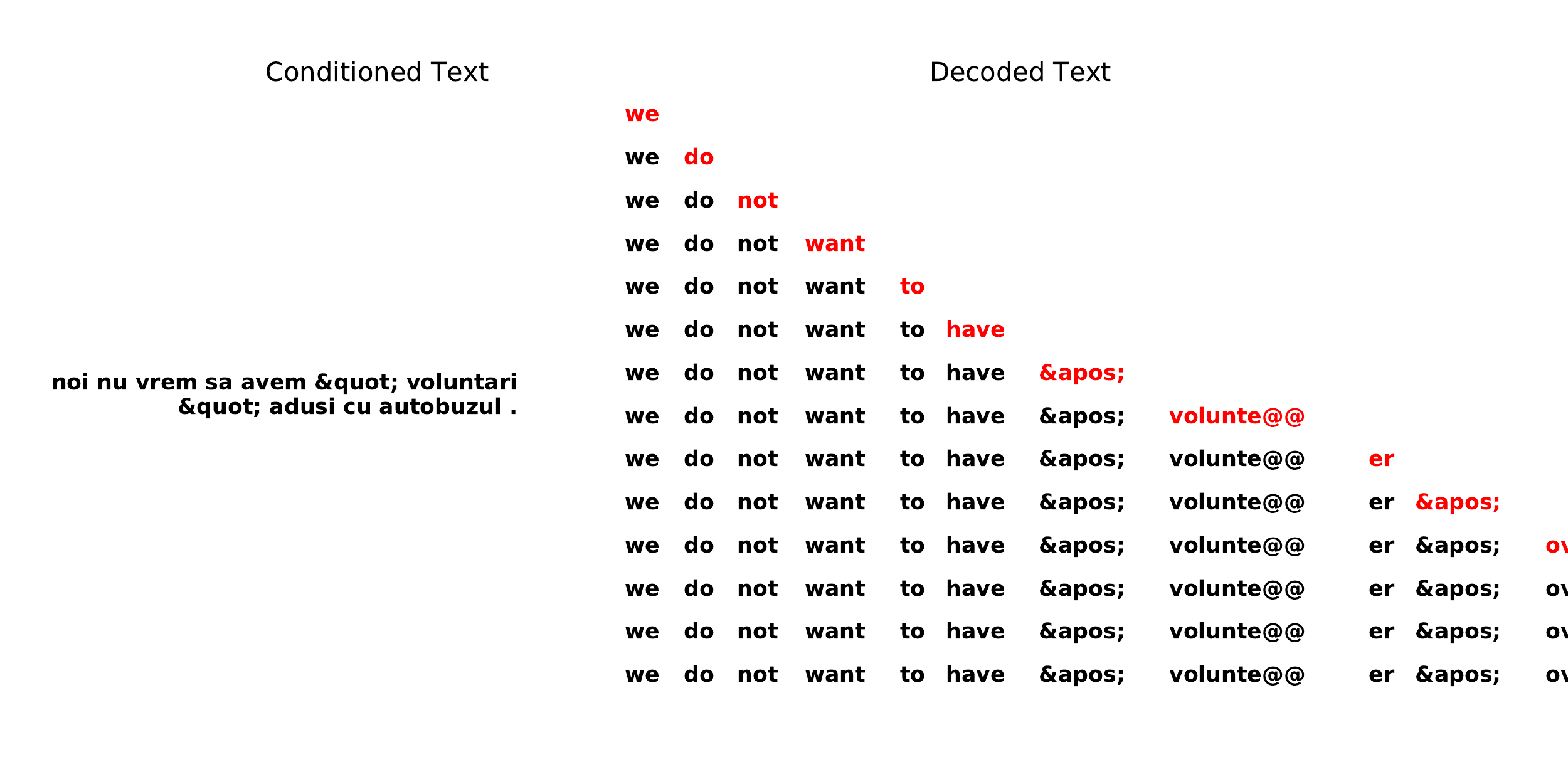}
    \caption{Generation order inferred by \textbf{Ours-L2R} for a text sample from the WMT validation set with the sample id $\mathbf{71}$.}
    \label{fig:wmt/71_l2r}
\end{figure}

\begin{figure}[!htbp]
    \centering
    \includegraphics[width=\linewidth]{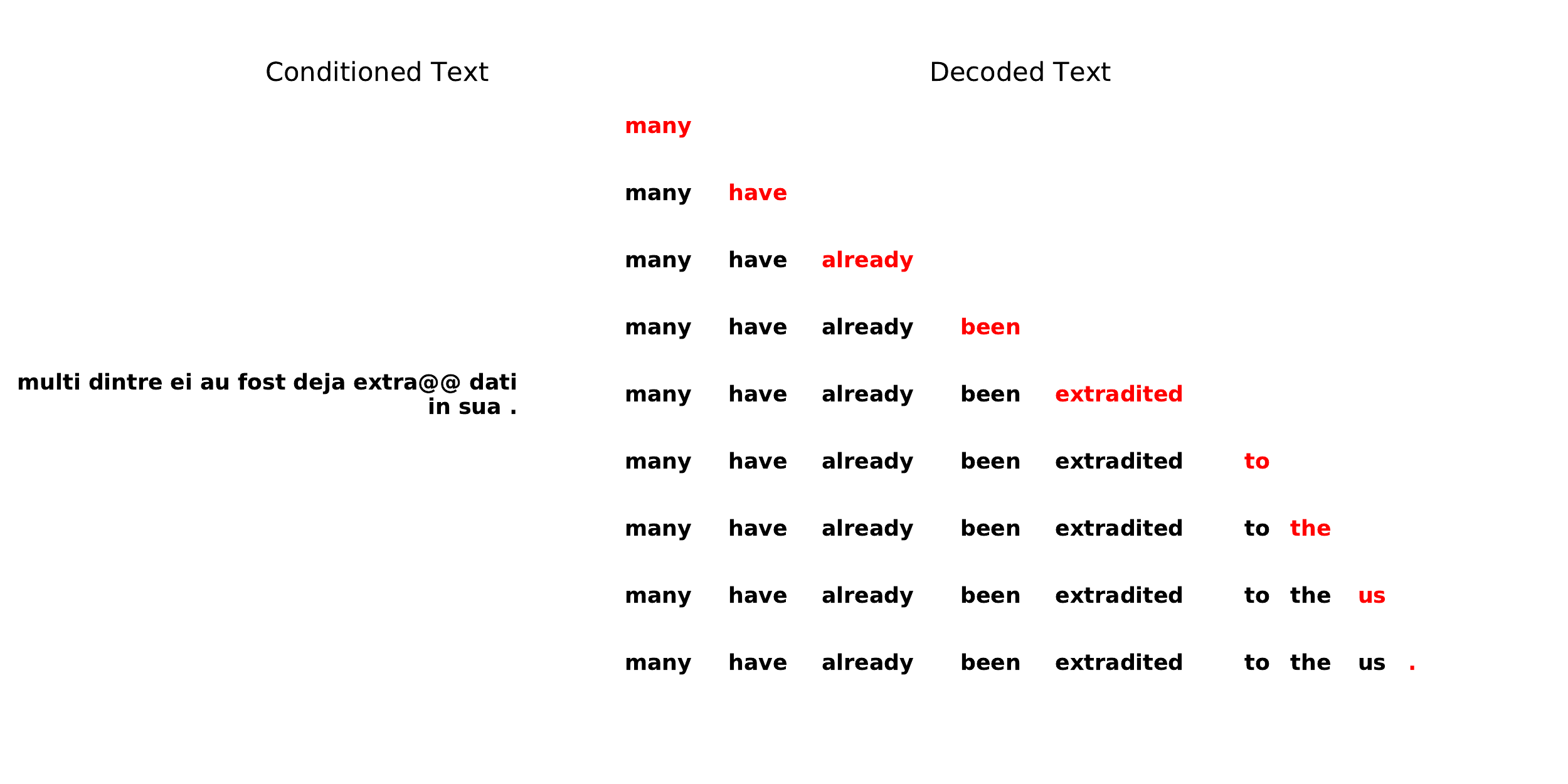}
    \caption{Generation order inferred by \textbf{Ours-L2R} for a text sample from the WMT validation set with the sample id $\mathbf{72}$.}
    \label{fig:wmt/72_l2r}
\end{figure}

\begin{figure}[!htbp]
    \centering
    \includegraphics[width=\linewidth]{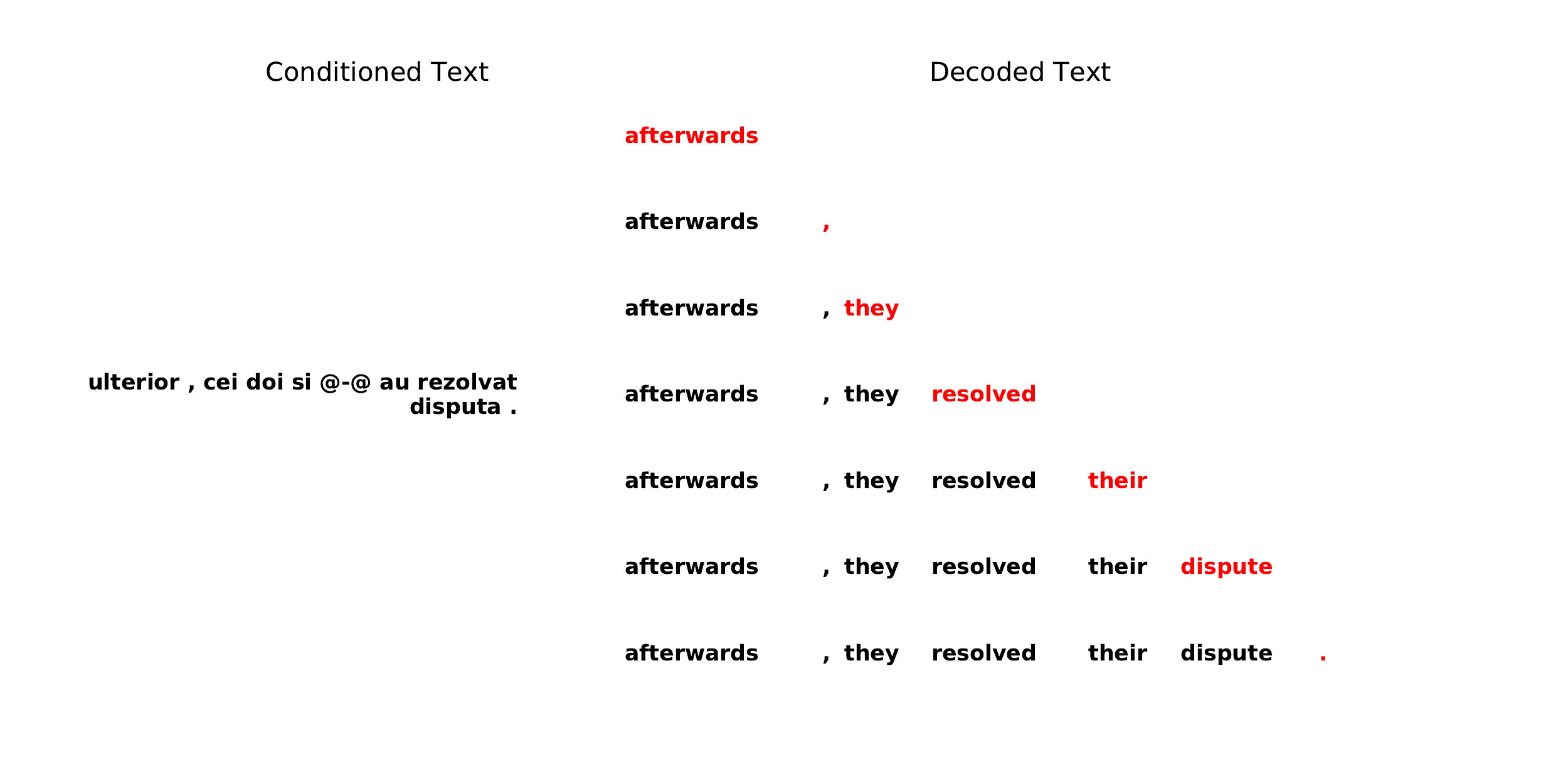}
    \caption{Generation order inferred by \textbf{Ours-L2R} for a text sample from the WMT validation set with the sample id $\mathbf{74}$.}
    \label{fig:wmt/74_l2r}
\end{figure}

\begin{figure}[!htbp]
    \centering
    \includegraphics[width=\linewidth]{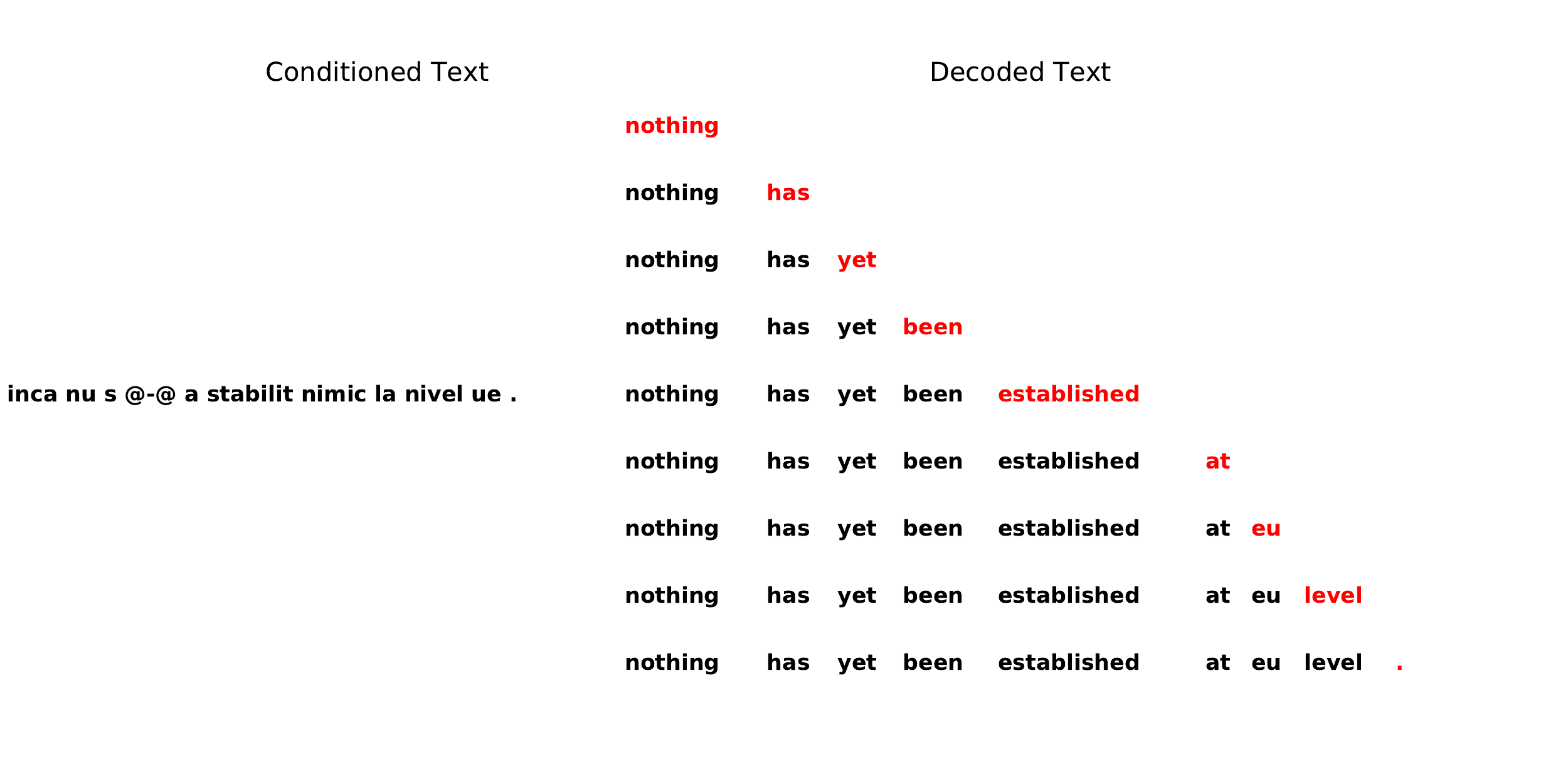}
    \caption{Generation order inferred by \textbf{Ours-L2R} for a text sample from the WMT validation set with the sample id $\mathbf{110}$.}
    \label{fig:wmt/110_l2r}
\end{figure}

\begin{figure}[!htbp]
    \centering
    \includegraphics[width=\linewidth]{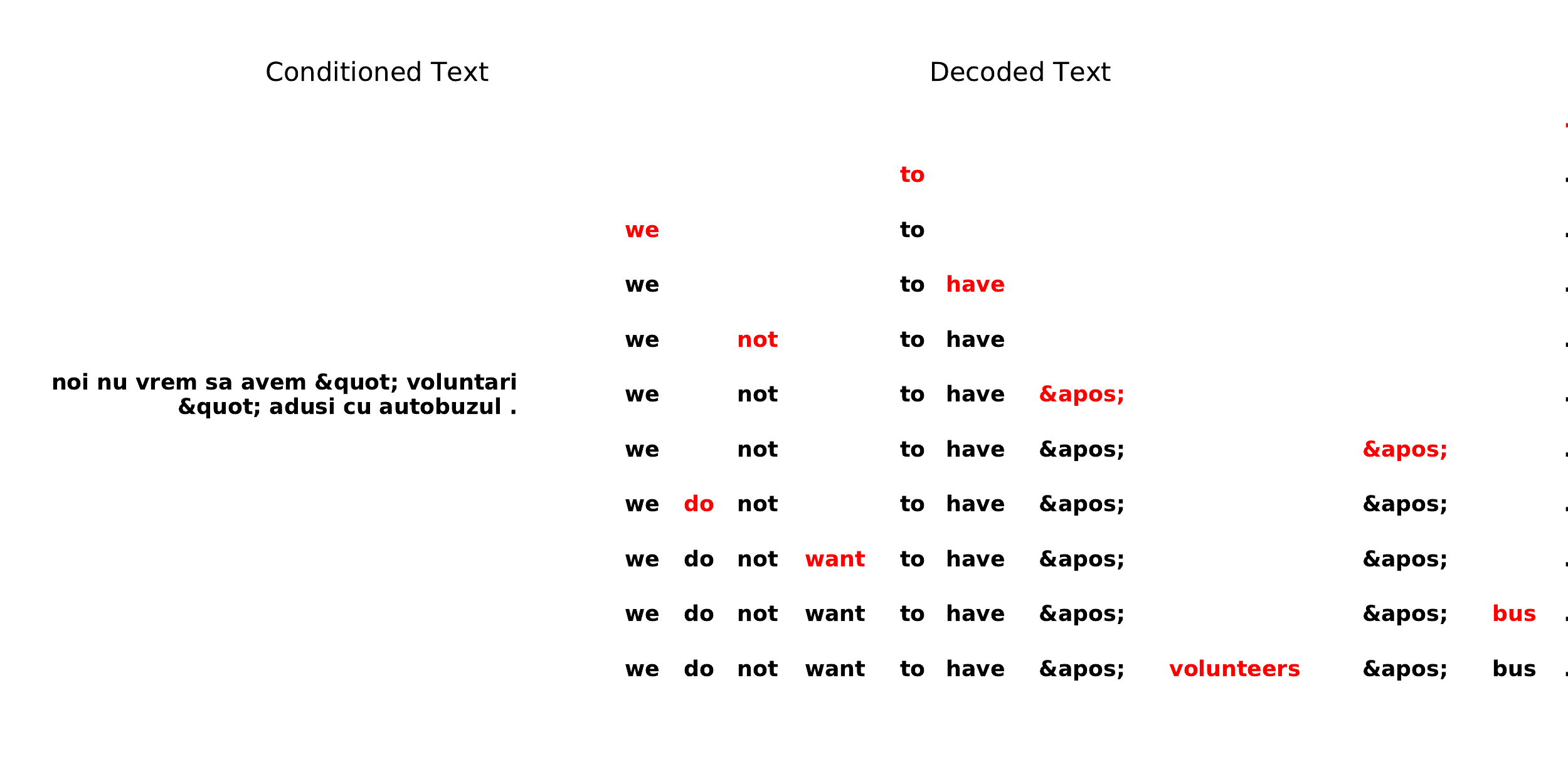}
    \caption{Generation order inferred by \textbf{Ours-Common} for a text sample from the WMT validation set with the sample id $\mathbf{71}$.}
    \label{fig:wmt/71_common}
\end{figure}

\begin{figure}[!htbp]
    \centering
    \includegraphics[width=\linewidth]{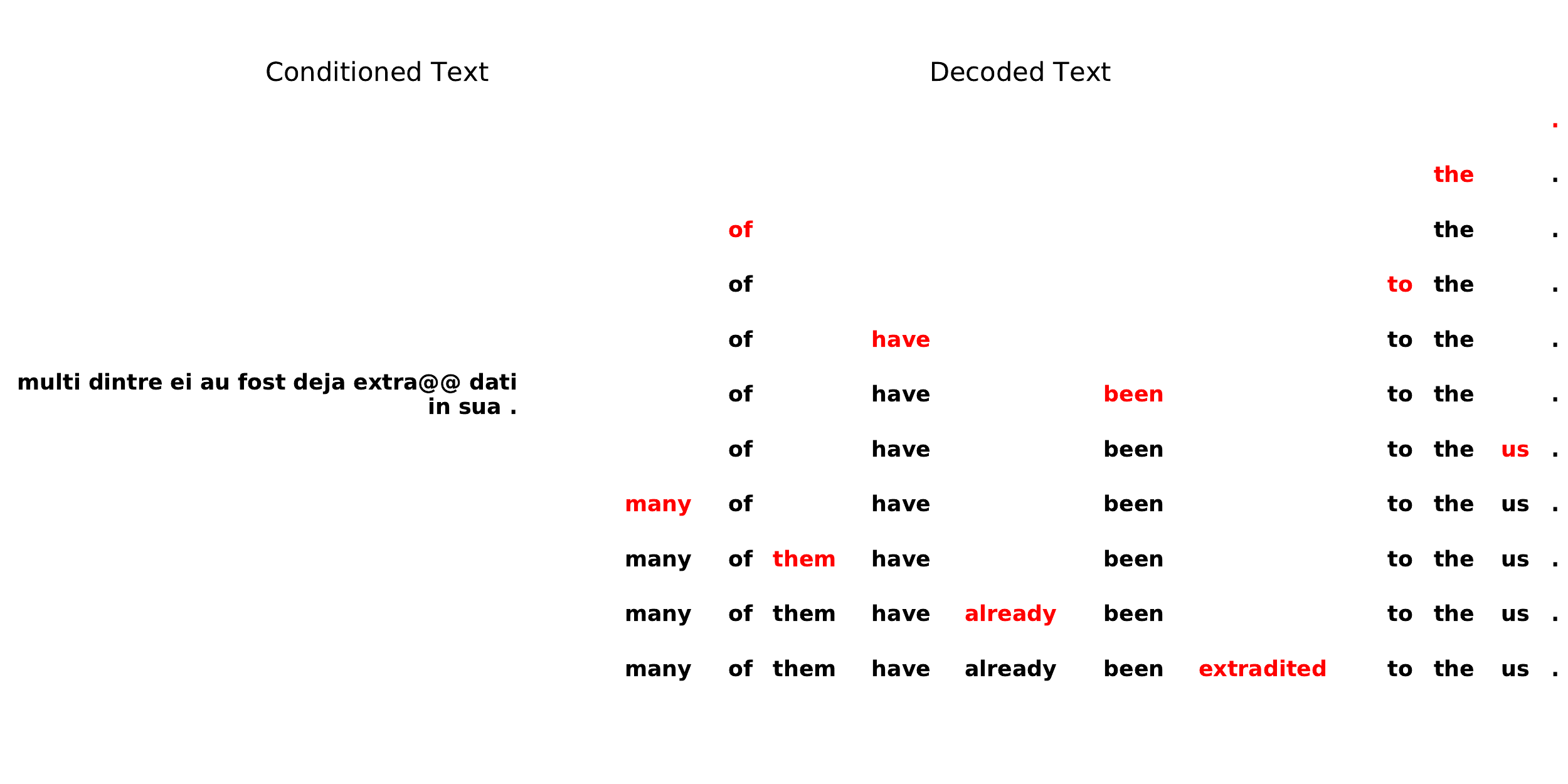}
    \caption{Generation order inferred by \textbf{Ours-Common} for a text sample from the WMT validation set with the sample id $\mathbf{72}$.}
    \label{fig:wmt/72_common}
\end{figure}

\begin{figure}[!htbp]
    \centering
    \includegraphics[width=\linewidth]{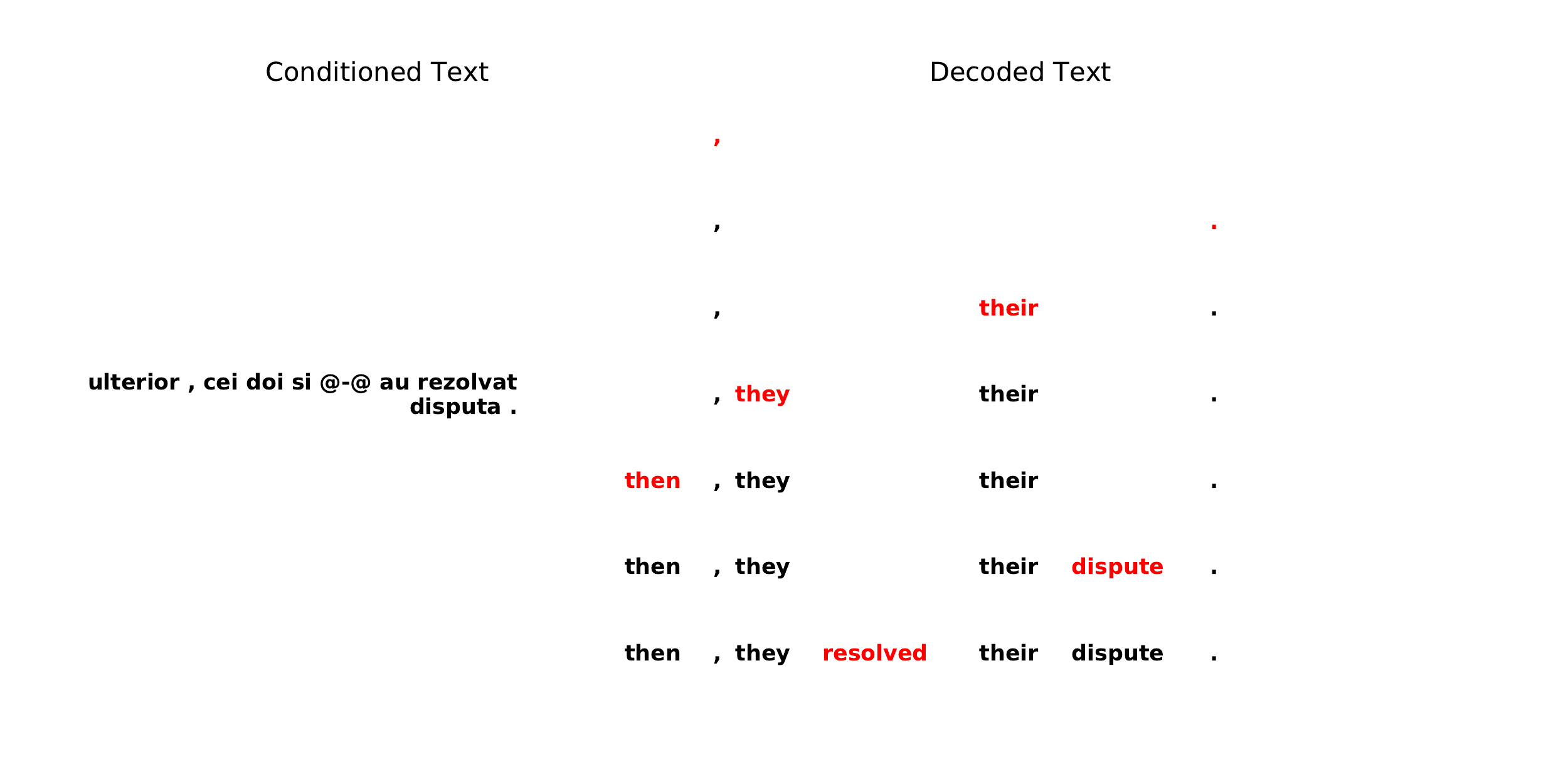}
    \caption{Generation order inferred by \textbf{Ours-Common} for a text sample from the WMT validation set with the sample id $\mathbf{74}$.}
    \label{fig:wmt/74_common}
\end{figure}

\begin{figure}[!htbp]
    \centering
    \includegraphics[width=\linewidth]{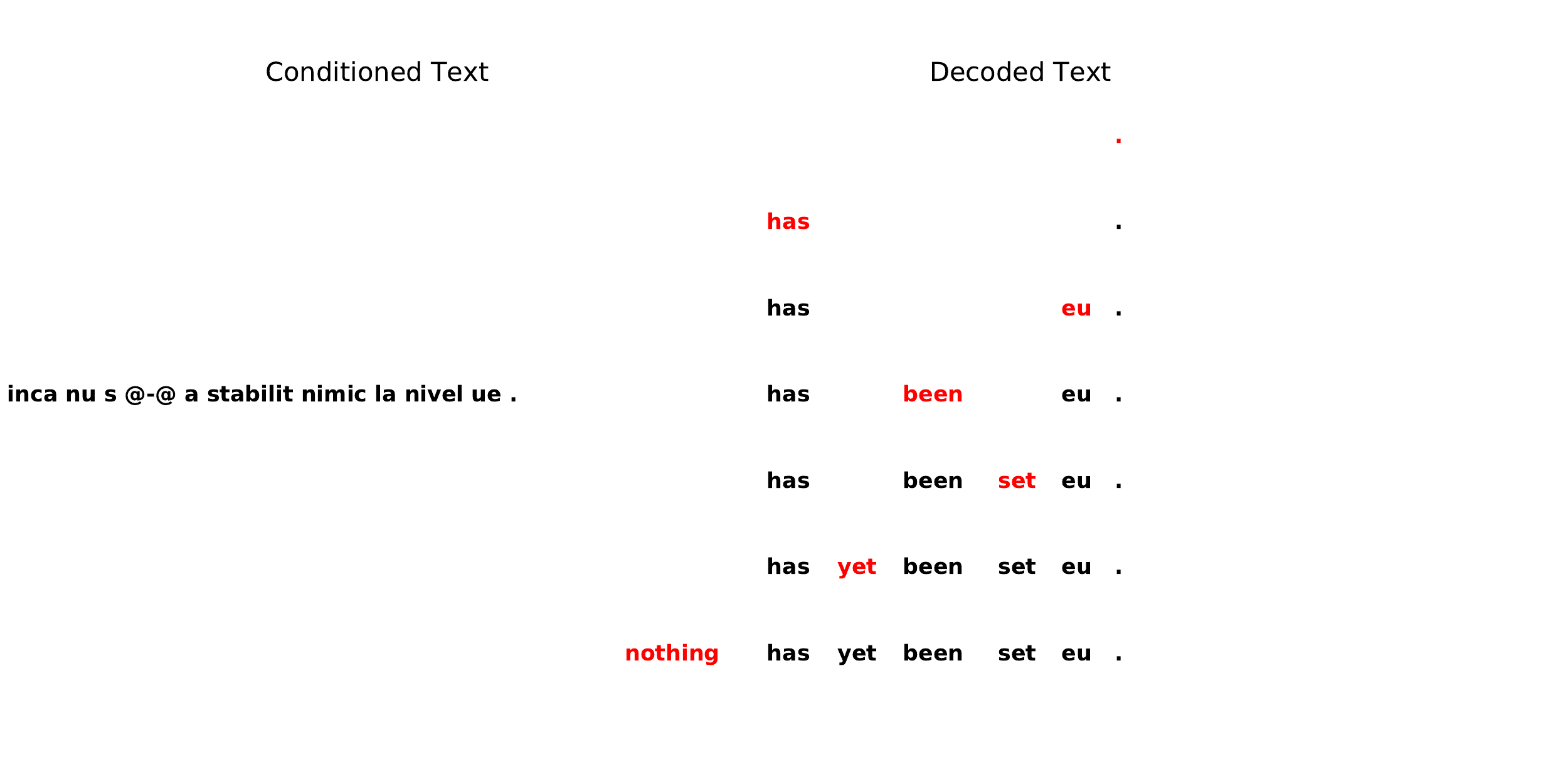}
    \caption{Generation order inferred by \textbf{Ours-Common} for a text sample from the WMT validation set with the sample id $\mathbf{110}$.}
    \label{fig:wmt/110_common}
\end{figure}

\begin{figure}[!htbp]
    \centering
    \includegraphics[width=\linewidth]{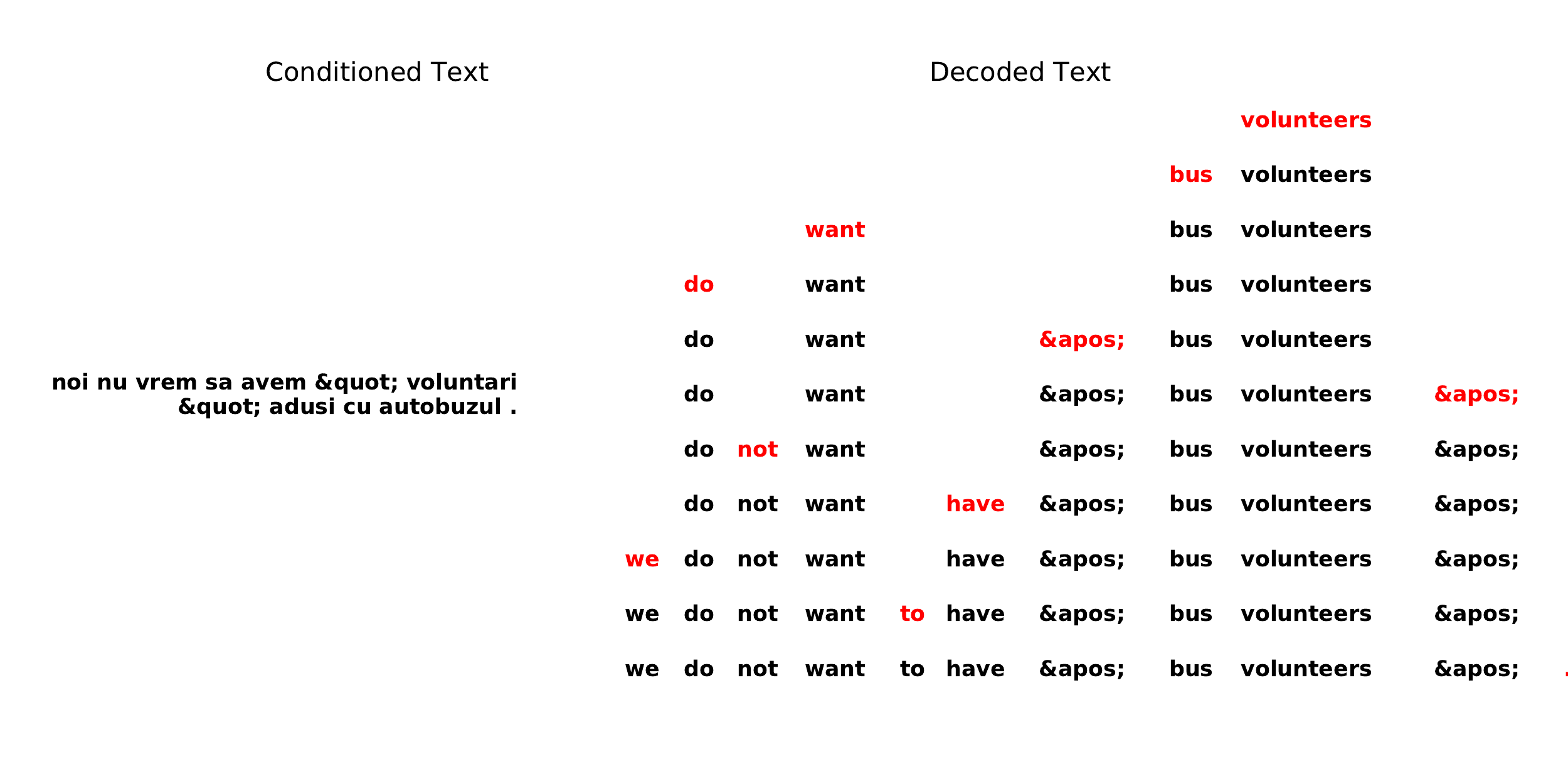}
    \caption{Generation order inferred by \textbf{Ours-Rare} for a text sample from the WMT validation set with the sample id $\mathbf{71}$.}
    \label{fig:wmt/71_rare}
\end{figure}

\begin{figure}[!htbp]
    \centering
    \includegraphics[width=\linewidth]{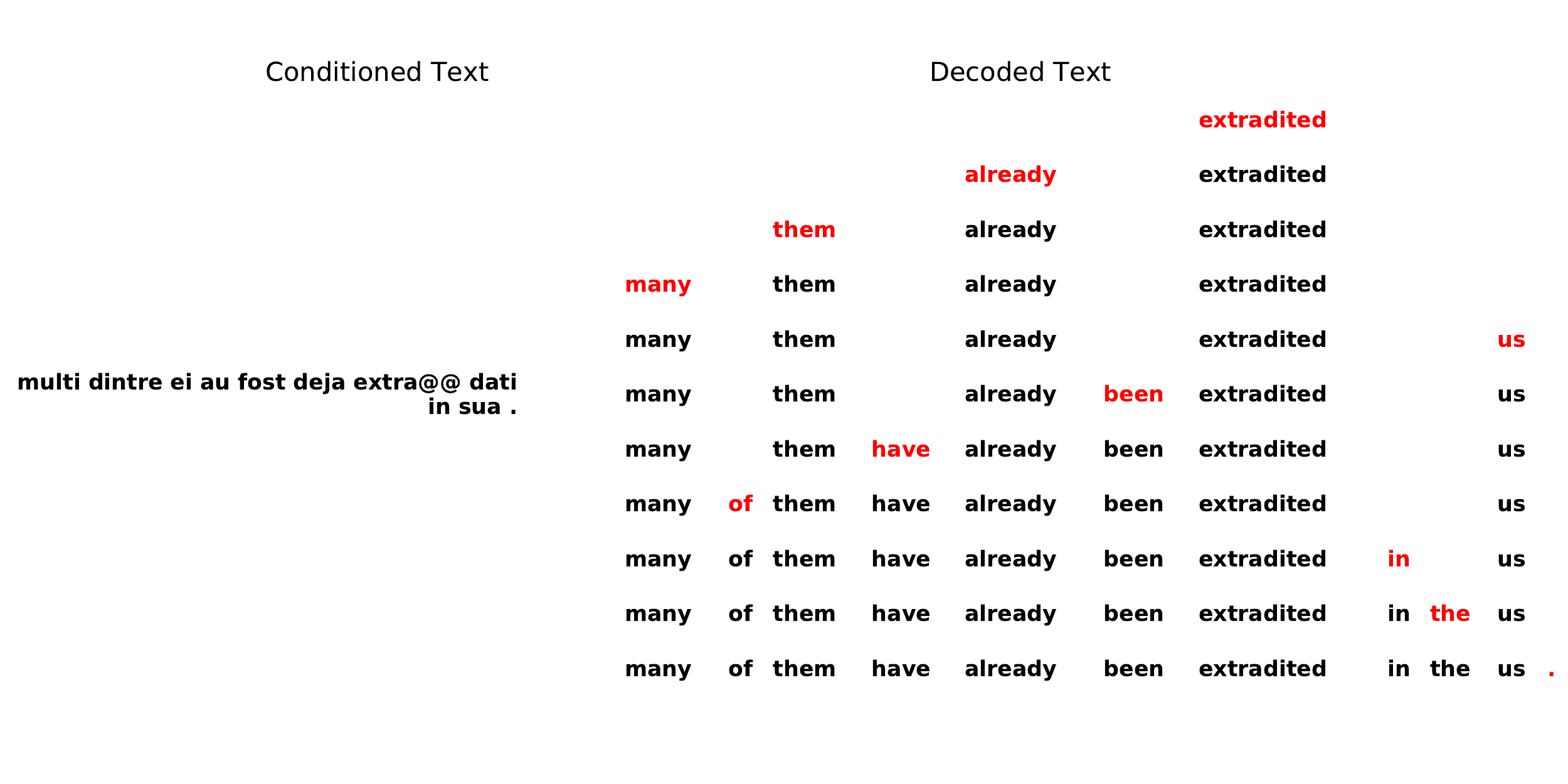}
    \caption{Generation order inferred by \textbf{Ours-Rare} for a text sample from the WMT validation set with the sample id $\mathbf{72}$.}
    \label{fig:wmt/72_rare}
\end{figure}

\begin{figure}[!htbp]
    \centering
    \includegraphics[width=\linewidth]{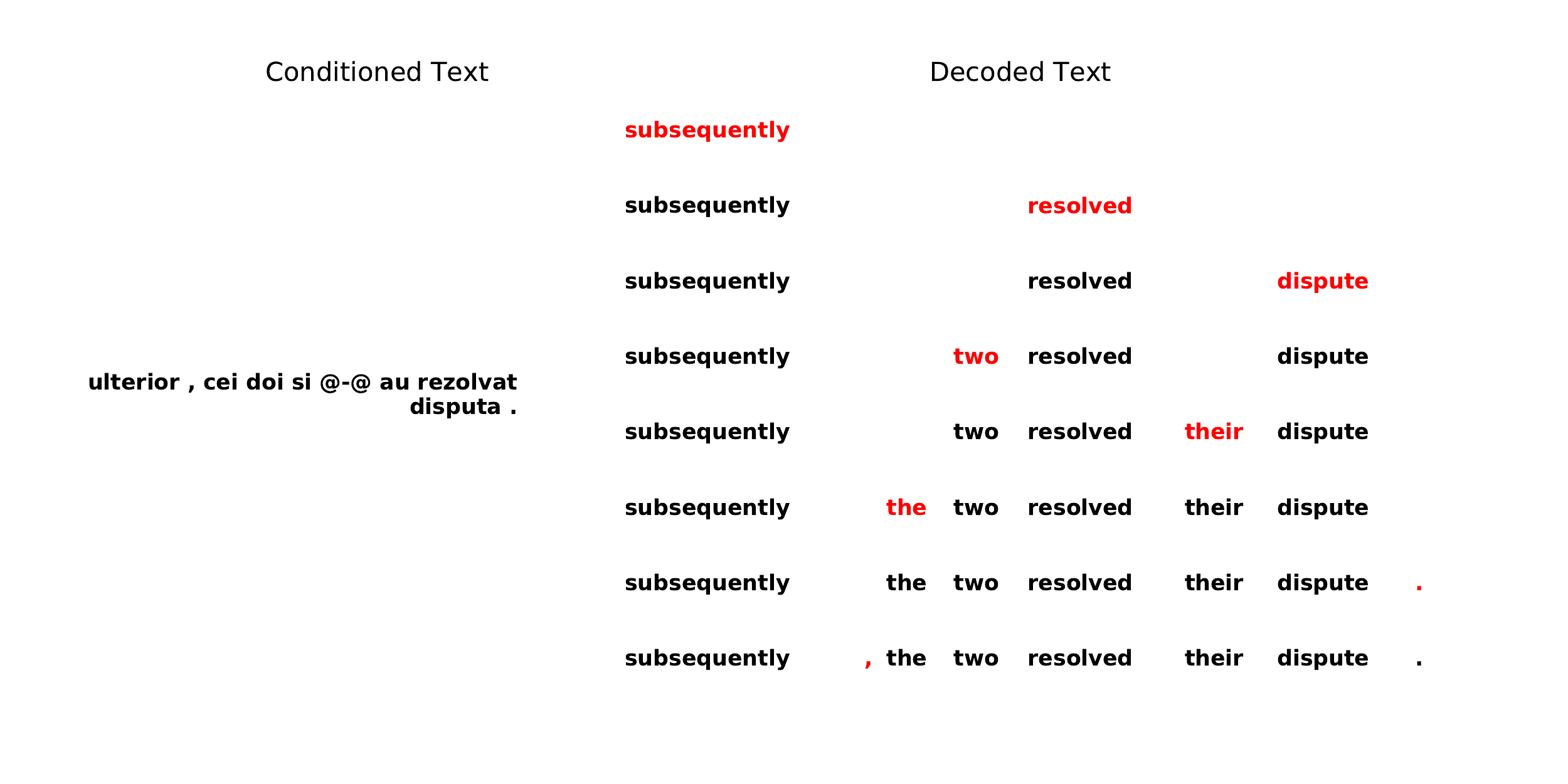}
    \caption{Generation order inferred by \textbf{Ours-Rare} for a text sample from the WMT validation set with the sample id $\mathbf{74}$.}
    \label{fig:wmt/74_rare}
\end{figure}

\begin{figure}[!htbp]
    \centering
    \includegraphics[width=\linewidth]{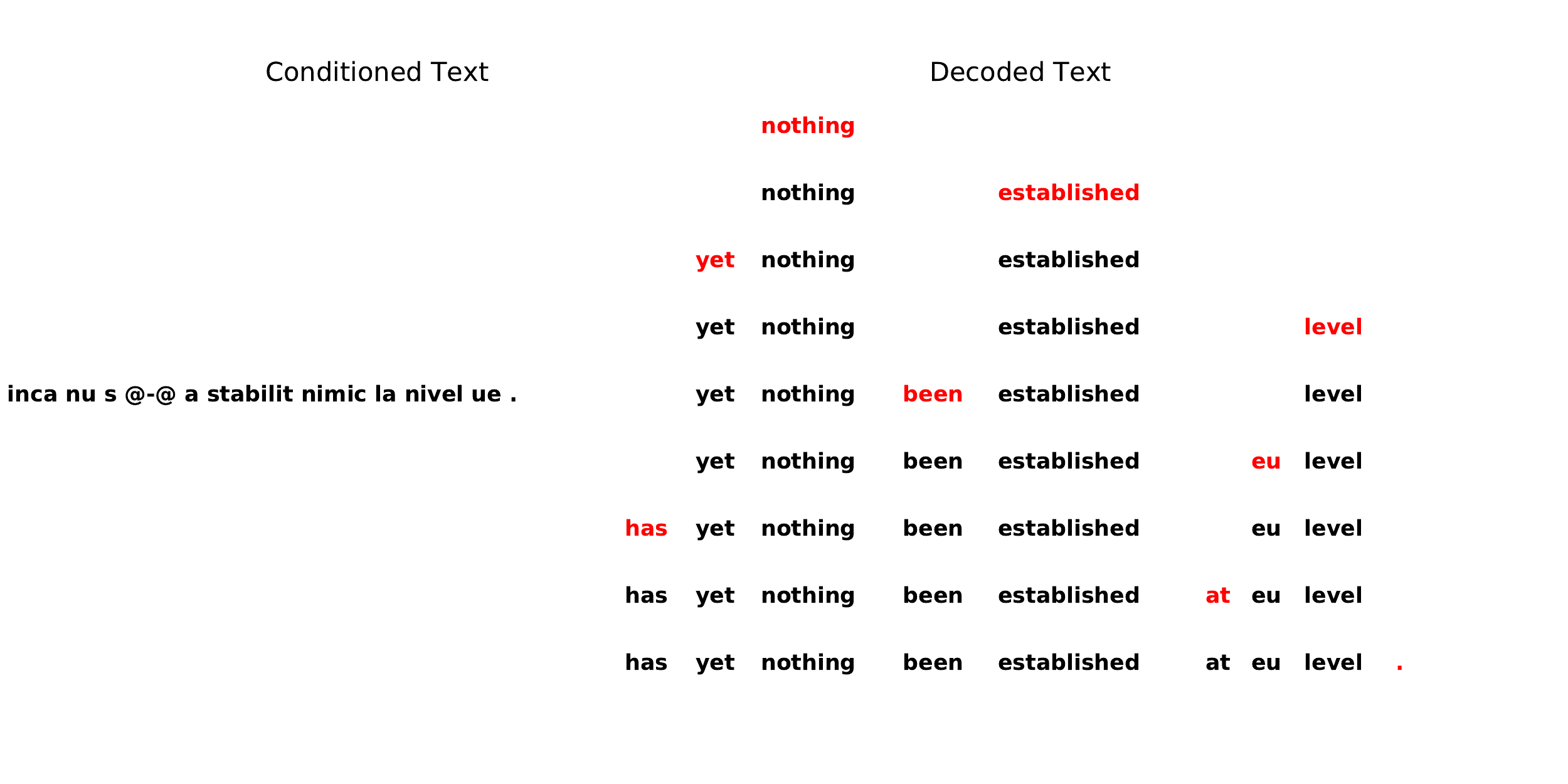}
    \caption{Generation order inferred by \textbf{Ours-Rare} for a text sample from the WMT validation set with the sample id $\mathbf{110}$.}
    \label{fig:wmt/110_rare}
\end{figure}

\label{app:visualization_app}

%% file: iclr2021_conference.bbl
\begin{thebibliography}{73}
\providecommand{\natexlab}[1]{#1}
\providecommand{\url}[1]{\texttt{#1}}
\expandafter\ifx\csname urlstyle\endcsname\relax
  \providecommand{\doi}[1]{doi: #1}\else
  \providecommand{\doi}{doi: \begingroup \urlstyle{rm}\Url}\fi

\bibitem[Aharoni \& Goldberg(2017)Aharoni and Goldberg]{aharoni2017towards}
Roee Aharoni and Yoav Goldberg.
\newblock Towards string-to-tree neural machine translation.
\newblock \emph{arXiv preprint arXiv:1704.04743}, 2017.

\bibitem[Alvarez{-}Melis \& Jaakkola(2017)Alvarez{-}Melis and
  Jaakkola]{DBLP:conf/iclr/Alvarez-MelisJ17}
David Alvarez{-}Melis and Tommi~S. Jaakkola.
\newblock Tree-structured decoding with doubly-recurrent neural networks.
\newblock In \emph{5th International Conference on Learning Representations,
  {ICLR} 2017, Toulon, France, April 24-26, 2017, Conference Track
  Proceedings}. OpenReview.net, 2017.
\newblock URL \url{https://openreview.net/forum?id=HkYhZDqxg}.

\bibitem[Anari \& Rezaei(2019)Anari and Rezaei]{bethe2018}
Nima Anari and Alireza Rezaei.
\newblock A tight analysis of bethe approximation for permanent.
\newblock \emph{2019 IEEE 60th Annual Symposium on Foundations of Computer
  Science (FOCS)}, pp.\  1434--1445, 2019.

\bibitem[Bahdanau et~al.(2015)Bahdanau, Cho, and
  Bengio]{DBLP:journals/corr/BahdanauCB14}
Dzmitry Bahdanau, Kyunghyun Cho, and Yoshua Bengio.
\newblock Neural machine translation by jointly learning to align and
  translate.
\newblock In Yoshua Bengio and Yann LeCun (eds.), \emph{3rd International
  Conference on Learning Representations, {ICLR} 2015, San Diego, CA, USA, May
  7-9, 2015, Conference Track Proceedings}, 2015.
\newblock URL \url{http://arxiv.org/abs/1409.0473}.

\bibitem[Brown et~al.(2020)Brown, Mann, Ryder, Subbiah, Kaplan, Dhariwal,
  Neelakantan, Shyam, Sastry, Askell, et~al.]{brown2020language}
Tom~B Brown, Benjamin Mann, Nick Ryder, Melanie Subbiah, Jared Kaplan, Prafulla
  Dhariwal, Arvind Neelakantan, Pranav Shyam, Girish Sastry, Amanda Askell,
  et~al.
\newblock Language models are few-shot learners.
\newblock \emph{arXiv preprint arXiv:2005.14165}, 2020.

\bibitem[Chan et~al.(2019)Chan, Kitaev, Guu, Stern, and
  Uszkoreit]{chan2019kermit}
William Chan, Nikita Kitaev, Kelvin Guu, Mitchell Stern, and Jakob Uszkoreit.
\newblock Kermit: Generative insertion-based modeling for sequences, 2019.

\bibitem[Charniak et~al.(2003)Charniak, Knight, and Yamada]{charniak2003syntax}
Eugene Charniak, Kevin Knight, and Kenji Yamada.
\newblock Syntax-based language models for statistical machine translation.
\newblock In \emph{Proceedings of MT Summit IX}, pp.\  40--46. Citeseer, 2003.

\bibitem[Chen et~al.(2018)Chen, Mishra, Rohaninejad, and
  Abbeel]{DBLP:conf/icml/ChenMRA18}
Xi~Chen, Nikhil Mishra, Mostafa Rohaninejad, and Pieter Abbeel.
\newblock Pixelsnail: An improved autoregressive generative model.
\newblock In Jennifer~G. Dy and Andreas Krause (eds.), \emph{Proceedings of the
  35th International Conference on Machine Learning, {ICML} 2018,
  Stockholmsm{\"{a}}ssan, Stockholm, Sweden, July 10-15, 2018}, volume~80 of
  \emph{Proceedings of Machine Learning Research}, pp.\  863--871. {PMLR},
  2018.
\newblock URL \url{http://proceedings.mlr.press/v80/chen18h.html}.

\bibitem[Cho et~al.(2014)Cho, van Merrienboer, G{\"{u}}l{\c{c}}ehre, Bahdanau,
  Bougares, Schwenk, and Bengio]{DBLP:conf/emnlp/ChoMGBBSB14}
Kyunghyun Cho, Bart van Merrienboer, {\c{C}}aglar G{\"{u}}l{\c{c}}ehre, Dzmitry
  Bahdanau, Fethi Bougares, Holger Schwenk, and Yoshua Bengio.
\newblock Learning phrase representations using {RNN} encoder-decoder for
  statistical machine translation.
\newblock In Alessandro Moschitti, Bo~Pang, and Walter Daelemans (eds.),
  \emph{Proceedings of the 2014 Conference on Empirical Methods in Natural
  Language Processing, {EMNLP} 2014, October 25-29, 2014, Doha, Qatar, {A}
  meeting of SIGDAT, a Special Interest Group of the {ACL}}, pp.\  1724--1734.
  {ACL}, 2014.
\newblock \doi{10.3115/v1/d14-1179}.
\newblock URL \url{https://doi.org/10.3115/v1/d14-1179}.

\bibitem[Dai et~al.(2019)Dai, Yang, Yang, Carbonell, Le, and
  Salakhutdinov]{transformer_xl}
Zihang Dai, Zhilin Yang, Yiming Yang, Jaime Carbonell, Quoc Le, and Ruslan
  Salakhutdinov.
\newblock Transformer-{XL}: Attentive language models beyond a fixed-length
  context.
\newblock In \emph{Proceedings of the 57th Annual Meeting of the Association
  for Computational Linguistics}, pp.\  2978--2988, Florence, Italy, July 2019.
  Association for Computational Linguistics.
\newblock \doi{10.18653/v1/P19-1285}.
\newblock URL \url{https://www.aclweb.org/anthology/P19-1285}.

\bibitem[Denkowski \& Lavie(2014)Denkowski and
  Lavie]{denkowski:lavie:meteor-wmt:2014}
Michael Denkowski and Alon Lavie.
\newblock Meteor universal: Language specific translation evaluation for any
  target language.
\newblock In \emph{Proceedings of the EACL 2014 Workshop on Statistical Machine
  Translation}, 2014.

\bibitem[Dyer et~al.(2016)Dyer, Kuncoro, Ballesteros, and
  Smith]{dyer2016recurrent}
Chris Dyer, Adhiguna Kuncoro, Miguel Ballesteros, and Noah~A Smith.
\newblock Recurrent neural network grammars.
\newblock \emph{arXiv preprint arXiv:1602.07776}, 2016.

\bibitem[Emelianenko et~al.(2019)Emelianenko, Voita, and
  Serdyukov]{DBLP:conf/nips/EmelianenkoVS19}
Dmitrii Emelianenko, Elena Voita, and Pavel Serdyukov.
\newblock Sequence modeling with unconstrained generation order.
\newblock In Hanna~M. Wallach, Hugo Larochelle, Alina Beygelzimer, Florence
  d'Alch{\'{e}}{-}Buc, Emily~B. Fox, and Roman Garnett (eds.), \emph{Advances
  in Neural Information Processing Systems 32: Annual Conference on Neural
  Information Processing Systems 2019, NeurIPS 2019, 8-14 December 2019,
  Vancouver, BC, Canada}, pp.\  7698--7709, 2019.
\newblock URL
  \url{http://papers.nips.cc/paper/8986-sequence-modeling-with-unconstrained-generation-order}.

\bibitem[Ford et~al.(2018{\natexlab{a}})Ford, Duckworth, Norouzi, and
  Dahl]{Ford2018TheIO}
N.~Ford, Daniel Duckworth, Mohammad Norouzi, and G.~Dahl.
\newblock The importance of generation order in language modeling.
\newblock \emph{ArXiv}, abs/1808.07910, 2018{\natexlab{a}}.

\bibitem[Ford et~al.(2018{\natexlab{b}})Ford, Duckworth, Norouzi, and
  Dahl]{ford2018importance}
Nicolas Ford, Daniel Duckworth, Mohammad Norouzi, and George~E Dahl.
\newblock The importance of generation order in language modeling.
\newblock \emph{arXiv preprint arXiv:1808.07910}, 2018{\natexlab{b}}.

\bibitem[Germain et~al.(2015)Germain, Gregor, Murray, and
  Larochelle]{pmlr-v37-germain15}
Mathieu Germain, Karol Gregor, Iain Murray, and Hugo Larochelle.
\newblock Made: Masked autoencoder for distribution estimation.
\newblock volume~37 of \emph{Proceedings of Machine Learning Research}, pp.\
  881--889, Lille, France, 07--09 Jul 2015. PMLR.
\newblock URL \url{http://proceedings.mlr.press/v37/germain15.html}.

\bibitem[Graff et~al.(2003)Graff, Kong, Chen, and Maeda]{gigaword2003}
David Graff, Junbo Kong, Ke~Chen, and Kazuaki Maeda.
\newblock English gigaword, 2003.

\bibitem[Grover et~al.(2019)Grover, Wang, Zweig, and Ermon]{plackett_luce1903}
Aditya Grover, E.~Wang, Aaron Zweig, and S.~Ermon.
\newblock Stochastic optimization of sorting networks via continuous
  relaxations.
\newblock \emph{ArXiv}, abs/1903.08850, 2019.

\bibitem[G{\=u} et~al.(2018)G{\=u}, Shavarani, and Sarkar]{gu-etal-2018-top}
Jetic G{\=u}, Hassan~S. Shavarani, and Anoop Sarkar.
\newblock Top-down tree structured decoding with syntactic connections for
  neural machine translation and parsing.
\newblock In \emph{Proceedings of the 2018 Conference on Empirical Methods in
  Natural Language Processing}, pp.\  401--413, Brussels, Belgium,
  October-November 2018. Association for Computational Linguistics.
\newblock \doi{10.18653/v1/D18-1037}.
\newblock URL \url{https://www.aclweb.org/anthology/D18-1037}.

\bibitem[Gu et~al.(2018)Gu, Bradbury, Xiong, Li, and
  Socher]{gu2018nonautoregressive}
Jiatao Gu, James Bradbury, Caiming Xiong, Victor O.~K. Li, and Richard Socher.
\newblock Non-autoregressive neural machine translation.
\newblock In \emph{5th International Conference on Learning Representations},
  2018.

\bibitem[Gu et~al.(2019{\natexlab{a}})Gu, Liu, and Cho]{gu2019insertion}
Jiatao Gu, Qi~Liu, and Kyunghyun Cho.
\newblock Insertion-based decoding with automatically inferred generation
  order.
\newblock \emph{Transactions of the Association for Computational Linguistics},
  7:\penalty0 661--676, 2019{\natexlab{a}}.

\bibitem[Gu et~al.(2019{\natexlab{b}})Gu, Wang, and
  Zhao]{levenhstein_transformer}
Jiatao Gu, Changhan Wang, and Junbo Zhao.
\newblock Levenshtein transformer.
\newblock In H.~Wallach, H.~Larochelle, A.~Beygelzimer, F.~d\textquotesingle
  Alch\'{e}-Buc, E.~Fox, and R.~Garnett (eds.), \emph{Advances in Neural
  Information Processing Systems}, volume~32, pp.\  11181--11191. Curran
  Associates, Inc., 2019{\natexlab{b}}.
\newblock URL
  \url{https://proceedings.neurips.cc/paper/2019/file/675f9820626f5bc0afb47b57890b466e-Paper.pdf}.

\bibitem[Higgins et~al.(2017)Higgins, Matthey, Pal, Burgess, Glorot, Botvinick,
  Mohamed, and Lerchner]{DBLP:conf/iclr/HigginsMPBGBML17}
Irina Higgins, Lo{\"{\i}}c Matthey, Arka Pal, Christopher Burgess, Xavier
  Glorot, Matthew Botvinick, Shakir Mohamed, and Alexander Lerchner.
\newblock beta-vae: Learning basic visual concepts with a constrained
  variational framework.
\newblock In \emph{5th International Conference on Learning Representations,
  {ICLR} 2017, Toulon, France, April 24-26, 2017, Conference Track
  Proceedings}. OpenReview.net, 2017.
\newblock URL \url{https://openreview.net/forum?id=Sy2fzU9gl}.

\bibitem[Karpathy \& Fei-Fei(2015)Karpathy and Fei-Fei]{karpathy2015deep}
Andrej Karpathy and Li~Fei-Fei.
\newblock Deep visual-semantic alignments for generating image descriptions.
\newblock In \emph{Proceedings of the IEEE conference on computer vision and
  pattern recognition}, pp.\  3128--3137, 2015.

\bibitem[Kim \& Rush(2016)Kim and Rush]{kim-rush-2016-sequence}
Yoon Kim and Alexander~M. Rush.
\newblock Sequence-level knowledge distillation.
\newblock In \emph{Proceedings of the 2016 Conference on Empirical Methods in
  Natural Language Processing}, pp.\  1317--1327, Austin, Texas, November 2016.
  Association for Computational Linguistics.
\newblock \doi{10.18653/v1/D16-1139}.
\newblock URL \url{https://www.aclweb.org/anthology/D16-1139}.

\bibitem[Kingma \& Ba(2015)Kingma and Ba]{adamopt}
Diederik~P. Kingma and Jimmy Ba.
\newblock Adam: {A} method for stochastic optimization.
\newblock In Yoshua Bengio and Yann LeCun (eds.), \emph{3rd International
  Conference on Learning Representations, {ICLR} 2015, San Diego, CA, USA, May
  7-9, 2015, Conference Track Proceedings}, 2015.
\newblock URL \url{http://arxiv.org/abs/1412.6980}.

\bibitem[Kingma \& Welling(2013)Kingma and Welling]{kingma2013auto}
Diederik~P Kingma and Max Welling.
\newblock Auto-encoding variational bayes.
\newblock \emph{arXiv preprint arXiv:1312.6114}, 2013.

\bibitem[Lin(2004)]{lin-2004-rouge}
Chin-Yew Lin.
\newblock {ROUGE}: A package for automatic evaluation of summaries.
\newblock In \emph{Text Summarization Branches Out}, pp.\  74--81, Barcelona,
  Spain, July 2004. Association for Computational Linguistics.
\newblock URL \url{https://www.aclweb.org/anthology/W04-1013}.

\bibitem[Lin et~al.(2015)Lin, Maire, Belongie, Bourdev, Girshick, Hays, Perona,
  Ramanan, Zitnick, and Dollár]{lin2015coco}
Tsung-Yi Lin, Michael Maire, Serge Belongie, Lubomir Bourdev, Ross Girshick,
  James Hays, Pietro Perona, Deva Ramanan, C.~Lawrence Zitnick, and Piotr
  Dollár.
\newblock Microsoft coco: Common objects in context, 2015.

\bibitem[Linderman et~al.(2018)Linderman, Mena, Cooper, Paninski, and
  Cunningham]{stick_breaking}
Scott Linderman, Gonzalo Mena, Hal Cooper, Liam Paninski, and John Cunningham.
\newblock Reparameterizing the birkhoff polytope for variational permutation
  inference.
\newblock volume~84 of \emph{Proceedings of Machine Learning Research}, pp.\
  1618--1627, Playa Blanca, Lanzarote, Canary Islands, 09--11 Apr 2018. PMLR.
\newblock URL \url{http://proceedings.mlr.press/v84/linderman18a.html}.

\bibitem[Ling et~al.(2016)Ling, Blunsom, Grefenstette, Hermann, Kocisk{\'y},
  Wang, and Senior]{Ling2016LatentPN}
W.~Ling, P.~Blunsom, Edward Grefenstette, K.~Hermann, Tom{\'a}s Kocisk{\'y},
  Fumin Wang, and A.~Senior.
\newblock Latent predictor networks for code generation.
\newblock \emph{ArXiv}, abs/1603.06744, 2016.

\bibitem[Luce(1959)]{luce59}
R.~Duncan Luce.
\newblock \emph{Individual Choice Behavior: A Theoretical analysis}.
\newblock Wiley, New York, NY, USA, 1959.

\bibitem[Luong et~al.(2015)Luong, Pham, and Manning]{DBLP:conf/emnlp/LuongPM15}
Thang Luong, Hieu Pham, and Christopher~D. Manning.
\newblock Effective approaches to attention-based neural machine translation.
\newblock In Llu{\'{\i}}s M{\`{a}}rquez, Chris Callison{-}Burch, Jian Su,
  Daniele Pighin, and Yuval Marton (eds.), \emph{Proceedings of the 2015
  Conference on Empirical Methods in Natural Language Processing, {EMNLP} 2015,
  Lisbon, Portugal, September 17-21, 2015}, pp.\  1412--1421. The Association
  for Computational Linguistics, 2015.
\newblock \doi{10.18653/v1/d15-1166}.
\newblock URL \url{https://doi.org/10.18653/v1/d15-1166}.

\bibitem[Mehri \& Sigal(2018)Mehri and Sigal]{mehri2018middle}
Shikib Mehri and Leonid Sigal.
\newblock Middle-out decoding.
\newblock In \emph{Advances in Neural Information Processing Systems}, pp.\
  5518--5529, 2018.

\bibitem[Mena et~al.(2018)Mena, Belanger, Linderman, and
  Snoek]{Mena2018sinkhorn}
Gonzalo Mena, David Belanger, Scott Linderman, and Jasper Snoek.
\newblock Learning latent permutations with gumbel-sinkhorn networks.
\newblock In \emph{International Conference on Learning Representations}, 2018.
\newblock URL \url{https://openreview.net/forum?id=Byt3oJ-0W}.

\bibitem[Mena et~al.(2020)Mena, Varol, Nejatbakhsh, Yemini, and
  Paninski]{mena2020sinkhornvariational}
Gonzalo Mena, Erdem Varol, Amin Nejatbakhsh, Eviatar Yemini, and Liam Paninski.
\newblock Sinkhorn permutation variational marginal inference.
\newblock volume 118 of \emph{Proceedings of Machine Learning Research}, pp.\
  1--9. PMLR, 08 Dec 2020.
\newblock URL \url{http://proceedings.mlr.press/v118/mena20a.html}.

\bibitem[Mikolov et~al.(2012)]{mikolov2012statistical}
Tom{\'a}{\v{s}} Mikolov et~al.
\newblock Statistical language models based on neural networks.
\newblock \emph{Presentation at Google, Mountain View, 2nd April}, 80:\penalty0
  26, 2012.

\bibitem[Munkres(1957)]{hungarian_alg}
James~R. Munkres.
\newblock {Algorithms for the Assignment and Transportation Problems}.
\newblock \emph{Journal of the Society for Industrial and Applied Mathematics},
  5\penalty0 (1):\penalty0 32--38, March 1957.

\bibitem[Oda et~al.(2015)Oda, Fudaba, Neubig, Hata, Sakti, Toda, and
  Nakamura]{oda2015ase:pseudogen1}
Yusuke Oda, Hiroyuki Fudaba, Graham Neubig, Hideaki Hata, Sakriani Sakti,
  Tomoki Toda, and Satoshi Nakamura.
\newblock Learning to generate pseudo-code from source code using statistical
  machine translation.
\newblock In \emph{Proceedings of the 2015 30th IEEE/ACM International
  Conference on Automated Software Engineering (ASE)}, ASE '15, pp.\  574--584,
  Lincoln, Nebraska, USA, November 2015. IEEE Computer Society.
\newblock ISBN 978-1-5090-0025-8.
\newblock \doi{10.1109/ASE.2015.36}.
\newblock URL \url{https://doi.org/10.1109/ASE.2015.36}.

\bibitem[Papineni et~al.(2002)Papineni, Roukos, Ward, and
  Zhu]{papineni-etal-2002-bleu}
Kishore Papineni, Salim Roukos, Todd Ward, and Wei-Jing Zhu.
\newblock {B}leu: a method for automatic evaluation of machine translation.
\newblock In \emph{Proceedings of the 40th Annual Meeting of the Association
  for Computational Linguistics}, pp.\  311--318, Philadelphia, Pennsylvania,
  USA, July 2002. Association for Computational Linguistics.
\newblock \doi{10.3115/1073083.1073135}.
\newblock URL \url{https://www.aclweb.org/anthology/P02-1040}.

\bibitem[Plackett(1975)]{plackett}
Robin~L Plackett.
\newblock The analysis of permutations.
\newblock pp.\  193--202, 1975.

\bibitem[Radford et~al.(2018)Radford, Narasimhan, Salimans, and
  Sutskever]{radford2018improving}
Alec Radford, Karthik Narasimhan, Tim Salimans, and Ilya Sutskever.
\newblock Improving language understanding by generative pre-training, 2018.

\bibitem[Radford et~al.(2019)Radford, Wu, Child, Luan, Amodei, and
  Sutskever]{radford2019language}
Alec Radford, Jeffrey Wu, Rewon Child, David Luan, Dario Amodei, and Ilya
  Sutskever.
\newblock Language models are unsupervised multitask learners.
\newblock \emph{OpenAI Blog}, 1\penalty0 (8):\penalty0 9, 2019.

\bibitem[Ruis et~al.(2020)Ruis, Stern, Proskurnia, and
  Chan]{DBLP:journals/corr/abs-2001-05540}
Laura Ruis, Mitchell Stern, Julia Proskurnia, and William Chan.
\newblock Insertion-deletion transformer.
\newblock \emph{CoRR}, abs/2001.05540, 2020.
\newblock URL \url{https://arxiv.org/abs/2001.05540}.

\bibitem[Rush et~al.(2015)Rush, Chopra, and Weston]{gigaword2015}
Alexander~M. Rush, Sumit Chopra, and Jason Weston.
\newblock A neural attention model for abstractive sentence summarization.
\newblock \emph{Proceedings of the 2015 Conference on Empirical Methods in
  Natural Language Processing}, 2015.
\newblock \doi{10.18653/v1/d15-1044}.
\newblock URL \url{http://dx.doi.org/10.18653/v1/D15-1044}.

\bibitem[Schulman et~al.(2017)Schulman, Wolski, Dhariwal, Radford, and
  Klimov]{schulman2017proximal}
John Schulman, Filip Wolski, Prafulla Dhariwal, Alec Radford, and Oleg Klimov.
\newblock Proximal policy optimization algorithms, 2017.

\bibitem[Sennrich et~al.(2016)Sennrich, Haddow, and Birch]{bytepairencoding}
Rico Sennrich, Barry Haddow, and Alexandra Birch.
\newblock Neural machine translation of rare words with subword units.
\newblock In \emph{Proceedings of the 54th Annual Meeting of the Association
  for Computational Linguistics (Volume 1: Long Papers)}, pp.\  1715--1725,
  Berlin, Germany, August 2016. Association for Computational Linguistics.
\newblock \doi{10.18653/v1/P16-1162}.
\newblock URL \url{https://www.aclweb.org/anthology/P16-1162}.

\bibitem[Shaw et~al.(2018)Shaw, Uszkoreit, and
  Vaswani]{self_attention_relative_pos}
Peter Shaw, Jakob Uszkoreit, and Ashish Vaswani.
\newblock Self-attention with relative position representations.
\newblock In \emph{Proceedings of the 2018 Conference of the North {A}merican
  Chapter of the Association for Computational Linguistics: Human Language
  Technologies, Volume 2 (Short Papers)}, pp.\  464--468, New Orleans,
  Louisiana, June 2018. Association for Computational Linguistics.
\newblock \doi{10.18653/v1/N18-2074}.
\newblock URL \url{https://www.aclweb.org/anthology/N18-2074}.

\bibitem[Sinkhorn(1964)]{sinkhorn1964}
Richard Sinkhorn.
\newblock A relationship between arbitrary positive matrices and doubly
  stochastic matrices.
\newblock \emph{Ann. Math. Statist.}, 35\penalty0 (2):\penalty0 876--879, 06
  1964.
\newblock \doi{10.1214/aoms/1177703591}.
\newblock URL \url{https://doi.org/10.1214/aoms/1177703591}.

\bibitem[Snover et~al.(2006)Snover, Dorr, Schwartz, Micciulla, and
  Makhoul]{Snover06astudy}
Matthew Snover, Bonnie Dorr, Richard Schwartz, Linnea Micciulla, and John
  Makhoul.
\newblock A study of translation edit rate with targeted human annotation.
\newblock In \emph{In Proceedings of Association for Machine Translation in the
  Americas}, pp.\  223--231, 2006.

\bibitem[Srivastava et~al.(2014)Srivastava, Hinton, Krizhevsky, Sutskever, and
  Salakhutdinov]{dropout}
Nitish Srivastava, Geoffrey Hinton, Alex Krizhevsky, Ilya Sutskever, and Ruslan
  Salakhutdinov.
\newblock Dropout: A simple way to prevent neural networks from overfitting.
\newblock 15\penalty0 (1):\penalty0 1929–1958, January 2014.
\newblock ISSN 1532-4435.

\bibitem[Stahlberg(2019)]{Stahlberg2019NeuralMT}
Felix Stahlberg.
\newblock Neural machine translation: A review.
\newblock \emph{ArXiv}, abs/1912.02047, 2019.

\bibitem[Stern et~al.(2019)Stern, Chan, Kiros, and
  Uszkoreit]{DBLP:conf/icml/SternCKU19}
Mitchell Stern, William Chan, Jamie Kiros, and Jakob Uszkoreit.
\newblock Insertion transformer: Flexible sequence generation via insertion
  operations.
\newblock In Kamalika Chaudhuri and Ruslan Salakhutdinov (eds.),
  \emph{Proceedings of the 36th International Conference on Machine Learning,
  {ICML} 2019, 9-15 June 2019, Long Beach, California, {USA}}, volume~97 of
  \emph{Proceedings of Machine Learning Research}, pp.\  5976--5985. {PMLR},
  2019.
\newblock URL \url{http://proceedings.mlr.press/v97/stern19a.html}.

\bibitem[Sun et~al.(2017)Sun, Lee, and Batra]{sun2017bidirectional}
Qing Sun, Stefan Lee, and Dhruv Batra.
\newblock Bidirectional beam search: Forward-backward inference in neural
  sequence models for fill-in-the-blank image captioning.
\newblock In \emph{Proceedings of the IEEE Conference on Computer Vision and
  Pattern Recognition}, pp.\  6961--6969, 2017.

\bibitem[Sutskever et~al.(2011)Sutskever, Martens, and
  Hinton]{sutskever2011generating}
Ilya Sutskever, James Martens, and Geoffrey~E Hinton.
\newblock Generating text with recurrent neural networks.
\newblock In \emph{ICML}, 2011.

\bibitem[Sutskever et~al.(2014{\natexlab{a}})Sutskever, Vinyals, and
  Le]{DBLP:conf/nips/SutskeverVL14}
Ilya Sutskever, Oriol Vinyals, and Quoc~V. Le.
\newblock Sequence to sequence learning with neural networks.
\newblock In Zoubin Ghahramani, Max Welling, Corinna Cortes, Neil~D. Lawrence,
  and Kilian~Q. Weinberger (eds.), \emph{Advances in Neural Information
  Processing Systems 27: Annual Conference on Neural Information Processing
  Systems 2014, December 8-13 2014, Montreal, Quebec, Canada}, pp.\
  3104--3112, 2014{\natexlab{a}}.
\newblock URL
  \url{http://papers.nips.cc/paper/5346-sequence-to-sequence-learning-with-neural-networks}.

\bibitem[Sutskever et~al.(2014{\natexlab{b}})Sutskever, Vinyals, and
  Le]{sutskever2014sequence}
Ilya Sutskever, Oriol Vinyals, and Quoc~V Le.
\newblock Sequence to sequence learning with neural networks.
\newblock In \emph{Advances in neural information processing systems}, pp.\
  3104--3112, 2014{\natexlab{b}}.

\bibitem[Sutton et~al.(2000)Sutton, McAllester, Singh, and
  Mansour]{NIPS1999_1713}
Richard~S Sutton, David~A. McAllester, Satinder~P. Singh, and Yishay Mansour.
\newblock Policy gradient methods for reinforcement learning with function
  approximation.
\newblock In S.~A. Solla, T.~K. Leen, and K.~M\"{u}ller (eds.), \emph{Advances
  in Neural Information Processing Systems 12}, pp.\  1057--1063. MIT Press,
  2000.
\newblock URL
  \url{http://papers.nips.cc/paper/1713-policy-gradient-methods-for-reinforcement-learning-with-function-approximation.pdf}.

\bibitem[Uria et~al.(2016)Uria, C{\^{o}}t{\'{e}}, Gregor, Murray, and
  Larochelle]{DBLP:journals/jmlr/UriaCGML16}
Benigno Uria, Marc{-}Alexandre C{\^{o}}t{\'{e}}, Karol Gregor, Iain Murray, and
  Hugo Larochelle.
\newblock Neural autoregressive distribution estimation.
\newblock \emph{J. Mach. Learn. Res.}, 17:\penalty0 205:1--205:37, 2016.
\newblock URL \url{http://jmlr.org/papers/v17/16-272.html}.

\bibitem[van~den Oord et~al.(2016)van~den Oord, Kalchbrenner, Espeholt,
  Kavukcuoglu, Vinyals, and Graves]{DBLP:conf/nips/OordKEKVG16}
A{\"{a}}ron van~den Oord, Nal Kalchbrenner, Lasse Espeholt, Koray Kavukcuoglu,
  Oriol Vinyals, and Alex Graves.
\newblock Conditional image generation with pixelcnn decoders.
\newblock In Daniel~D. Lee, Masashi Sugiyama, Ulrike von Luxburg, Isabelle
  Guyon, and Roman Garnett (eds.), \emph{Advances in Neural Information
  Processing Systems 29: Annual Conference on Neural Information Processing
  Systems 2016, December 5-10, 2016, Barcelona, Spain}, pp.\  4790--4798, 2016.
\newblock URL
  \url{http://papers.nips.cc/paper/6527-conditional-image-generation-with-pixelcnn-decoders}.

\bibitem[Vaswani et~al.(2017)Vaswani, Shazeer, Parmar, Uszkoreit, Jones, Gomez,
  Kaiser, and Polosukhin]{attallyouneed}
Ashish Vaswani, Noam Shazeer, Niki Parmar, Jakob Uszkoreit, Llion Jones,
  Aidan~N Gomez, \L~ukasz Kaiser, and Illia Polosukhin.
\newblock Attention is all you need.
\newblock In I.~Guyon, U.~V. Luxburg, S.~Bengio, H.~Wallach, R.~Fergus,
  S.~Vishwanathan, and R.~Garnett (eds.), \emph{Advances in Neural Information
  Processing Systems 30}, pp.\  5998--6008. Curran Associates, Inc., 2017.
\newblock URL
  \url{http://papers.nips.cc/paper/7181-attention-is-all-you-need.pdf}.

\bibitem[Vedantam et~al.(2015)Vedantam, Zitnick, and
  Parikh]{DBLP:conf/cvpr/VedantamZP15}
Ramakrishna Vedantam, C.~Lawrence Zitnick, and Devi Parikh.
\newblock Cider: Consensus-based image description evaluation.
\newblock In \emph{{IEEE} Conference on Computer Vision and Pattern
  Recognition, {CVPR} 2015, Boston, MA, USA, June 7-12, 2015}, pp.\
  4566--4575. {IEEE} Computer Society, 2015.
\newblock \doi{10.1109/CVPR.2015.7299087}.
\newblock URL \url{https://doi.org/10.1109/CVPR.2015.7299087}.

\bibitem[Vinyals et~al.(2015{\natexlab{a}})Vinyals, Fortunato, and
  Jaitly]{DBLP:conf/nips/VinyalsFJ15}
Oriol Vinyals, Meire Fortunato, and Navdeep Jaitly.
\newblock Pointer networks.
\newblock In Corinna Cortes, Neil~D. Lawrence, Daniel~D. Lee, Masashi Sugiyama,
  and Roman Garnett (eds.), \emph{Advances in Neural Information Processing
  Systems 28: Annual Conference on Neural Information Processing Systems 2015,
  December 7-12, 2015, Montreal, Quebec, Canada}, pp.\  2692--2700,
  2015{\natexlab{a}}.
\newblock URL \url{http://papers.nips.cc/paper/5866-pointer-networks}.

\bibitem[Vinyals et~al.(2015{\natexlab{b}})Vinyals, Toshev, Bengio, and
  Erhan]{DBLP:conf/cvpr/VinyalsTBE15}
Oriol Vinyals, Alexander Toshev, Samy Bengio, and Dumitru Erhan.
\newblock Show and tell: {A} neural image caption generator.
\newblock In \emph{{IEEE} Conference on Computer Vision and Pattern
  Recognition, {CVPR} 2015, Boston, MA, USA, June 7-12, 2015}, pp.\
  3156--3164. {IEEE} Computer Society, 2015{\natexlab{b}}.
\newblock \doi{10.1109/CVPR.2015.7298935}.
\newblock URL \url{https://doi.org/10.1109/CVPR.2015.7298935}.

\bibitem[Vinyals et~al.(2016)Vinyals, Bengio, and
  Kudlur]{DBLP:journals/corr/VinyalsBK15}
Oriol Vinyals, Samy Bengio, and Manjunath Kudlur.
\newblock Order matters: Sequence to sequence for sets.
\newblock In Yoshua Bengio and Yann LeCun (eds.), \emph{4th International
  Conference on Learning Representations, {ICLR} 2016, San Juan, Puerto Rico,
  May 2-4, 2016, Conference Track Proceedings}, 2016.
\newblock URL \url{http://arxiv.org/abs/1511.06391}.

\bibitem[{Vontobel}(2010)]{bethe2011}
P.~O. {Vontobel}.
\newblock The bethe permanent of a non-negative matrix.
\newblock In \emph{2010 48th Annual Allerton Conference on Communication,
  Control, and Computing (Allerton)}, pp.\  341--346, 2010.

\bibitem[Wang et~al.(2018)Wang, Pham, Yin, and Neubig]{wang2018tree}
Xinyi Wang, Hieu Pham, Pengcheng Yin, and Graham Neubig.
\newblock A tree-based decoder for neural machine translation.
\newblock \emph{arXiv preprint arXiv:1808.09374}, 2018.

\bibitem[Welleck et~al.(2019{\natexlab{a}})Welleck, Brantley, Daum{\'e}~III,
  and Cho]{welleck2019non}
Sean Welleck, Kiant{\'e} Brantley, Hal Daum{\'e}~III, and Kyunghyun Cho.
\newblock Non-monotonic sequential text generation.
\newblock \emph{arXiv preprint arXiv:1902.02192}, 2019{\natexlab{a}}.

\bibitem[Welleck et~al.(2019{\natexlab{b}})Welleck, Brantley, III, and
  Cho]{DBLP:conf/icml/WelleckBDC19}
Sean Welleck, Kiant{\'{e}} Brantley, Hal~Daum{\'{e}} III, and Kyunghyun Cho.
\newblock Non-monotonic sequential text generation.
\newblock In Kamalika Chaudhuri and Ruslan Salakhutdinov (eds.),
  \emph{Proceedings of the 36th International Conference on Machine Learning,
  {ICML} 2019, 9-15 June 2019, Long Beach, California, {USA}}, volume~97 of
  \emph{Proceedings of Machine Learning Research}, pp.\  6716--6726. {PMLR},
  2019{\natexlab{b}}.
\newblock URL \url{http://proceedings.mlr.press/v97/welleck19a.html}.

\bibitem[Wu et~al.(2018)Wu, Tan, He, Tian, Qin, Lai, and
  Liu]{wu-etal-2018-beyond}
Lijun Wu, Xu~Tan, Di~He, Fei Tian, Tao Qin, Jianhuang Lai, and Tie-Yan Liu.
\newblock Beyond error propagation in neural machine translation:
  Characteristics of language also matter.
\newblock In \emph{Proceedings of the 2018 Conference on Empirical Methods in
  Natural Language Processing}, pp.\  3602--3611, Brussels, Belgium,
  October-November 2018. Association for Computational Linguistics.
\newblock \doi{10.18653/v1/D18-1396}.
\newblock URL \url{https://www.aclweb.org/anthology/D18-1396}.

\bibitem[Xu et~al.(2015)Xu, Ba, Kiros, Cho, Courville, Salakhutdinov, Zemel,
  and Bengio]{DBLP:conf/icml/XuBKCCSZB15}
Kelvin Xu, Jimmy Ba, Ryan Kiros, Kyunghyun Cho, Aaron~C. Courville, Ruslan
  Salakhutdinov, Richard~S. Zemel, and Yoshua Bengio.
\newblock Show, attend and tell: Neural image caption generation with visual
  attention.
\newblock In Francis~R. Bach and David~M. Blei (eds.), \emph{Proceedings of the
  32nd International Conference on Machine Learning, {ICML} 2015, Lille,
  France, 6-11 July 2015}, volume~37 of \emph{{JMLR} Workshop and Conference
  Proceedings}, pp.\  2048--2057. JMLR.org, 2015.
\newblock URL \url{http://proceedings.mlr.press/v37/xuc15.html}.

\bibitem[Yamada \& Knight(2001)Yamada and Knight]{yamada2001syntax}
Kenji Yamada and Kevin Knight.
\newblock A syntax-based statistical translation model.
\newblock In \emph{Proceedings of the 39th Annual Meeting of the Association
  for Computational Linguistics}, pp.\  523--530, 2001.

\bibitem[Zhou et~al.(2019)Zhou, Zhang, and Zong]{zhou2019synchronous}
Long Zhou, Jiajun Zhang, and Chengqing Zong.
\newblock Synchronous bidirectional neural machine translation.
\newblock \emph{Transactions of the Association for Computational Linguistics},
  7:\penalty0 91--105, 2019.

\end{thebibliography}
